\documentclass{article}
\usepackage[T1]{fontenc}
\usepackage{iclr2027_conference,times}
\usepackage{amsmath,amssymb,amsthm,booktabs,graphicx,microtype,adjustbox,array,tabularx}
\usepackage{tikz}
\usetikzlibrary{positioning,arrows.meta}
\theoremstyle{definition}
\newtheorem{proposition}{Proposition}

\usepackage[utf8]{inputenc}
\usepackage{hyperref}
\hypersetup{hidelinks}
\usepackage{url}

\iclrfinalcopy
\title{Better Policies, Worse Verifier Audits}

\author{%
Daniel Cie\'slak\thanks{Corresponding author: \texttt{daniel.cieslak@pg.edu.pl}}\\
{\normalfont\small Gda\'nsk University of Technology}\\
{\normalfont\small Gda\'nsk, Poland}
\And
Jacek \L ubi\'nski\\
{\normalfont\small Gda\'nsk University of Technology}\\
{\normalfont\small Gda\'nsk, Poland}
}

\begin{document}
\maketitle

\begin{abstract}
Verifier stacks increasingly sit inside post-training and test-time compute loops. We study a
specific control loop: a component verifier is evaluated on fresh policy outputs against the
\emph{full} success label and refreshed when that aggregate score falls. For a conjunctive target
$y=F\wedge M$, its on-policy AUC is $(1-q)A_{\rm own}+qA_{\rm cross}$, where
$q=\Pr[F=1\mid y=0]$. As the policy masters $F$, $q$ can rise toward one, removing weight from
$F$-failures and rewarding information about $M$. We call this \emph{on-policy audit collapse}.
We derive a sharp identified set, a complete-erasure threshold, and the exact point at which the
audit can prefer a worse component verifier; its population optimum is the full-outcome predictor,
not the component oracle. An exact oracle exhibits the collapse, a natural multi-label text control
shows inversion can select a head with lower canonical $F$-AUC at the observed failure mixture,
and a prediction-registered ECG study shows that acting on the audit can reduce off-policy $F$-AUC,
transfer the damage to an unseen generator, and increase downstream constraint violations. Preserving
own-failure support prevents the damage. Matching label-complexity bounds and a controlled-mixture
audit give the corresponding repair: audit the component on its own failures, and keep those failures
in the refresh distribution.
\end{abstract}

\section{Introduction}
\label{sec:intro}

Verifier stacks increasingly sit inside the optimization loop. Post-training systems combine
outcome, process, rubric and constraint signals, while test-time compute uses verifier scores to
search, rank and select candidates
\citep{cobbe2021verifiers,lightman2023verify,brown2024monkeys,snell2024scaling,mahmoud2026rubric,
wu2023finegrained,zhou2023modpo,moskovitz2023constrained,liu2026gdpo,bai2026rdpo}. A natural but unsafe control loop is to monitor a component verifier on fresh policy samples
against the full task outcome; if that aggregate score drops, collect new on-policy data and refresh
the verifier. Policy-interaction data are in fact used to refine reward models and iterative RLHF
pipelines repeatedly update preference models on fresh policy data \citep{bai2022training,sun2025ugda}.
This paper shows that this loop can become systematically wrong \emph{because the policy improved}.

Let success require $y=F\wedge M$ and let $s_F$ be a verifier intended to score $F$. Among
failures, some violate $F$ and others already satisfy $F$ but fail only $M$. Pairwise additivity
of AUC gives
\begin{equation}
\operatorname{AUC}(s_F;y)=(1-q)A_{\rm own}+qA_{\rm cross},\qquad
q=\Pr[F=1\mid y=0].
\label{eq:intro_mix}
\end{equation}
Here $A_{\rm own}$ is the audit's \emph{own-failure separability}: target successes versus
$F$-failures. It is not, in general, identical to canonical $F$-AUC; Section~\ref{sec:setup}
separates the two estimands explicitly. $A_{\rm cross}$ compares successes with failures that already satisfy $F$;
an exact $F$-oracle has $A_{\rm cross}=1/2$ by construction. When optimization increases satisfaction of $F$ faster than the remaining conjunct---in particular as $F$ saturates while $M$ remains imperfect---$q$ rises. Consequently, the standard aggregate monitor progressively removes weight from
the verifier's own failure mode and places it on information about the \emph{other} conjunct.
The paradox is therefore directional: \textbf{the more completely the policy masters a component,
the less the usual on-policy audit measures that component.} We call this \emph{on-policy audit
collapse}: optimization removes the very failure mode needed to evaluate the module that learned it.

This creates a concrete failure chain. First, the monitor loses identification: a falling AUC need
not imply a worse component verifier. Second, it can become actively anti-aligned with the audit's intended subproblem: at a calculable
$q^\ast$ it prefers a head with worse $A_{\rm own}$ because that head leaks information about
$M$; in our text and engine controls this can also select a head with worse canonical $F$-AUC. Third, the usual response to a bad monitor---on-policy refresh---can remove component
competence. Once optimization has depleted $F$-failures, refresh data alone no longer identify the
boundary the component verifier must retain; outcome labels can actively entangle it, while even
correct $F$ labels cannot reconstruct failures absent from support. The result can look benign on the
monitored distribution while becoming worse exactly on the off-support failures a constraint checker
must catch after a generator, prompt distribution, or policy shift.

\begin{center}
\begin{tikzpicture}[node distance=2.4mm,>=Latex,every node/.style={font=\scriptsize,align=center}]
\node[draw,rounded corners,inner sep=3pt] (a) {policy learns $F$};
\node[draw,rounded corners,inner sep=3pt,right=of a] (b) {$F$-failures vanish\\$q\!\uparrow$};
\node[draw,rounded corners,inner sep=3pt,right=of b] (c) {audit changes meaning\\cross-signal wins};
\node[draw,rounded corners,inner sep=3pt,right=of c] (d) {refresh can\\forget $F$};
\draw[->] (a)--(b); \draw[->] (b)--(c); \draw[->] (c)--(d);
\end{tikzpicture}
\end{center}
Viewed as a learning system, this is a control-plane failure: optimizing the policy changes the data
used to decide whether its verifier should be replaced or refreshed.

\paragraph{When does the failure apply?}
The central result requires three ingredients: a head intended for component $F$, an \emph{on-policy}
audit against the full outcome $y$ rather than a component-specific label, and use of that audit to
compare, early-stop or trigger refreshes. If trusted $F$ labels are available and the component is
audited directly (preferably on a reference distribution that preserves $F$-failures), the collapse in
Eq.~\ref{eq:intro_mix} is avoided. Our claim is therefore about this specific but natural interface,
not about every verifier evaluation protocol.

\paragraph{Why this interface is current.}
Its ingredients already coexist in modern post-training systems. Rubric RL independently scores
prompt-specific criteria and aggregates them into a scalar on-policy reward
\citep{gunjal2026rubrics,mahmoud2026rubric}; policy-aware rubric work reports that substantial
criterion mass is already saturated or dead under current policies \citep{tyagi2026policyaware}; and
fresh policy-interaction data are used to refine reward models or preference models
\citep{sun2025ugda,bai2022training}. We do \emph{not} claim that these systems implement our exact
audit-trigger-refresh controller end to end. They establish that criterion-level verifiers, policy-induced
saturation, scalar aggregation and on-policy refinement are contemporary ingredients; Appendix~\ref{app:workflow_relevance}
maps the interfaces explicitly.

Our experiments are organized to test this chain rather than to accumulate unrelated benchmarks.
The engine domain gives an exact component oracle, allowing us to separate real verifier failure
from a case-mix artifact. GSM8K supplies a familiar language-domain negative control for ranking
inversion. A natural multi-label GoEmotions stress test uses separately annotated semantic criteria to test
audit collapse, inversion and refresh under support depletion; a constructed AG News control is
reported in the appendix. ECG then stresses the operational
consequence on a frozen external benchmark with three generator families, subject-held-out splits,
two refresh model classes, downstream best-of-$16$ selection, and a precommitted set of
predictions. Across these settings the observations line up with the mechanism: exact-oracle AUC
collapses, the analytic inversion threshold predicts where inversion does and does not occur,
optimization removes own-failure support, outcome refresh damages off-policy canonical $F$-AUC and
downstream constraints in every paired split, and deliberately restoring failure support prevents
the damage.

\paragraph{Contributions.}
\textbf{(1) On-policy audit collapse.} We derive the mixture law, a sharp identified set, and the
complete-erasure threshold $q\ge\max(A,1-A)$ (Section~\ref{sec:law}).
\textbf{(2) Audit inversion.} We give the threshold $q^\ast$ at which cross-component signal
overtakes own-failure separability and show that the aggregate audit's population maximizer is the
full-outcome predictor, not the component oracle (Proposition~\ref{prop:inversion}).
\textbf{(3) The refresh trap.} A controlled semantic-language stress test and a prediction-registered
ECG study separate label mismatch from support mismatch and connect refresh to downstream harm
(Section~\ref{sec:refresh}).
\textbf{(4) Support-preserving repair.} We give matching passive label-complexity bounds and a
designed two-mixture audit whose error does not depend on the deployed policy's saturation
(Section~\ref{sec:fix}). Downstream selection/training-interface results are deliberately kept as a
boundary analysis in the appendix.

\section{Mechanism-first evaluation}
\label{sec:setup}

We test the chain implied by Eq.~\ref{eq:intro_mix}, not a collection of interchangeable
benchmarks. Engine supplies an exact oracle and matched pools for the inversion threshold; GSM8K is
a language-domain negative control where inversion should not occur; GoEmotions supplies two
separately annotated semantic criteria; and ECG is the frozen operational stress test with three
generator families, subject-held-out evaluation, two refresh model classes and a downstream gate.
Protocols and domain-specific details are in the appendices.

In every domain $y=F\wedge M$. We keep two component estimands distinct. The theory isolates
\emph{own-failure separability}
$A_{\rm own}=\operatorname{AUC}(s_F;F{=}1,M{=}1\text{ vs }F{=}0)$, because these are exactly the
outcome-audit pairs whose negative is an $F$-failure. Canonical component quality is instead
$A_F=\operatorname{AUC}(s_F;F{=}1\text{ vs }F{=}0)$. They need not agree: if
$\rho=\Pr[M=1\mid F=1]$ and
$A_{\rm aux}=\operatorname{AUC}(s_F;F{=}1,M{=}0\text{ vs }F{=}0)$, then
\begin{equation}
A_F=\rho A_{\rm own}+(1-\rho)A_{\rm aux}.
\label{eq:canonical_component_auc}
\end{equation}
Equation~\ref{eq:canonical_component_auc} is just the law of total probability over $M$ within the
canonical positive class $F=1$; no independence assumption is used.
Propositions~\ref{prop:identified_set}--\ref{prop:inversion} characterize what the \emph{outcome}
audit identifies and rewards about $A_{\rm own}$; whenever we claim actual component degradation we
report $A_F$ separately. Best-of-$K$ supply, selection efficiency and Answer@$K$ are defined and
audited in Appendix~\ref{app:moved}.

\section{The paradox: policy improves, audit collapses}
\label{sec:signature}

\begin{table}[b]
\centering
\caption{The policy improves under the oracle target while the component head's full-target AUC
falls toward chance. The component is feasibility for engine and boxed format for GSM8K. Engine:
$K=16$, three seeds, $1\,600$ pairs (the original round; the matched study of
Appendix~\ref{sec:train} uses $4\,000$ pairs). GSM8K: $500$ prompts, $K=16$ prefix of $64$, three
seeds.}
\label{tab:signature}
\small
\setlength{\tabcolsep}{2.8pt}
\renewcommand{\arraystretch}{1.05}
\begin{adjustbox}{max width=\linewidth}
\begin{tabular}{llcccccc}
\toprule
 & policy & target/draw & component/draw & supply@16 & full-target AUC & native selector@16 & $\eta$ \\
\midrule
engine & base & $0.026$ & $0.26$ & $0.285$ & $0.802$ & $0.093$ & $0.26$ \\
 & after DPO & $\mathbf{0.095 \pm 0.006}$ & $0.87$ & $0.460 \pm 0.015$ & $\mathbf{0.512 \pm 0.010}$ & $0.112 \pm 0.008$ & $0.05 \pm 0.04$ \\
\midrule
GSM8K & base & $0.302$ & $0.42$ & $0.872$ & $0.919$ & $0.734$ & $0.76$ \\
 & after DPO & $\mathbf{0.647 \pm 0.006}$ & $0.98$ & $0.933 \pm 0.005$ & $\mathbf{0.506 \pm 0.004}$ & $0.673 \pm 0.014$ & $0.09 \pm 0.03$ \\
\bottomrule
\end{tabular}
\end{adjustbox}
\end{table}

The first empirical fact is intentionally paradoxical. In engine, oracle target success per draw
rises $3.6\times$ while the feasibility head's aggregate AUC falls from $0.80$ to $0.51$. In
GSM8K, target success rises $2.1\times$ while the format head falls from $0.92$ to $0.51$. At the
same time the native selectors fail to convert the newly available supply into selected success:
selection efficiency falls to $0.05$ and $0.09$. If aggregate on-policy AUC were a direct measure
of component quality, these numbers would suggest that both components deteriorated sharply just
as the policies learned to satisfy them.

\paragraph{The exact-oracle control rules out that interpretation.}
The component itself remains available. Two unused engine critics with different seeds show the
same on-policy collapse ($0.45$--$0.58$) while their mass heads stay at $0.96$--$0.98$; the
feasibility head retains own-property AUC $0.940$--$0.991$ on matched DPO pools, and the GSM8K
format head retains $0.995$ accuracy against unboxed failures. More strongly, the next section
shows that an \emph{exact} component oracle exhibits the same aggregate symptom. Thus the collapse
cannot, by itself, identify component degradation. It is the empirical signature predicted by a
changing failure mixture.

\section{On-policy audit collapse}
\label{sec:law}

For one candidate pool $r$, every negative example belongs to one of two semantic
failure bins: it either fails the scored condition $F$, or it satisfies $F$ and fails
$M$. Pairwise additivity of AUC therefore gives
\begin{equation}
A_r=(1-q_r)A_{{\rm own},r}+q_rA_{{\rm cross},r},\qquad
q_r=\Pr[F=1\mid y=0,r].
\label{eq:mix}
\end{equation}
Here $A_{{\rm own},r}$ is the audit's own-failure separability: target successes versus failures
of $F$. Equation~\ref{eq:canonical_component_auc} distinguishes this audit estimand from canonical
$F$-AUC. The second term,
$A_{{\rm cross},r}$, asks the same head to order target successes against failures
that already satisfy $F$. As $q_r$ grows, the aggregate audit places less weight on
the head's own subproblem.

All statements below are per pool; because
$\bar A\neq(1-\bar q)\bar A_{\rm own}+\bar q\bar A_{\rm cross}$ in general, every
reconstruction averages the per-pool identity (Section~\ref{sec:setup}).

\begin{proposition}[Identified set for own-failure separability]
\label{prop:identified_set}
Partition negatives into semantic failure bins $B_1,\ldots,B_J$, with
$\pi_j=\Pr[B_j\mid y=0]$ and
$a_j=\operatorname{AUC}(s;y=1\text{ vs }B_j)$. Given only aggregate AUC $A$ and the
observed bin weights $\pi$, the sharp identified set is
\[
\mathcal I(A,\pi)=\{a\in[0,1]^J:\pi^\top a=A\}.
\]
For the two-bin split in Eq.~\ref{eq:mix}, if $q<1$,
\begin{equation}
A_{\rm own}\in\left[
\max\!\left(0,\frac{A-q}{1-q}\right),
\min\!\left(1,\frac{A}{1-q}\right)\right].
\label{eq:identified_interval}
\end{equation}
This interval is all of $[0,1]$ exactly when
\begin{equation}
q\ge\max(A,1-A)=\tfrac12+\left|A-\tfrac12\right|.
\label{eq:erasure_threshold}
\end{equation}
The bounds are attainable. If the bin-conditional score distributions are held fixed
and aggregate AUC is measured under known mixture rows of $\Pi$, then the vector
$a$ is point-identified for every compatible value iff
$\operatorname{rank}(\Pi)=J$; with two bins, two distinct mixture weights suffice. If the
scores, but not the bins, of all negatives are also observed, the sharp interval for
$A_{\rm own}$ becomes a trimming bound \citep{horowitz1995identification}: the means of
the lowest and of the highest $(1-q)$-fraction of the values $h(x)$, the share of
positives ranked above negative $x$ (Appendix~\ref{app:v24}).
\end{proposition}

\paragraph{More than two conjuncts.}
For $J>2$ semantic failure modes the same object is simply $A=\sum_j\pi_j a_j$. Mastering one
component shrinks the weight of its own-failure bin and reallocates audit mass to whichever other
failure modes remain. There is no privileged scalar $q$ in the multi-bin case, but the identification
question is unchanged: the auditor needs either direct bin labels or enough deliberately distinct
mixture rows for the resulting mixture matrix to have full column rank.

\paragraph{What complete erasure means.}
Equation~\ref{eq:erasure_threshold} is stronger than saying that AUC is sensitive to
case mix. Even after the auditor observes the mixture weight $q$ exactly, the
aggregate statistic can contain \emph{zero} information about $A_{\rm own}$. Every
matched engine row in Table~\ref{tab:law} is in this regime: $q\ge0.916$ and
aggregate AUC $0.49$--$0.58$ are compatible with every $A_{\rm own}\in[0,1]$, while
direct bin audits measure $0.94$--$0.99$.

\paragraph{What does a realistic auditor know?}
An auditor usually has outcome labels and scores for every candidate but failure
causes only on a labelled subsample, from which $q$ is estimated. This extra information
barely helps: on the ECG base test pools, observing every negative's score lowers the
share of erased pools only from $37\%$ to $34\%$ (Llama) and from $48\%$ to $48\%$
(Phi), and estimating $q$ from $50$ labels leaves $36\%$ and $48\%$, because in small
pools most $h$-values sit at $0$ or $1$ (Appendix~\ref{app:v24}).

The same point holds without a learned verifier. The exact oracle $s_F=F$ has
$\operatorname{AUC}(F;y)=1-q/2$: on stored pools this perfect head falls from $0.833$ to
$0.515$ (engine) and from $0.922$ to $0.519$ (GSM8K), realizing a composition-only path
($A_{\rm own}=1$, $A_{\rm cross}=1/2$). Section~\ref{sec:scope} gives the opposite
control, genuine critic exploitation, with the same aggregate symptom.

\begin{table}[b]
\centering
\caption{Chance-level aggregate AUC can coexist with near-perfect ordering of the
component's own failures. $q$ is the share of off-target candidates that already
satisfy the scored conjunct. ECG rows are frozen pre-DPO test pools; Gemma's has no
informative target-positive requirement and is omitted rather than imputed.}
\label{tab:law}
\small
\setlength{\tabcolsep}{4pt}
\begin{tabular}{lcccc}
\toprule
policy & $q$ & aggregate AUC & $A_{\rm own}$ & $A_{\rm cross}$ \\
\midrule
1.5B sum & $0.970 \pm 0.007$ & $0.532 \pm 0.004$ & $0.966 \pm 0.020$ & $0.519 \pm 0.007$ \\
1.5B hard gate & $0.916 \pm 0.014$ & $0.535 \pm 0.006$ & $0.944 \pm 0.016$ & $0.494 \pm 0.001$ \\
7B sum & $0.984 \pm 0.005$ & $0.494 \pm 0.024$ & $0.987 \pm 0.017$ & $0.487 \pm 0.025$ \\
7B hard gate & $0.953 \pm 0.011$ & $0.583 \pm 0.007$ & $0.962 \pm 0.037$ & $0.561 \pm 0.004$ \\
\midrule
GSM8K base & $0.16$ & $0.919$ & $0.997$ & $0.507$ \\
GSM8K sum DPO & $0.96$ & $0.506\pm0.004$ & $0.995$ & $0.489\pm0.006$ \\
\midrule
ECG Llama base & $0.611$ & $0.751$ & $1.000$ & $0.607$ \\
ECG Phi base & $0.666$ & $0.791$ & $1.000$ & $0.690$ \\
\bottomrule
\end{tabular}
\end{table}

\begin{figure}[t]
\centering
\includegraphics[width=.88\linewidth]{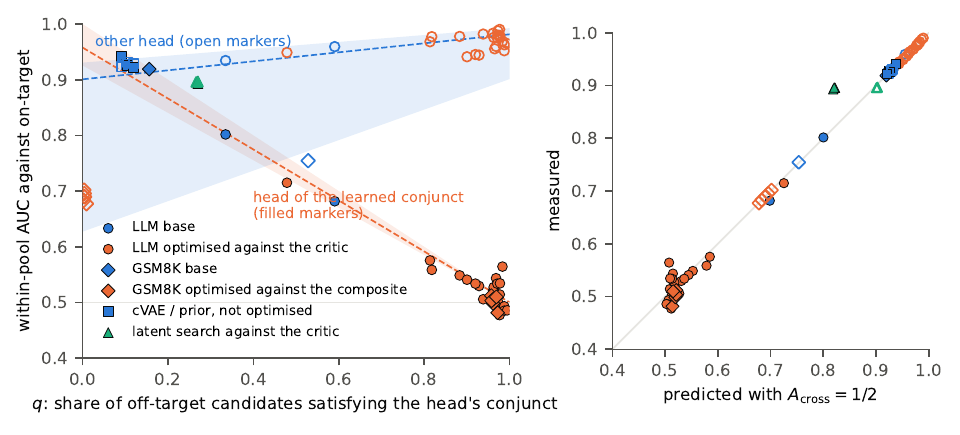}
\caption{\textbf{Audit collapse.} Left: aggregate outcome AUC changes systematically with the
failure mixture $q$ even when the component's own-bin ordering remains strong. Right: the
zero-parameter $A_{\rm cross}=1/2$ reconstruction tracks measured AUC. The complete-erasure region
of Eq.~\ref{eq:erasure_threshold} is shown separately in Appendix~\ref{app:identifiability_proof}.}
\label{fig:law}
\end{figure}

\paragraph{External check in a non-extreme regime.}
On ECG the deployability head is effectively exact ($A_{\rm own}=1.000$), yet the
aggregate reads $0.751$ and $0.791$ at $q=0.61$--$0.67$ (Table~\ref{tab:law}); the mean
per-pool interval from $(A_r,q_r)$ is $[0.47,1]$ and $[0.45,1]$, with $37\%$ and $48\%$ of pools erased.

\subsection{Audit inversion: cross-component signal can win}
Non-identification says the aggregate cannot certify a component. Equation~\ref{eq:mix} says
more: it states what the aggregate \emph{rewards} when used to compare, select or refresh
verifiers.

\begin{proposition}[Audit inversion and its maximizer]
\label{prop:inversion}
Fix a pool with mixture $q$. \textbf{(i)} For two heads $f,g$,
$A(g)-A(f)=(1-q)(a^{\rm own}_g-a^{\rm own}_f)+q(a^{\rm cross}_g-a^{\rm cross}_f)$. If
$f$ has better own-failure separability ($\Delta_{\rm own}=a^{\rm own}_f-a^{\rm own}_g>0$) but $g$
separates $F$-satisfying candidates by $M$ better ($\Delta_{\rm cross}=a^{\rm cross}_g-a^{\rm cross}_f>0$),
the audit prefers $g$ iff
\[
q>q^\ast=\frac{\Delta_{\rm own}}{\Delta_{\rm own}+\Delta_{\rm cross}},
\]
and for $\Delta_{\rm own}=0$ it prefers $g$ at every $q>0$. \textbf{(ii)} Over all scores
$s(x)$ of the candidate, the pool's aggregate AUC is maximized by any strictly increasing
transform of $\eta_y(x)=\Pr[y=1\mid x]$. \textbf{(iii)} If $F=\phi(x)$ is computable from the
candidate, then $a^{\rm own}(\eta_y)=1$, and the exact oracle $s=\phi$ satisfies
$A(\eta_y)-A(\phi)=q\,\bigl(a^{\rm cross}(\eta_y)-\tfrac12\bigr)\ge0$: the outcome predictor weakly
dominates the perfect component verifier at every $q$, strictly unless $\eta_y$ is constant on
$F$-satisfying candidates.
\end{proposition}

Part (i) is Eq.~\ref{eq:mix} differenced; (ii) is the Neyman--Pearson property of AUC
\citep{clemencon2008ranking}; (iii) follows because $y=1$ implies $F=1$, so $\eta_y=0$ on every
$F$-failure while the oracle is constant on $\{F=1\}$ (Appendix~\ref{app:refresh}). The content
is not the arithmetic but its use. If aggregate on-policy AUC is used to \emph{compare or early-stop}
component heads, its selection pressure favors scores increasingly aligned with $\eta_y$, hence with
information about the rest of the target. A subsequent outcome-labelled refresh is a distinct
training intervention, studied in Section~\ref{sec:refresh}; it need not literally optimize AUC to
inherit the same entangling signal. The within/cross split mirrors the xAUC decomposition used for
group fairness \citep{kallus2019xauc}; here the groups are semantic failure bins and the question is
which head the audit statistic prefers.

\paragraph{The threshold predicts where inversion appears.}
On matched 1.5B sum-DPO engine pools the learned feasibility head has
$(a^{\rm own},a^{\rm cross})=(0.966,0.519)$ and the exact oracle $(1,\tfrac12)$, so
$q^\ast=0.64$; at the observed $q=0.970$ the audit ranks the oracle \emph{below} the imperfect
head ($0.515$ vs $0.532$). The oracle also has canonical $F$-AUC $1$ by definition, while direct
feasibility AUC of the learned heads is $0.940$--$0.991$ on the matched DPO pools
(Appendix~\ref{app:moved}), so this is not merely a relabelling of $A_{\rm own}$. On GSM8K the format head's
$a^{\rm cross}=0.489<\tfrac12$, so no $q^\ast$ exists, and indeed the oracle stays above it
($0.518$ vs $0.506$; Appendix~\ref{app:gsm}). A holistic on-target engine classifier, the
extreme of entanglement, avoids the component-AUC cliff yet captures only $\eta=0.49$--$0.60$ of
the best-of-$K$ lift versus $0.90$--$0.94$ for the gated two-head critic (Appendix~\ref{app:moved}):
the head favored by the aggregate audit is not the best component representation for the selector.

\paragraph{Optimization drives audits into this regime.}
On the post-DPO ECG test pools, scored by the same frozen critic, the within-pool
own-failure bin of the deployability head is empty in every evaluable pool of the four
evaluable Llama gate/inference runs and in $64$--$72\%$ of evaluable pools of the six
Phi gate/inference runs, up from $13\%$ and $36\%$ at base. There no passive audit recovers $A_{\rm own}$ from the deployed mixture (Section~\ref{sec:fix}), and the on-policy refresh data no longer identify how the head should rank the missing own-failure mode. The next section tests what happens when a refresh procedure acts on that depleted support.

\section{The refresh trap: acting on audit collapse}
\label{sec:refresh}

Proposition~\ref{prop:inversion} describes what the aggregate audit prefers. Deployment makes the
problem operational: a low monitor value can trigger collection of fresh policy samples and a
verifier refresh. The mixture law predicts two distinct hazards. \textbf{Label mismatch:} if a
component verifier is refreshed against the full outcome $y$, optimization explicitly moves it
toward the audit's preferred entangled predictor. \textbf{Support mismatch:} even with the correct
component labels $F$, a saturated policy may provide almost no $F$-failures, so the refresh data
contain little or no information about the decision boundary the verifier must retain. The second
hazard means the problem is not fixed simply by choosing the ``right'' label.

\paragraph{Natural multi-label semantic stress test.}
On GoEmotions \citep{demszky2020goemotions}, keep $F$ as absence of \emph{annoyance} and use a
post-hoc, transparently reported high-$q$ positive control with $M$ = absence of \emph{gratitude}.
At the naturally observed $q=.522$, the component specialist has canonical $F$-AUC $.679$ but an
aggregate audit of only $.442$. The deliberately entangled score $z(s_F)+z(s_M)$ has substantially
\emph{worse} canonical $F$-AUC ($.570$) yet the aggregate audit prefers it ($.681$; bootstrap
95\% difference $[.210,.270]$), with predicted crossing $q^*=.126$. Sweeping leakage strength
$s_\lambda=z(s_F)+\lambda z(s_M)$ produces the same quantitative tradeoff: as canonical $F$-AUC
falls from $.668$ to $.483$, the observed aggregate audit rises from $.470$ to $.725$.
Support depletion is equally directional. With correct $F$ labels, canonical $F$-AUC is $.545$ when
only eight own failures remain; adding 80 random development examples barely changes it ($.550$),
whereas adding 80 targeted own failures raises it to $.606$ (full support $.679$). The pair choice,
leakage sweep, random-vs-targeted replay and the original near-boundary disapproval control are all
post hoc and reported in full in Appendix~\ref{app:goemotions}; AG News supplies an independent
constructed replication (Appendix~\ref{app:agnews}).

Immediately before this study, the post-DPO ECG pools already exhibit the predicted support shift:
the own-failure bin of the deployability head is empty in every evaluable pool of the four Llama
gate/inference runs and in $64$--$72\%$ of evaluable pools of the six Phi runs, versus $13\%$ and
$36\%$ at base. We therefore test not merely whether the monitor changes, but whether acting on it
causes measurable off-policy and downstream harm.

\paragraph{Protocol: held-out subjects, two refresh classes, and an unseen generator.}
We use the frozen ECG pools, where $F$ is a hard deployment constraint: a valid pipeline within
latency, cycle and RAM budgets. For Llama and Phi, a specialist $S$ is trained on the base-policy
train pool with $F$ labels, in two head classes: an MLP with the frozen critic's architecture and
gradient-boosted trees. Refresh data are post-DPO test pools from one half of test subjects;
every reported refresh metric is evaluated on the complementary subject half over $10$ random
subject splits. $R_y$ refreshes with outcome labels, $R_F$ with component labels, and
$R_{\rm mix}$ with component labels plus an equal number of re-injected base-policy candidates.
We evaluate four levels of evidence: the on-policy monitor, off-policy $F$-AUC on the source
generator, transfer $F$-AUC on unseen Gemma outputs, and a best-of-$16$ gated consumer. Six
predictions were hashed before the run (Appendix~\ref{app:refresh}); an exploratory pilot using a
different head procedure and split seed had informed P1--P3, while P4--P6 were untested. The four
family$\times$head-class configurations are the primary experimental settings; the 10 overlapping
subject splits within each setting are robustness resamples, so counts such as $40/40$ below report
sign consistency rather than 40 independent replications.

\begin{table}[t]
\centering
\caption{Refreshing a component verifier on on-policy ECG data (subject-held-out half; mean of
$10$ paired splits; each cell MLP~/~GBM). The monitor is full-target on-policy AUC. $F$-AUC is
measured off-policy on base pools from the source generator and unseen Gemma. The gate column tests
whether verifier damage changes a downstream best-of-16 decision; violations up to about $.15$ are
unavoidable because some pools contain no feasible candidate.}
\label{tab:refresh}
\scriptsize
\setlength{\tabcolsep}{2.6pt}
\begin{adjustbox}{max width=\linewidth}
\begin{tabular}{llccccc}
\toprule
 & head & monitor AUC & $F$-AUC, base & $F$-AUC, unseen gen. & gate $F$-violation & gate Answer@16 \\
\midrule
Llama & $S$ (deployed) & .545 / .614 & 1.000 / 1.000 & .998 / 1.000 & .152 / .151 & .495 / .495 \\
 & $R_y$: outcome refresh & .649 / .564 & .883 / .567 & .708 / .493 & .163 / .206 & .475 / .447 \\
 & $R_F$: component refresh & .552 / .494 & .999 / .765 & .996 / .775 & .153 / .177 & .491 / .478 \\
 & $R_{\rm mix}$: $R_F$ + base failures & .551 / .617 & 1.000 / 1.000 & .998 / 1.000 & .152 / .151 & .495 / .495 \\
\midrule
Phi & $S$ (deployed) & .641 / .507 & 1.000 / 1.000 & .999 / .999 & .145 / .144 & .565 / .565 \\
 & $R_y$: outcome refresh & .682 / .495 & .908 / .519 & .635 / .519 & .163 / .184 & .535 / .508 \\
 & $R_F$: component refresh & .674 / .618 & .994 / 1.000 & .999 / .999 & .144 / .144 & .564 / .565 \\
 & $R_{\rm mix}$: $R_F$ + base failures & .671 / .762 & 1.000 / 1.000 & .999 / .998 & .144 / .144 & .565 / .565 \\
\bottomrule
\end{tabular}
\end{adjustbox}
\end{table}

\begin{figure}[t]
\centering
\includegraphics[width=\linewidth]{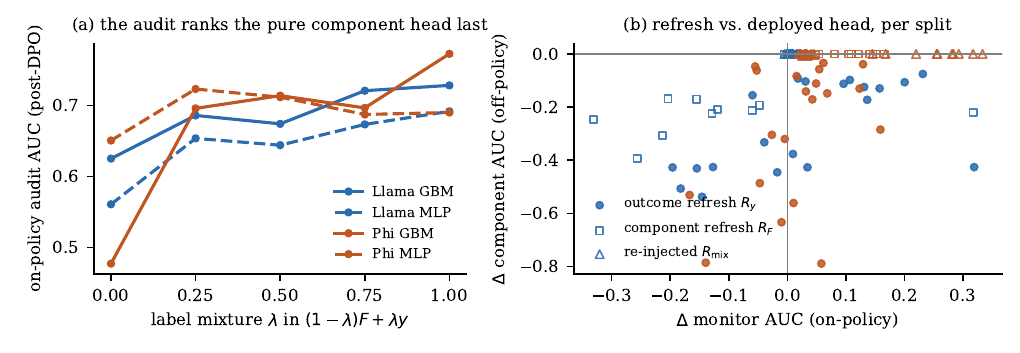}
\caption{\textbf{(a) The audit prefers entanglement.} Heads trained on base data with labels
$(1-\lambda)F+\lambda y$: the pure component head ($\lambda=0$), which has canonical $F$-AUC $1.000$, receives the lowest on-policy audit AUC in all four family$\times$class settings (P1).
\textbf{(b) Acting on the audit damages the component.} Refresh minus deployed specialist in each
of $10$ held-out-subject splits and two head classes. Outcome refresh lowers off-policy canonical $F$-AUC in every paired split while its effect on the monitor changes sign across settings; restoring
base failures leaves canonical $F$-AUC essentially unchanged.}
\label{fig:refresh}
\end{figure}

\paragraph{Result 1: the monitor selects the wrong representation.}
P1 held in all four family$\times$class settings. Among heads trained from the same base data with
label mixture $(1-\lambda)F+\lambda y$, the pure component head ($\lambda=0$) had the \emph{lowest}
on-policy aggregate AUC, despite $F$-AUC $1.000$. This is the empirical analogue of
Proposition~\ref{prop:inversion}: adding outcome information improves the audit score while moving
the representation away from a pure component verifier.

\paragraph{Result 2: outcome refresh causes systematic off-policy forgetting.}
P3 and P5 held. Relative to the deployed specialist, $R_y$ lowered source-generator off-policy
$F$-AUC with the same sign in every held-out-subject split across both families and both head
classes ($40/40$ paired evaluations; sign robustness, not independent replications): by $0.09$--$0.12$ for MLPs and $0.43$--$0.48$ for GBMs. The effect is not confined to the refresh
generator: on unseen Gemma it drops $0.29$--$0.51$. By contrast, the monitor itself is not a
consistent indicator of this damage: its mean change is positive in two settings and negative in
two, and the preregistered P2 prediction that Phi's monitor would improve by at least $0.05$ in
both classes failed ($+0.042$, $-0.012$). The intervention can therefore destroy the component
without producing a reliable improvement even in the statistic that motivated it.

\paragraph{Result 3: the forgotten component changes downstream decisions.}
P6 held. Using $R_y$ as the feasibility gate of a best-of-$16$ selector increases constraint
violations with the same sign in every held-out-subject split across both families and both head
classes ($40/40$ paired evaluations; sign robustness, not independent replications), by $+0.011$ to $+0.055$ above the unavoidable floor, and
lowers Answer@16 by $0.02$--$0.06$. This closes the loop from an apparently
bad audit to an operationally worse consumer: the effect is not only a metric disagreement.

\paragraph{Result 4: correct labels are insufficient; restored support is sufficient here.}
$R_F$ isolates support mismatch from label mismatch. It still loses $0.235$ source-generator
$F$-AUC in the Llama-GBM setting and admits more violations, despite being refreshed with the
correct component labels. In contrast, P4 held for $R_{\rm mix}$: re-injecting base-policy
candidates while keeping component labels preserves $F$-AUC at $1.000$ in every source setting,
keeps unseen-generator AUC at $0.998$--$1.000$, and returns the gate metrics to the deployed
specialist. The difference between $R_F$ and $R_{\rm mix}$ is support, not supervision.

\paragraph{Closed-loop counterfactual.}
Composing the frozen trigger signal with the already-trained refresh heads closes the loop without
new training. Refreshing with $R_y$ whenever aggregate audit AUC falls by at least $.05$ from base
triggers all four primary settings even though each deployed specialist retains $F$-AUC $1.000$;
the resulting mean $F$-AUC/violation/Answer@16 changes from $1.000/.148/.530$ to
$.719/.179/.491$. Thresholds $.10$--$.25$ trigger fewer settings but remain worse than no refresh
on all three outcomes; a reference-distribution $F$-specific trigger fires in none. The full trigger--harm
curve is shown in Appendix~\ref{app:controller}.

\paragraph{Empirical conclusion.}
Across the study, five of six registered predictions held. More importantly, the evidence follows
the mechanism at four levels: the audit rewards entanglement, optimization removes own-failure
support, refresh causes held-out and cross-generator component forgetting, and that forgetting
raises downstream violations. The failed prediction is informative rather than cosmetic: the
monitored AUC is not reliably repaired by refresh, strengthening the narrower conclusion that it
is an unsuitable control signal. The successful rescue identifies the actionable fix used in the
next section: preserve or deliberately reconstruct the component's own failure distribution.

\section{Repair: restore failure support}
\label{sec:fix}

\paragraph{When can the audit recover component behavior?}
Three routes exist: condition directly on semantic failure bins when they have
support; impose an externally justified restriction $A_{\rm cross}\in[\ell,u]$, which
for $q<1$ gives
\begin{equation}
A_{\rm own}\in\left[
\max\!\left(0,\tfrac{A-qu}{1-q}\right),
\min\!\left(1,\tfrac{A-q\ell}{1-q}\right)\right];
\label{eq:sensitivity_interval}
\end{equation}
or vary the failure mixture with bin-conditional distributions fixed, which by
Proposition~\ref{prop:identified_set} reduces recovery to the rank of the mixture
matrix and also restores support lost to saturation
(Appendix~\ref{app:identifiability_proof}). \textbf{This is a designed reweighting/resampling
operation, not a license to treat arbitrary policy checkpoints as mixture rows:} if either
bin-conditional score distribution changes, naive two-row inversion is biased. Appendix~\ref{app:transfer}
gives the explicit drift term and an empirical boundary case where $A_{\rm own}$ shifts from $.965$
to $.821$ out of distribution and the transported forecast misses by $.101$.
Across two policies, the AUC change splits exactly into a mixture-shift term and
conditional-ordering drift, so both mechanisms can coexist
(Eq.~\ref{eq:change}, Appendix~\ref{app:moved}).

\paragraph{An audit card.}
When aggregate outcome AUC changes for a component head: \textbf{(1)} define semantic failure bins and report
their weights and per-pool identified intervals (Eq.~\ref{eq:identified_interval}) with
a cluster-bootstrap outer interval, flagging pools with $q\ge\max(A,1-A)$;
\textbf{(2)} if the interval is uninformative, label $m$ random off-target candidates
and estimate the pair-weighted within-pool $A_{\rm own}$ directly; \textbf{(3)} if the
own-failure bin lacks support, construct two controlled negative mixtures by reweighting or
resampling fixed bin-conditional candidates, then invert the $2\times2$ system of
Eq.~\ref{eq:mix}; \textbf{(4)} do not use aggregate on-policy outcome AUC as the sole criterion to select, early-stop or refresh a
component verifier (Proposition~\ref{prop:inversion}); refresh only on
data that contain its own failures and report its quality off-policy (Section~\ref{sec:refresh});
\textbf{(5)} audit separately the interface that consumes the score (Appendices~\ref{sec:select}--\ref{sec:train}).
The next result prices steps 2 and 3.

\begin{proposition}[Cost of restoring identification]
\label{prop:audit_cost}
Fix a pool's positives and let $h(x)\in[0,1]$ be the fraction of them the head ranks
above a negative $x$ (ties $\tfrac12$), so $A_{\rm own}=\mathbb E[h(X)\mid X\in
B_{\rm own}]$. \textbf{(i) Passive audit.} Label $m$ negatives drawn i.i.d.\ from the
pool's failure mixture and average $h$ over the revealed own failures. If
$m(1-q)\ge\max\{8\log(2/\delta),\,\log(4/\delta)/\varepsilon^2\}$, the estimate is within
$\varepsilon$ of $A_{\rm own}$ with probability at least $1-\delta$. Conversely, for
$\varepsilon\le\tfrac14$ and $\delta<\tfrac14$, every estimator built from such a sample
errs by at least $\varepsilon$ with probability above $\delta$ on some instance unless
$m(1-q)\ge\max\{q\log(1/(2\delta)),\,\log(1/(4\delta))/(24\varepsilon^2)\}$, so the two
budgets match up to constants. \textbf{(ii) Designed audit.} If aggregate AUCs are
estimated within $\varepsilon_1,\varepsilon_2$ at known weights $q_1\ne q_2$, inverting
Eq.~\ref{eq:mix} gives $|\hat A_{\rm own}-A_{\rm own}|\le(q_2\varepsilon_1+q_1
\varepsilon_2)/|q_2-q_1|$, independent of the deployed policy's $q$.
\end{proposition}

The passive audit thus becomes expensive exactly where Proposition~\ref{prop:identified_set}
makes the aggregate uninformative, and impossible once the own bin is empty; a designed
mixture removes the dependence (proofs in Appendix~\ref{app:v23}). On the ECG utility
head, $\sqrt{m(1-q)}\cdot\mathrm{RMSE}$ stays at $0.20$--$0.24$ for every $q\le0.95$,
$m=100$ labels give $96$--$98\%$ subject-bootstrap coverage, and two-mixture inversion
reaches RMSE $0.01$--$0.02$ at $|q_2-q_1|=0.6$ (Appendices~\ref{app:theory},~\ref{app:v23}).
The same designed mixture is what protects a refresh: $R_{\rm mix}$ in Table~\ref{tab:refresh}
differs from $R_F$ only by re-injected base-policy failures and keeps canonical $F$-AUC at $1.000$ in
every setting, and re-injection likewise recovers $A_{\rm own}$ with RMSE $\le0.02$ on post-DPO ECG
pools whose own bin is empty (Appendix~\ref{app:v23}).

\section{Scope: what the results do and do not establish}
\label{sec:scope}

The AUC symptom does not identify its cause: an exact oracle collapses and an unbounded
autoencoder genuinely exploits the critic (feasibility $0.9999$ predicted, $1.8\%$ actual;
Appendix~\ref{app:escape}); this is not a claim that in-bounds exploitation is absent.
Proposition~\ref{prop:inversion} is exact, but its engine and GSM8K checks reuse published
quantities and are post-hoc. The GoEmotions study is a post-hoc natural multi-label text stress test
with a small bag-of-words verifier; the AG News experiment is a more constructed control with
24 deterministic robustness splits rather than independent replications. Together they establish that
the mechanism is not specific to engineered numeric features, not that our refresh effect has been
measured with an LLM-scale semantic verifier. Contemporary rubric-RL work independently shows
criterion-level saturation and scalar aggregation at LLM scale \citep{gunjal2026rubrics,tyagi2026policyaware};
we use that literature only to establish interface relevance, not as validation of our inversion or refresh
claims. The ECG refresh study is CPU-only and uses one benchmark whose component is a deterministic
resource constraint, two head classes and ten subject splits; we did not refresh an LLM-scale verifier,
did not study online RL with a refreshed verifier, and our prediction that refresh repairs the monitor failed. Engine and GSM8K policies are Qwen2.5, the
strongest training result is offline DPO with three seeds, and ECG training is informative in one
family. The audit requires trusted component labels (proxy bins change the estimand); the engine
oracle is simplified, and ECG is a signal-processing benchmark, not a clinical validation.

\section{Related work}
\label{sec:related}
Modern post-training makes the component/aggregate interface concrete. Rubric RL scores
prompt-specific criteria separately and then aggregates them into scalar on-policy rewards
\citep{gunjal2026rubrics,mahmoud2026rubric}; policy-aware rubric work further shows that criterion
informativeness changes with the policy because many criteria become saturated or dead
\citep{tyagi2026policyaware}. Fresh policy interactions are also used to refine reward or preference
models \citep{sun2025ugda,bai2022training}. Two adjacent loop failures are especially relevant:
on-policy optimization against one verifier can remove behaviors needed for later objectives
\citep{wei2026support}, and a verifier useful on one task can regress a learner after task shift
\citep{goodverifiers2026}. We instead study the \emph{audit interface}: full-outcome AUC can stop
identifying one component, prefer an entangled verifier, and induce a harmful refresh even when the
component verifier was initially correct.

The statistical ingredients are classical. ROC behavior depends on case mix
\citep{sadatsafavi2022mroc}; partial identification and trimming bounds describe incomplete
observables \citep{manski2003partial,tamer2010partial,horowitz1995identification}; deliberately
varying mixture weights is a linear special case of finite-mixture identification
\citep{henry2014mixtures}; the within/cross split parallels xAUC \citep{kallus2019xauc}; and the
population AUC maximizer is the classical optimal ranker \citep{clemencon2008ranking}. We do not
claim these ingredients as new. The new consequence is their combination under policy-induced
\emph{semantic} failure mixtures: complete loss of component identification, a calculable
verifier-ranking inversion, and an audit-driven maintenance failure.

\label{main:body_end}

\section*{AI use statement}
Generative AI was used to improve the clarity and presentation of manuscript claims, assist with software implementation and debugging, and provide critical feedback on claims and arguments during manuscript development. AI-generated suggestions were treated as editorial and technical assistance rather than as evidence or independent validation. The authors are fully responsible for the final claims, proofs, code, citations, analyses, and submission.

\section*{Ethics statement}
The engine evaluations concern a simplified geometry, mass and printability model,
not safety certification or fabrication validation; propulsion design can have
dual-use implications. The ECG study uses public, de-identified PhysioNet recordings
(MIT-BIH Arrhythmia and Noise Stress Test databases) under their published terms; it is
a signal-processing deployment benchmark, not a diagnostic or clinical validation, and
its outputs should not be used for patient care.

\section*{Reproducibility statement}
The supplement contains the exact anonymous source snapshot, the CPU-only GoEmotions and
AG News semantic stress-test scripts with frozen JSON summaries, and the frozen controller
threshold analysis used in the final revision. The original large engine/ECG/GSM8K candidate
pools and model checkpoints are not bundled in this lightweight artifact; their protocols,
registered predictions, assumptions and reported frozen summaries are documented in the
appendix. Thus the supplement supports an exact paper rebuild and independent reruns of the
public-data semantic controls, but not an end-to-end regeneration of every large system
experiment from raw candidates. See Appendices~\ref{app:train}, \ref{app:goemotions},
\ref{app:agnews}, \ref{app:prereg} and~\ref{app:compute}.

\bibliography{refs}
\bibliographystyle{iclr2027_conference}

\appendix
\section*{Appendix roadmap}
The appendix is intentionally evidence-first. Appendices~A--AD contain the oracle definitions,
core audit-collapse controls, identification proofs, transfer checks, preregistration material and the
refresh/semantic experiments that support the main claim. Appendices~AE--AF are optional boundary
analyses of downstream selection and DPO interfaces; they are not needed for Propositions~1--3 or
the refresh-trap conclusion. Bundled scripts and frozen summaries are indexed in the supplement README; appendix-only workspace paths are provenance labels and do not imply that the corresponding large artifacts are packaged.

\makeatletter
\renewcommand\thesection{\ifnum\c@section>26 A\@Alph{\numexpr\c@section-26\relax}\else\@Alph\c@section\fi}
\makeatother
\section{The engine oracle}
\label{app:oracle}

The solver decodes a design vector to a multi-zone body, computes its mass by
integration, evaluates printability (overhang, minimum ligament and wall, DfAM
flags) and envelope fit, and returns \texttt{feasible} as their conjunction, in
$7.9$\,ms; the $200\,000$-row reference sample took $72$\,s. Its thermochemistry is
checked against NASA CEA through \texttt{rocketcea}, independent of the project's
own Cantera path: across $28$ operating points the maximum error is $0.25\%$ on
chamber temperature and $1.7\%$ on characteristic velocity. The flow solver is
grid-convergent on the Sod problem ($L_1$ error $0.0079 \to 0.0042 \to 0.0024$ at
$100/200/400$ cells) and reproduces the oblique-shock relation to $1.2^\circ$. Those
checks cover the flow path, which is sized from the requirement and not the
generator's to choose. The on-target criterion is two-sided because under a
one-sided one $86.4\%$ of on-target draws in the reference sample satisfy it by
being small rather than by meeting the budget.

\section{The identity for the mass head}
\label{app:mass}

With the roles swapped, the mass head's score against $y = F \wedge M$ decomposes as
$(1-q)\,\mathrm{AUC}(s_M; y{=}1 \text{ vs } F{=}0) + q\,\mathrm{AUC}(s_M; y{=}1 \text{ vs }
F{=}1, M{=}0)$. The second term is the head's own sub-problem and is $0.96$--$0.99$ on
every arm; the first is $0.84$--$0.95$, since infeasible designs are also, on
average, further from the budget. As $q \to 1$ the mass head's AUC against $y$ rises
to its own accuracy ($0.937$ on the base pool, $0.98$ on the trained pools).

\section{GSM8K with a conjunctive target}
\label{app:gsm}

\begin{table}[h]
\centering
\caption{GSM8K, $500$ test prompts, three training seeds per trained arm. Upper:
policy rates, median within-pool score SD at $K=64$, and correctness-head AUC
against the conjunctive target. Lower: supply and answer at $K=16$ (first $16$
of $64$ samples) under sum, soft log-probability and correctness-only selectors.}
\label{tab:gsm}
\small
\setlength{\tabcolsep}{3pt}
\begin{adjustbox}{max width=\linewidth}
\begin{tabular}{lcccccc}
\toprule
policy & correct & boxed & on target & SD format & SD correct & correct AUC \\
\midrule
base & $0.601$ & $0.42$ & $0.302$ & $8.6$ & $1.45$ & $0.755$ \\
sum DPO & $0.648\pm0.006$ & $0.98$ & $0.647\pm0.006$ & $0.02$ & $0.85\pm0.05$ & $0.690\pm0.006$ \\
gated DPO & $0.644\pm0.009$ & $0.98$ & $0.642\pm0.009$ & $0.02$ & $0.83\pm0.08$ & $0.689\pm0.012$ \\
\bottomrule
\end{tabular}
\end{adjustbox}
\medskip
\begin{tabular}{lcccc}
\toprule
policy & supply@16 & sum & log-prob & correct only \\
\midrule
base & $0.872$ & $0.734$ & $0.736$ & $0.576$ \\
sum DPO & $0.933\pm0.005$ & $0.673\pm0.014$ & $0.669\pm0.014$ & $0.670\pm0.015$ \\
gated DPO & $0.927\pm0.006$ & $0.685\pm0.011$ & $0.686\pm0.014$ & $0.685\pm0.015$ \\
\bottomrule
\end{tabular}
\end{table}

At $K{=}64$ supply is $0.948$ (base) and $0.967$--$0.969$ (trained); best-of-$64$ by the
sum is $0.762$ (base) against $0.645 \pm 0.012$ and $0.657 \pm 0.011$. Pairs built from
the base train slice by the sum and by the gate coincide (chosen on target $0.521$
against $0.520$, rejected $0.079$ against $0.080$); pairs built from a trained
policy's own samples separate on target from off at $0.70$ against $0.59$ under
each tested rule, which closely follows the correctness head there.

\section{Engine arms not in the main text}
\label{app:engine}

\paragraph{Sampling budget.} From one $64$-sample run per arm (first $K$ taken,
nested), answer rate under the two-stage gate, one seed each ($1\,600$-pair policies):
1.5B base $0.028/0.078/0.204/0.291/0.391$ and 1.5B DPO
$0.104/0.265/0.452/0.512/0.561$ at $K = 1/4/16/32/64$; 7B base
$0.044/0.142/0.318/0.441/0.544$ and 7B DPO $0.096/0.204/0.365/0.428/0.504$. The
single-seed 7B reversal at $K{=}64$ is $1.8$ standard errors and uses the weakest of
the three 7B seeds; we do not claim it.

\paragraph{Training-side mitigations.} DPO at $\beta \in \{0.1, 0.5, 1.0\}$ and IPO at
target margins $0.005$ and $0.020$ on the 1.5B designer ($1\,600$ pairs, one seed
each) reach feasible rates $0.39$--$0.87$; each sits on Eq.~\ref{eq:mix} at its own
$q$ (Table~\ref{tab:law}); under the two-stage gate they capture $0.75$--$0.93$ of
their supply and the best composed system is the one with the most supply ($\beta =
1.0$, $0.514$). An IPO target $40\times$ out of reach ($5.0$ nats/token) broke the
output format (parse rate $0.40$) and is excluded from every table.

\paragraph{Scale.} Under the gate the 7B designer after DPO reaches $0.414 \pm 0.029$
at $K{=}16$ against the 1.5B's $0.439 \pm 0.015$; its feasibility head saturates to
the same endpoint with five times the seed spread. The 7B base policy's parse rate is
$0.76$; on parsed candidates its feasible rate is $0.43$.

\paragraph{Latent search.} Two hundred steps of gradient ascent on the unconstrained
autoencoder's latent code against the critic's two terms give supply $0.937 \pm
0.017$ at $K{=}64$ against $0.336 \pm 0.003$ for direct sampling (three decoder
seeds), not compute-matched. Its pools are the ones with $A_{\rm cross} = 0.79$:
joint ascent couples the heads.

\begin{figure}[t]
\centering
\includegraphics[width=0.92\linewidth]{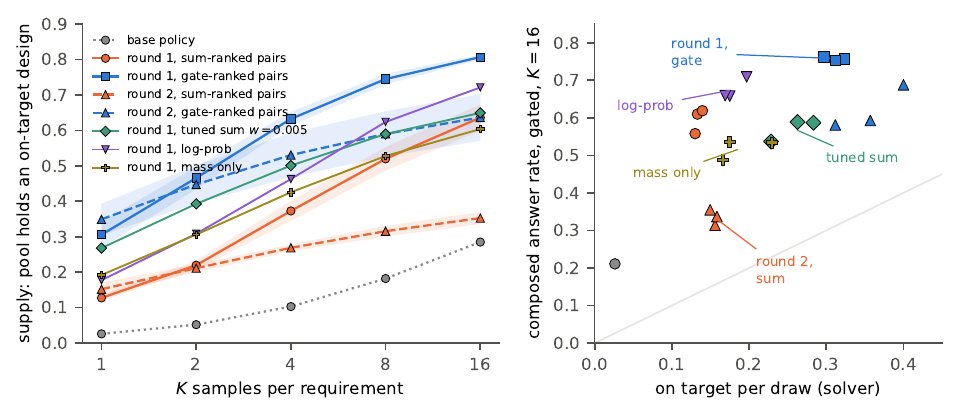}
\caption{Training-side supply and selection geometry. Left:
supply against $K$ (seed means; bands one sd for packaged arms). Right: per-draw
success against composed answer under the gate, one marker per seed.}
\label{fig:training}
\end{figure}

\section{The designs in parameter space}
\label{app:designs}

\begin{table}[h]
\centering
\caption{On-target designs of each policy (test split, all seeds pooled) against the
prior's on-target designs in the reference sample. ``At bound'': radius controls
within $0.05$ of $0$ or $1$; ``block-like'': every radius control above $0.9$;
``sd'': mean per-coordinate standard deviation; $L_2$: mean pairwise distance
between designs; ``families'': share of the four lattice families
(gyroid / Schwarz-P / diamond / IWP).}
\label{tab:designs}
\small
\begin{adjustbox}{max width=\linewidth}
\begin{tabular}{lcccccccc}
\toprule
designs & $n$ & distinct & radius mean & at bound & block-like & sd & $L_2$ & families \\
\midrule
prior (reference) & $1\,594$ & $1.00$ & $0.50$ & $0.03$ & $0.00$ & $0.28$ & $2.15$ & $.24/.25/.25/.26$ \\
base policy & $418$ & $0.96$ & $0.49$ & $0.05$ & $0.00$ & $0.23$ & $1.80$ & $.30/.27/.24/.19$ \\
round 1, sum pairs & $6\,442$ & $1.00$ & $0.42$ & $0.00$ & $0.00$ & $0.17$ & $1.34$ & $.42/.40/.16/.03$ \\
round 1, gate pairs & $14\,915$ & $1.00$ & $0.49$ & $0.01$ & $0.00$ & $0.17$ & $1.34$ & $.09/.32/.48/.12$ \\
round 2, sum pairs & $7\,420$ & $1.00$ & $0.42$ & $0.00$ & $0.00$ & $0.07$ & $0.57$ & $.00/.45/.55/.00$ \\
round 2, gate pairs & $17\,104$ & $1.00$ & $0.48$ & $0.00$ & $0.00$ & $0.13$ & $1.04$ & $.17/.45/.35/.03$ \\
\bottomrule
\end{tabular}
\end{adjustbox}
\end{table}

Every policy's on-target designs have $2.85$--$2.97$ active zones on average, a
printability score of $1.0$, and mass at $0.99$ of the budget on average
(\texttt{design\_degeneracy.py}). The benchmark --- oracle, critic, splits, pools and every
result file --- is released with the paper.

\paragraph{Within-requirement diversity.} Conditional on at least two on-target
samples per requirement, mean within-requirement coordinate spread is $0.101$--$0.104$
for matched sum DPO and $0.097$--$0.103$ for gated DPO. It falls to $0.038$--$0.039$ in
the sum-ranked second round, against $0.059$--$0.064$ for gated ranking. This
conditions on successful samples and is descriptive.

\section{The escape}
\label{app:escape}

\begin{figure}[h]
\centering
\includegraphics[width=\linewidth]{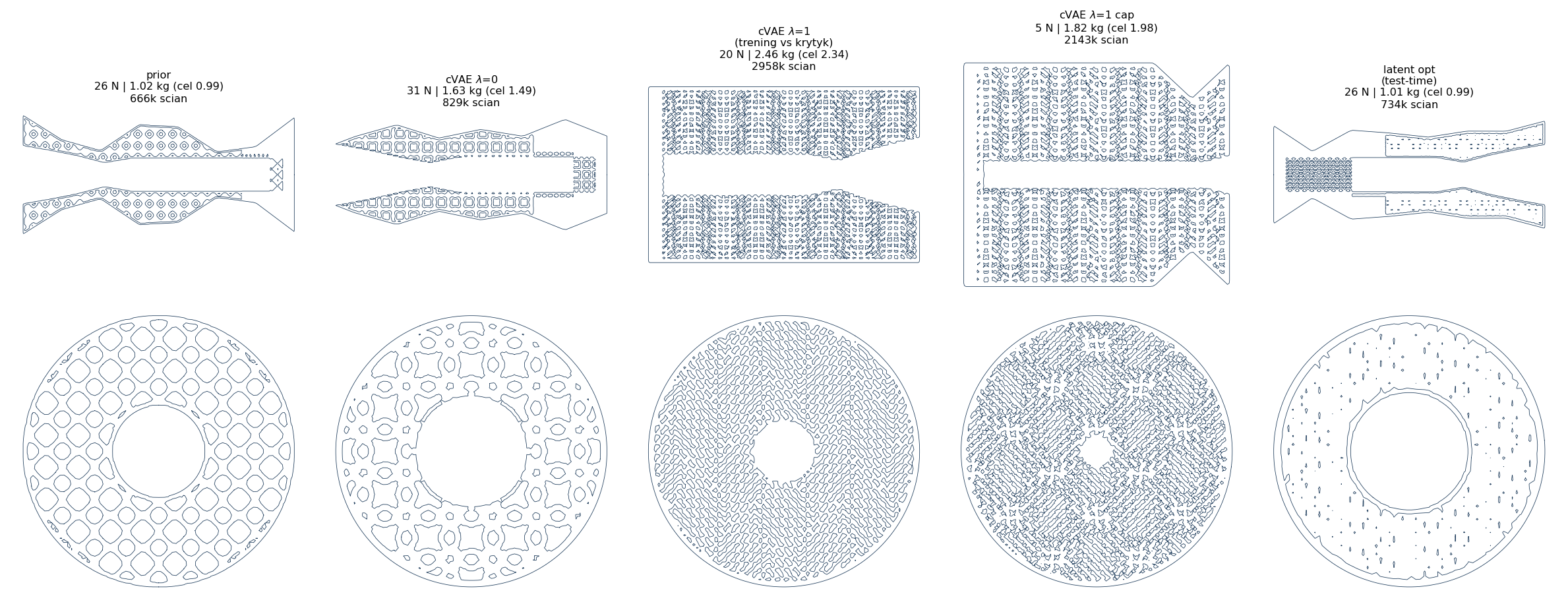}
\caption{One design from each cVAE arm's pool, decoded by the function the solver
grades (longitudinal section above, cross-section at mid-length below). The prior,
the unconstrained decoder and latent search produce engines with visibly different
lattice families; the decoder trained on the critic's terms produces a block ---
every radius control pinned at maximum after clipping from far outside the box ---
that passes the mass-and-printability criterion and that the critic scored feasible
at $0.9999$ before clipping.}
\label{fig:designs}
\end{figure}

\section{Training details}
\label{app:train}

Policies: Qwen2.5-1.5B-Instruct and Qwen2.5-7B-Instruct, LoRA rank $16$ ($\alpha = 32$,
dropout $0.05$), DPO $\beta = 0.1$ (sigmoid loss), learning rate $5\times 10^{-5}$,
batch $4 \times$ accumulation $2$ at 1.5B and $2 \times4$ at 7B, three epochs,
max length $2560$ for engine DPO, bf16, trl. Pairs: the score's top-$4$
against bottom-$4$ of $16$ samples per requirement (engine) or of $8$ per prompt
(GSM8K); unparseable completions carry the score $-10^3$ and so form the rejected
side. Engine pair pools: the original arms used $400$ requirements drawn by the
bench ($1\,600$ pairs, $\approx 600$ steps); the round-1 arms of Table~\ref{tab:train}
use the first $1\,000$ requirements of the committed train split ($4\,000$ pairs,
$\approx 1\,500$ steps); round 2 starts from the round-1 sum policy (seed 0), merged,
which is also its reference. Seeds vary initialisation and batch order; the pairs
are deterministic given the pool. GSM8K heads: RoBERTa-base, two labels, two epochs,
learning rate $2\times 10^{-5}$, on $12\,000$ (prompt, completion) pairs from the train
slice with a $15\%$ validation split.

Engine DPO hard-gated ranking uses the squared log-mass proximity score for
$P(F)\geq0.5$ and subtracts $100$ below the threshold; unparsed rows remain at
$-1000$. Two-stage evaluation instead minimises relative linear-mass error among
passing candidates and falls back to maximum $P(F)$ if none pass. GRPO uses a
soft reward $\log\sigma(s_F)+\texttt{mass\_term}$ versus $s_F+\texttt{mass\_term}$,
and unparsed reward $-100$. The TRL 1.12.0 GRPOTrainer recipe uses its default
DAPO token-normalised loss and group reward scaling, not the original sequence-
normalised GRPO loss. Group size $8$, LoRA $r=16$, two epochs over $1000$ prompts,
lr $10^{-5}$, $\beta=0.04$, temperature $0.9$, top-p $1.0$ for training,
completion limit $700$, per-device completion batch $8$, accumulation $2$,
yield $1000$ optimizer steps; evaluation uses top-p $0.95$.

\section{7B and online GRPO experiments: registration and support}
\label{app:new_campaign}

Registered 2026-09-15 before the runs. P-A1 held: hard-gated DPO at 7B improves
per-draw and two-stage composed answer on each seed, ratio $1.79\times$ inside
$[1.3,2.3]$. P-A2 held: feasibility is lower on every gated seed. P-A3 is partial:
exact identity, $A_{\rm cross}\in[0.45,0.58]$ and sum-arm $q\geq0.9$ all hold;
$A_{\rm own}\geq0.93$ on every arm fails on gated seed 1 ($0.9193$ parsed-only,
$0.9250$ including unparsed failures). No run is removed or threshold changed.

\begin{table}[h]
\centering
\caption{Parsed-only own-bin support for the added 7B runs. Bootstrap intervals
are post-hoc percentile intervals from 3000 requirement-level resamples, seed
20260918, not training-seed intervals; pairs within a requirement are dependent.}
\small
\setlength{\tabcolsep}{3pt}
\begin{adjustbox}{max width=\linewidth}
\begin{tabular}{llrrrr}
\toprule
ranking & seed & $A_{\rm own}$ & \shortstack{supported\\requirements} & \shortstack{positive--infeasible\\pairs} & 95\% CI \\
\midrule
sum & 0 & $1.0000$ & 11 & 127 & $[1.0000,1.0000]$ \\
sum & 1 & $0.9679$ & 38 & 313 & $[0.9371,0.9933]$ \\
sum & 2 & $0.9929$ & 20 & 270 & $[0.9762,1.0000]$ \\
hard gate & 0 & $0.9834$ & 70 & 815 & $[0.9677,0.9964]$ \\
hard gate & 1 & $0.9193$ & 83 & 1218 & $[0.8634,0.9662]$ \\
hard gate & 2 & $0.9822$ & 69 & 773 & $[0.9612,0.9969]$ \\
\bottomrule
\end{tabular}
\end{adjustbox}
\end{table}

An empirical interval of $[1,1]$ with only 11 supported requirements indicates no
observed errors, not proof of perfect ordering. The interval for the failed seed
includes $0.93$ but does not change the registered outcome. The intersection of
requirements supporting $A_{\rm own}$ in all six matched 7B arms contains only one
requirement, so restricting every arm to a common support would replace one support
problem with a more severe one. Identity statistics therefore average each arm over
its informative pools; the direct feasibility audit in Appendix~\ref{app:selection_audit}
uses $40$--$225$ informative requirements per seed and gives AUC $0.940$--$0.991$.

GRPO predictions: feasibility above $0.8$ (P-B1) and sum per-draw in
$[0.08,0.18]$ held; supply advantage at least $0.10$ (P-B2) failed.
Gate per-draw above the sum on every seed held, but gate per-draw at least $0.20$
(P-B3) failed. Coordinate sd below $0.17$ (P-B4) held for both rewards.
All parse rates exceed $0.998$, so no engineering discard applies. The strict
no-difference branch is not supported for supply, but this is a modest online
result, not a large two-algorithm replication. Endpoint outcomes do not establish
the predicted reward-optimisation trajectory.

\begin{table}[h]
\centering
\caption{Training scope checks: a second DPO round and online GRPO. ``Composed'' uses the
same two-stage evaluator as Table~\ref{tab:train}; mean $\pm$ sd over three seeds.}
\label{tab:train_scope}
\small
\setlength{\tabcolsep}{4pt}
\begin{adjustbox}{max width=\linewidth}
\begin{tabular}{llccccc}
\toprule
policy & pairs/reward & on target / draw & feasible / draw & supply @16 & parse & composed @16 \\
\midrule
round 2 & sum pairs & $0.155 \pm 0.005$ & $0.86$ & $0.353 \pm 0.017$ & $0.95 \pm 0.05$ & $0.335 \pm 0.021$ \\
round 2 & gate pairs & $0.356 \pm 0.044$ & $0.88$ & $0.636 \pm 0.066$ & $0.98$ & $0.621 \pm 0.059$ \\
GRPO, online & sum reward & $0.086 \pm 0.002$ & $0.85$ & $0.481 \pm 0.004$ & $1.00$ & $0.460 \pm 0.004$ \\
GRPO, online & log-prob reward & $0.089 \pm 0.003$ & $0.82$ & $0.528 \pm 0.003$ & $1.00$ & $0.492 \pm 0.008$ \\
\bottomrule
\end{tabular}
\end{adjustbox}
\end{table}

\section{Registered predictions and outcomes}
\label{app:prereg}

\begin{table}[h]
\centering
\caption{Evidence chronology for the follow-up studies. ``Pre-specified''
means the named plan/protocol was committed before the corresponding new analysis
or training run; it does not imply that earlier results in the paper were unknown.}
\label{tab:chronology}
\scriptsize
\setlength{\tabcolsep}{3pt}
\begin{adjustbox}{max width=\linewidth}
\begin{tabular}{p{0.24\linewidth}p{0.18\linewidth}p{0.24\linewidth}p{0.27\linewidth}}
\toprule
follow-up & record & status before result & role in paper \\
\midrule
7B hard-gate / GRPO & registration & before added runs & scale / optimization-scope checks; misses retained \\
strong scalar controls & protocol & after original sum-vs-gate, before new policies & tests tuned, log-prob, and mass-only pair builders \\
three-conjunct gate & plan & before job & omitted-objective boundary test \\
1.5B contrast scalar & plan & before $w=0.1$ training & endpoint test after pair-geometry sweep \\
source-term transfer & plan & source terms/code frozen; rough outcome location known & pre-specified transfer, not fully prospective forecast \\
within-run transfer & plan & before checkpoint pools & next-checkpoint forecast \\
7B contrast scalar & none & after 1.5B contrast result & non-preregistered scale follow-up \\
\bottomrule
\end{tabular}
\end{adjustbox}
\end{table}

Committed predictions are reported with misses rather than filtered post hoc. The
chronology above distinguishes original results, prospective follow-ups, and later
post-hoc or non-preregistered checks. The detailed registered statements follow.
(1) Identity within $0.03$ with $A_{\rm cross} \in [0.45, 0.58]$ on every engine arm:
partially held (mean $0.016$, max $0.057$, so not every arm is within $0.03$;
cross $0.47$--$0.56$ on the original arms).
(2) Variance-normalised sums below $\eta = 0.5$, gates above $0.85$: held for the
z-scorings and the gates; exceeded by $0.03$ by the rank average; a fixed weight of
$0.005$, identified in a later unregistered ablation, was not registered and matches the gates.
(3) Argmax dominance of the logit separating trained from base pools: failed
(dominance is $0.70$--$0.93$ on every arm).
(4) Round-1 gated pairs: per-draw above $0.095$, supply above $0.460$, feasible below
$0.88$: held. Registered before the matched control existed; against the matched
control the gated advantage is $2.3\times$, not $3.3\times$.
(5) Round-2 sum pairs: per-draw at or below $0.10$: failed ($0.155$); supply fell instead.
(6) GSM8K: boxed rate above $0.85$ and $q$ above $0.85$: held ($0.98$, $0.96$).
$A_{\rm cross} \in [0.50, 0.65]$: missed low by $0.02$ on some seeds ($0.47$--$0.51$).
Composite: the registered no-cliff branch (spread ratio below $0.5$, rules agree
within $0.1$) obtained. Gate pairs carry more correctness contrast than sum pairs:
failed (identical pairs). The fall in best-of-$16$ was not registered.
(7) Strong training controls (T1--T5, Appendix~\ref{app:pair_audit}): tuned sum and
log-prob within $0.04$ of the gate --- failed, both below by $0.19$ and $0.08$;
mass-only below its falsifier --- held; every new arm above the unit sum --- failed;
$q>0.9$ on tuned-sum arms --- failed ($0.57$--$0.63$), $A_{\rm own}>0.9$ --- held.
(8) Pre-specified later-arm forecast on the twelve new arms (F1--F6, Appendix~\ref{app:transfer}):
error $\leq0.05$ --- held ($0.020$); persistence $>0.10$ on the non-saturated subset ---
missed by $0.007$; AUC monotone in $q$ with mass-only highest --- held on the mean,
missed on one seed ($0.737$ against a registered $>0.80$); $\eta\geq0.5$ everywhere ---
held, but the gate ties the tuned sum under the raw selector; three-conjunct terms
unchanged --- held.
(9) Pair-contrast sweep and conditioning of the new arms (C1--C4,
Appendices~\ref{app:pair_audit},~\ref{app:conditioning},~\ref{app:physics_gate}):
no weight within $0.85$ of the gate's contrast --- failed ($0.91$ at $w\approx0.1$);
mass gain $0.95$ for tuned sum and mass-only, $0.45$ for log-prob --- all failed at
$0.72$--$0.74$; printability below $0.80$ under mass-only --- held ($0.57$);
envelope conditioning by rule --- held on all four; no second-level collapse ---
held; mass-trackers hottest --- held; every arm above base survival --- failed on two
seeds. (10) Three-conjunct DPO (Appendix~\ref{app:physics_gate}): seven of eleven
held, falsifier not fired, all misses toward a smaller effect.
(11) Contrast-matched sum, 1.5B (Appendix~\ref{app:pair_audit}): decision rule
returned ``contrast suffices''; every registered point value (feasible $0.79$, per
draw $0.27$, supply $0.76$, composed $0.70$, gain $0.80$, $q$ $0.90$, envelope
$\rho<0.45$) missed upward --- the policy reached the gate; $A_{\rm own}>0.9$ and the
base-pool forecast within $0.04$ held. A later, non-preregistered 7B follow-up at
$w=0.05$ is reported in Appendix~\ref{app:pair_audit}; it is not counted as a
registered outcome.
(12) Within-run forecast (Appendix~\ref{app:transfer}): next-checkpoint error
$\leq0.06$ with persistence worse and $\geq14/18$ wins --- held ($0.049$, $0.096$,
$15/18$; $0.021$ against $0.064$ with pooled source terms, an unregistered variant); $q\geq0.8$ by step $500$ on the saturating runs and persistence worst on the
first transition --- held; $A_{\rm own}\geq0.9$ at every checkpoint --- failed once
($0.853$); tuned-sum $q>0.5$ by step $500$ --- missed by $0.002$; base-pool error
$\leq0.07$ --- held ($0.032$); null arm --- held.

\section{Closest-work map}
\label{app:related_map}
\begin{table}[t]
\centering
\caption{Closest literatures and the distinction relevant to this paper. We separate ingredients
that are classical from the new verifier-audit consequence.}
\label{tab:closest_work}
\small
\setlength{\tabcolsep}{4pt}
\begin{tabular}{p{0.23\linewidth} p{0.30\linewidth} p{0.39\linewidth}}
\toprule
literature & closest object & distinction in this work \\
\midrule
ROC / case mix \citep{sadatsafavi2022mroc} & how population composition changes ROC behavior & inverse identification from semantic failure bins; complete-erasure threshold and component-specific audit target \\
finite-mixture identification \citep{henry2014mixtures} & identification as mixture weights vary while components stay fixed & our two-mixture rank condition is a special linear case; novelty is the verifier-audit interpretation and its operational consequences \\
rubric RL / policy-aware criteria \citep{gunjal2026rubrics,tyagi2026policyaware} & criterion-level judgments are aggregated for on-policy RL; criterion saturation changes with the policy & establishes the modern component/aggregate setting; we analyze outcome-audit identification and maintenance rather than reward weighting \\
RM evaluation \citep{wen2025rethinking,kim2025rewardoveropt,razin2025teacher,shum2026swerm} & static verifier/RM metrics need not determine downstream optimized-policy quality & a specific structural reason: outcome AUC of a component verifier becomes a moving semantic mixture and can prefer the wrong verifier \\
verifier-induced support reshaping \citep{wei2026support} & optimizing one verifier can remove policy behaviors needed for later objectives & support loss is not claimed as new; we show that the audit itself loses component identification and can trigger a refresh that forgets the missing failure mode \\
verifier-loop regression \citep{goodverifiers2026} & a verifier that is useful on one task can regress a learner after task shift & our failure occurs even without a bad verifier: an exact component oracle can look worse solely because policy success changes the failure mixture \\
continual learning / replay \citep{robins1995catastrophic,kirkpatrick2017overcoming} & preserving old capabilities under sequential distribution shift & audit geometry predicts when optimization removes the semantic failure support that refresh needs; replay is a repair, not the novelty \\
reward overoptimization / verifier failure \citep{gao2023scaling,rafailov2024scaling,skalse2022defining,mahmoud2026rubric} & proxy or verifier failure under optimization & the same full-target-AUC symptom can arise from an exact own-conjunct oracle or from genuine critic exploitation \\
verifier-guided test-time selection \citep{cobbe2021verifiers,lightman2023verify,brown2024monkeys,snell2024scaling,dorner2026roc,tts2026wild} & ranker quality and inference-time scaling & marginal component AUC does not determine relative component influence inside a composite selector \\
interface transfer in reward models \citep{shum2026swerm} & similar test-time selection can coexist with different RL outcomes & formal separation among component audit, composite selection, and the preference-pair training interface \\
multi-objective post-training \citep{dai2024saferlhf,moskovitz2023constrained,rame2023soups,zhou2023modpo,wu2023finegrained,liu2026gdpo,bai2026rdpo} & constraints, reward combination, model interpolation, or scalarization & before DPO, scalarization is a data-construction operator: a selection-optimal weight need not preserve useful training-pair contrasts \\
\bottomrule
\end{tabular}
\end{table}

\section{Compute}
\label{app:compute}

Original GPU experiments ran on single A100-80GB nodes; the added 7B jobs allowed
A100-80GB or H100. Actual accelerator assignment for the added 7B jobs is not recoverable from the packaged logs, so we do not attribute those runs to a more specific GPU model.
Original DPO at 1.5B: $\approx 2.5$\,h per
seed; at 7B: $\approx 5$\,h. Generating $1\,000$ engine requirements at $K{=}16$:
$\approx 4.5$\,h at 1.5B; GSM8K $500 \times 64$: $\approx 3$\,h. The cVAE arms, latent
search, oracle scoring and every analysis of \S\ref{sec:law}--\ref{sec:train} ran on
CPU; re-labelling a $16\,000$-candidate pool takes $20$\,s on $20$ cores.

\section{Additional geometry observations}
\paragraph{What the original 1.5B designers propose.} The on-target criterion can be met by a
degenerate body --- the constrained autoencoder of \S\ref{sec:scope} met it with a
block --- so we grade the policies' on-target designs in parameter space against the
prior's (Appendix~\ref{app:designs}, Figure~\ref{fig:designs_llm}). None are blocks:
across $6\,400$--$17\,000$ on-target designs per policy, no design has every radius
control near its bound, $0$--$1\%$ of radius controls sit at a bound (prior $3\%$),
the mean radius control is $0.42$--$0.49$ (prior $0.50$), every design is distinct at
two decimals, bodies use $2.9$ zones on average, and the oracle's printability score
is $1.0$. What training changes is diversity: the per-coordinate spread of on-target
designs falls from $0.28$ (prior) and $0.23$ (base) to $0.17$ after round 1 under
either score, and to $0.07$ after a sum-ranked second round, whose designs use two of
the four lattice families; the gate-ranked second round keeps $0.13$ and three. The
narrowing of Figure~\ref{fig:training} is visible in the geometry.

\begin{figure}[t]
\centering
\includegraphics[width=\linewidth]{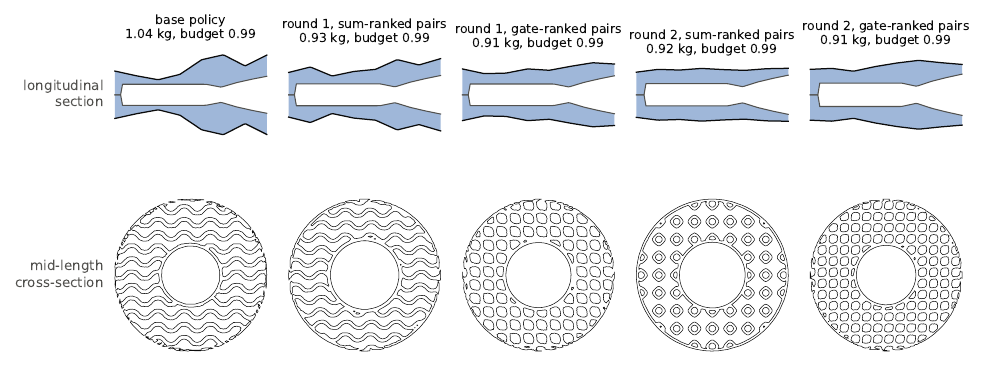}
\caption{One on-target design from each policy on the same requirement, decoded by
the function the solver grades: longitudinal section (bore and wall) and mid-length
cross-section through the lattice. The policies propose engines with different
lattice families and wall profiles, not blocks.}
\label{fig:designs_llm}
\end{figure}

\section{Post-hoc selection and uncertainty audit}
\label{app:selection_audit}

\begin{table}[h]
\centering
\caption{Selection efficiency $\eta$ at $K=16$ on identical engine pools using the
same critic. The raw unit sum performs poorly after feasibility saturates. Rules that
make the feasibility contribution flat in the saturated region recover most of the
available selection lift. Mean $\pm$ sd over seeds where applicable.}
\label{tab:rules}
\small
\setlength{\tabcolsep}{2.8pt}
\renewcommand{\arraystretch}{1.05}
\begin{adjustbox}{max width=\linewidth}
\begin{tabular}{lcccccc}
\toprule
rule & 1.5B base & 1.5B DPO, $1.6$k & 1.5B DPO, $4$k & 1.5B DPO, gated & 7B base & 7B DPO \\
\midrule
$\texttt{logit} + \texttt{mass}$ (as trained) & $0.26$ & $0.05 \pm 0.04$ & $0.10 \pm 0.02$ & $0.08 \pm 0.02$ & $0.18$ & $0.05 \pm 0.02$ \\
z-scored, reference statistics & $0.52$ & $0.49 \pm 0.02$ & $0.48 \pm 0.03$ & $0.36 \pm 0.01$ & $0.47$ & $0.47 \pm 0.11$ \\
z-scored per pool & $0.49$ & $0.29 \pm 0.02$ & $0.33 \pm 0.02$ & $0.49 \pm 0.01$ & $0.52$ & $0.35 \pm 0.04$ \\
rank average & $0.70$ & $0.53 \pm 0.03$ & $0.58 \pm 0.01$ & $0.71 \pm 0.03$ & $0.69$ & $0.54 \pm 0.03$ \\
$0.05\cdot\texttt{logit} + \texttt{mass}$ & $0.54$ & $0.53 \pm 0.03$ & $0.50 \pm 0.03$ & $0.37 \pm 0.01$ & $0.50$ & $0.50 \pm 0.11$ \\
\midrule
$0.005\cdot\texttt{logit} + \texttt{mass}$ & $\mathbf{0.85}$ & $0.91 \pm 0.03$ & $0.88 \pm 0.02$ & $0.83 \pm 0.02$ & $0.90$ & $0.92 \pm 0.03$ \\
$\log\sigma(\texttt{logit}) + \texttt{mass}$ & $0.51$ & $0.92 \pm 0.02$ & $0.89 \pm 0.00$ & $0.82 \pm 0.03$ & $0.76$ & $0.90 \pm 0.05$ \\
$p \geq 0.5$, then closest mass & $0.71$ & $\mathbf{0.94 \pm 0.02}$ & $\mathbf{0.92 \pm 0.01}$ & $\mathbf{0.90 \pm 0.01}$ & $\mathbf{0.92}$ & $\mathbf{0.94 \pm 0.04}$ \\
mass head alone & $0.50$ & $0.93 \pm 0.02$ & $0.91 \pm 0.01$ & $0.86 \pm 0.01$ & $0.74$ & $0.91 \pm 0.08$ \\
\midrule
one head trained on on-target & $0.58$ & $0.52 \pm 0.03$ & $0.44 \pm 0.02$ & $0.43 \pm 0.04$ & $0.67$ & $0.54 \pm 0.05$ \\
\bottomrule
\end{tabular}
\end{adjustbox}
\end{table}

\begin{table}[h]
\centering
\caption{AUC-preserving scale sensitivity on the matched 1.5B sum-DPO pools.
Multiplying the feasibility logit by positive $c$ preserves its marginal ROC/AUC
exactly, yet changes the additive selector's best-of-16 answer. Candidate pools and
critic outputs are fixed.}
\label{tab:scale_sensitivity}
\small
\begin{tabular}{lccccc}
\toprule
$c$ & $0.01$ & $0.1$ & $1$ & $10$ & $100$ \\
\midrule
raw-sum answer@16 & $0.544$ & $0.312$ & $0.185$ & $0.162$ & $0.160$ \\
\bottomrule
\end{tabular}
\end{table}

This follow-up uses stored test pools; it is post-hoc, not a new registered
prediction. We compute AUC directly against oracle feasibility on parsed samples,
using every feasible sample as positive, including feasible mass misses. Across
the matched DPO arms it is $0.940$--$0.991$, with $40$--$225$ informative
requirements per seed. Precision of the predicted passing set is $0.961$--$0.981$.
These measurements supplement own-bin AUC, which only uses on-target positives.
A fixed held-out score alone does not establish on-policy reliability.

\begin{table}[h]
\centering
\caption{Right-minus-sum differences on the same test requirements. DPO right arm:
hard-gated ranking; GRPO right arm: log-probability reward. Intervals use 2000
paired requirement bootstrap resamples after averaging the three fixed training
seeds within each requirement. They quantify test-requirement variation conditional
on those seeds, not population uncertainty over training seeds. Seed SD is reported
separately in the training tables.}
\label{tab:paired_audit}
\small
\begin{tabular}{llcc}
\toprule
family & endpoint & difference & requirement bootstrap 95\% CI \\
\midrule
DPO 1.5B & per draw & $0.1765$ & $[0.1630, 0.1905]$ \\
DPO 1.5B & supply@16 & $0.1707$ & $[0.1453, 0.1973]$ \\
DPO 1.5B & two-stage answer@16 & $0.1620$ & $[0.1373, 0.1877]$ \\
DPO 7B & per draw & $0.0827$ & $[0.0707, 0.0951]$ \\
DPO 7B & supply@16 & $0.1597$ & $[0.1377, 0.1820]$ \\
DPO 7B & two-stage answer@16 & $0.1477$ & $[0.1260, 0.1690]$ \\
GRPO 1.5B & per draw & $0.0035$ & $[-0.0014, 0.0086]$ \\
GRPO 1.5B & supply@16 & $0.0467$ & $[0.0287, 0.0653]$ \\
GRPO 1.5B & two-stage answer@16 & $0.0320$ & $[0.0123, 0.0513]$ \\
\bottomrule
\end{tabular}
\end{table}

\paragraph{Residual spread is not uniformly noise.} We permute feasibility logits
within each requirement's predicted passing set, preserving the set and mass
scores, and re-evaluate the raw sum (100 permutations). On gated 7B pools mean
raw answer falls from $0.237$ to $0.209$: residual ordering retains useful signal.
This does not make the raw sum competitive with two-stage selection, but it
precludes treating all residual feasibility variation as irrelevant noise.
Permutation is a conditional score ablation, not an intervention on training.

\paragraph{Parsing and the second-round lineage.} Relative to the single parent
policy, sum-ranked children lose supply even on requirements where both parent
and child pools are fully parsed: differences $-0.141$, $-0.110$, $-0.061$ on
$915$, $762$, $685$ requirements. Gated children improve supply on the corresponding
subsets by $0.050$, $0.103$, $0.218$. This shows that format breakage alone cannot
account for all observed loss; the subsets are post-hoc and selected, so this is
not a causal correction for parsing. Within-requirement coordinate spread is
reported in the main text conditional on at least two successful samples.

\paragraph{Top-1 failures and sampling budget.} The analysis artifact includes
nested $K=1,2,4,8,16$ answer curves, failures split into unparsed, infeasible parsed,
and feasible mass misses, and the fraction of requirements with a solution that
the selector misses. These directly measure the top of the ranking, which AUC
alone does not determine. For matched sum DPO at 1.5B, the raw sum misses an
available solution on $0.446$--$0.482$ of requirements, compared with
$0.035$--$0.048$ under two-stage selection.

\section{Complete original mixture table}\label{app:original_mixture}
\begin{table}[t]
\centering
\caption{Eq.~\ref{eq:mix} on the head that scores the learned conjunct, the original
policy-sampled pool, parsed candidates only; the identity is evaluated per
requirement before aggregation, and rows with $\pm$ are mean $\pm$ sd over three
training seeds. ``Predicted'' is $(1-q)A_{\rm own} + \tfrac12 q$, i.e. the
identity with $A_{\rm cross}$ set to $\tfrac12$; the one constant it uses,
$A_{\rm own}$, is measured on the same pool. The engine's 7B base has parse rate
$0.76$, the 1.5B base $0.97$; all other engine rows $\geq 0.95$.}
\label{tab:law_original}
\small
\setlength{\tabcolsep}{3pt}
\begin{adjustbox}{max width=\linewidth}
\begin{tabular}{llcccccc}
\toprule
 & policy & conjunct rate & $q$ & measured & $A_{\rm own}$ & $A_{\rm cross}$ & predicted \\
\midrule
engine & 1.5B base & $0.26$ & $0.33$ & $0.802$ & $0.951$ & $0.496$ & $0.800$ \\
 & 1.5B DPO, sum pairs, $1.6$k & $0.87$ & $0.97$ & $0.512 \pm 0.010$ & $0.976 \pm 0.013$ & $0.500 \pm 0.008$ & $0.513$ \\
 & 1.5B DPO, sum pairs, $4$k & $0.87$ & $0.97$ & $0.532 \pm 0.004$ & $0.966 \pm 0.020$ & $0.519 \pm 0.007$ & $0.513$ \\
 & 1.5B DPO, gated pairs, $4$k & $0.80$ & $0.92$ & $0.535 \pm 0.006$ & $0.944 \pm 0.016$ & $0.494 \pm 0.001$ & $0.537$ \\
 & 1.5B round 2, sum pairs & $0.86$ & $0.99$ & $0.488 \pm 0.004$ & $0.984 \pm 0.023$ & $0.484 \pm 0.004$ & $0.505$ \\
 & 1.5B round 2, gated pairs & $0.88$ & $0.97$ & $0.545 \pm 0.019$ & $0.943 \pm 0.010$ & $0.533 \pm 0.023$ & $0.513$ \\
 & 1.5B IPO / DPO $\beta$ sweep (4 arms) & $0.39$--$0.83$ & $0.48$--$0.94$ & $0.715$--$0.506$ & $0.93$--$0.95$ & $0.47$--$0.49$ & $0.725$--$0.527$ \\
 & 7B base & $0.32$ & $0.59$ & $0.682$ & $0.983$ & $0.481$ & $0.698$ \\
 & 7B DPO, sum pairs & $0.79$ & $0.92$ & $0.513 \pm 0.054$ & $0.978 \pm 0.015$ & $0.476 \pm 0.013$ & $0.537$ \\
\midrule
GSM8K & base & $0.42$ & $0.16$ & $0.919$ & $0.997$ & $0.507$ & $0.920$ \\
 & DPO, sum pairs & $0.98$ & $0.96$ & $0.506 \pm 0.004$ & $0.995$ & $0.489 \pm 0.006$ & $0.518$ \\
 & DPO, gated pairs & $0.98$ & $0.96$ & $0.497 \pm 0.016$ & $0.995$ & $0.479 \pm 0.013$ & $0.518$ \\
\bottomrule
\end{tabular}
\end{adjustbox}
\end{table}

\section{Matched train-pair follow-up: tuned sum, log-probability and mass-only DPO}
\label{app:pair_audit}

\paragraph{Protocol, fixed before training.} The protocol, committed as
\texttt{PROTOCOL\_STRONG\_CONTROLS\_2026-09-18.md}, was written after
the original sum-versus-gate result was known and before any new pool or policy
existed; it is a prospective follow-up, not an original preregistration. A new
Qwen2.5-1.5B-Instruct pool on all $500$ committed validation requirements ($K{=}16$,
seed 0, the same shots and critic) selects a weight from the thirteen-point grid
$0$, $0.0005$, $0.001$, $0.002$, $0.005$, $0.01$, $0.02$, $0.05$, $0.1$, $0.2$,
$0.5$, $1$, $2$ by the highest oracle-labelled validation top-1 answer; no test
label enters. The sweep
peaks at $w=0.005$ ($0.238$) and is monotone on either side ($0.166$ at $w=0$,
$0.208$ at $0.01$, $0.134$ at $0.1$, $0.090$ at $1$). Three rules --- the weighted
sum at that weight, $\log\sigma(\texttt{logit})+\texttt{mass\_term}$, and
\texttt{mass\_term} alone --- were then trained on the stored $1000\times16$ train
pool with exactly the original recipe ($4\,000$ pairs, four per prompt, $\beta=0.1$,
three epochs, lr $5\times10^{-5}$, LoRA $r=16$), three seeds each, no seed dropped,
and evaluated on the committed $1\,000$ test requirements with common selectors. All
nine pools were re-labelled by the oracle independently of the bench.

\paragraph{Pair content.} We reconstruct the original sum pairs exactly from the
stored train pool and build the other rules' pairs with the same top-four versus
bottom-four builder. Table~\ref{tab:pair_content} gives the oracle content of each
rule's pairs and Table~\ref{tab:strong_controls} the policies trained on them.

\begin{table}[h]
\centering
\caption{Matched train-pair content, chosen/rejected: feasible rate, on-target rate
and median absolute oracle log-mass error (parsed sides). Exact triple overlap
(requirement, chosen, rejected) is $0.099$ between sum and gate, $0.586$ between sum
and log-prob, $0.210$ between gate and $w=0.005$, and $0.463$ between $w=0.005$ and
mass-only.}
\label{tab:pair_content}
\small
\begin{tabular}{lccc}
\toprule
ranking & feasible & on target & mass error \\
\midrule
sum ($w=1$) & $0.654/0.015$ & $0.0653/0.0018$ & $0.495/0.604$ \\
hard gate & $0.565/0.084$ & $0.0890/0.0005$ & $0.290/0.918$ \\
tuned sum ($w=0.005$) & $0.349/0.223$ & $0.1008/0.0000$ & $0.193/1.007$ \\
log-prob & $0.635/0.019$ & $0.0823/0.0008$ & $0.436/0.616$ \\
mass-only & $0.279/0.245$ & $0.0930/0.0000$ & $0.172/1.007$ \\
\bottomrule
\end{tabular}
\end{table}

\begin{table}[h]
\centering
\caption{Policies trained on the five rules; test $1000\times16$, mean $\pm$ seed sd.
The last three columns are the gate's advantage under the common two-stage
selector with a paired requirement bootstrap (seeds fixed), and the head's
mixture terms on the trained policy's pool.}
\label{tab:strong_controls}
\small
\setlength{\tabcolsep}{3pt}
\begin{adjustbox}{max width=\linewidth}
\begin{tabular}{lcccccccc}
\toprule
ranking & per draw & feasible & supply@16 & two-stage & gate $-$ rule [95\% CI] & $q$ & AUC & $A_{\rm own}$ \\
\midrule
sum & $0.134 \pm 0.005$ & $0.867$ & $0.636 \pm 0.037$ & $0.596 \pm 0.033$ & $0.162$ $[0.135,0.187]$ & $0.970$ & $0.532$ & $0.966$ \\
hard gate & $0.311 \pm 0.013$ & $0.804$ & $0.807 \pm 0.007$ & $0.758 \pm 0.005$ & --- & $0.916$ & $0.535$ & $0.944$ \\
tuned sum & $0.258 \pm 0.028$ & $0.658$ & $0.649 \pm 0.023$ & $0.571 \pm 0.030$ & $0.187$ $[0.159,0.215]$ & $0.601$ & $0.674$ & $0.928$ \\
contrast-matched sum, $w{=}0.1$ & $0.312 \pm 0.021$ & $0.867$ & $0.821 \pm 0.017$ & $0.770 \pm 0.023$ & $-0.012$ $[-0.030,0.006]$ & $0.961$ & $0.515$ & $0.927$ \\
log-prob & $0.180 \pm 0.015$ & $0.871$ & $0.721 \pm 0.038$ & $0.675 \pm 0.029$ & $0.082$ $[0.059,0.104]$ & $0.973$ & $0.524$ & $0.947$ \\
mass-only & $0.190 \pm 0.035$ & $0.451$ & $0.604 \pm 0.024$ & $0.519 \pm 0.027$ & $0.238$ $[0.208,0.267]$ & $0.381$ & $0.778$ & $0.948$ \\
\bottomrule
\end{tabular}
\end{adjustbox}
\end{table}

\paragraph{Registered predictions and outcomes.} Before any result was read
(\texttt{PREREG\_TUNED\_SUM\_THEORY\_2026-09-18.md}) we registered, from the
aggregation argument of \S\ref{sec:select}, that the tuned sum and the
log-probability form would each land within $0.04$ of the gate's composed answer and
$0.05$ of its per-draw rate and supply (T1, T2), that mass-only would stay below
$0.70$ composed with the falsifier ``within $0.03$ of the gate'' (T3), that every
new arm would beat the unit sum (T4), and that $q>0.9$ and $A_{\rm own}>0.9$ on
every tuned-sum and log-prob arm (T5). T3 held on every seed and the falsifier did
not fire. T1, T2 and T4 missed: the tuned sum is $0.187$ below the gate, the
log-probability form $0.082$ below, and the tuned sum and mass-only fall below the
unit sum on composed answer. T5 missed on $q$ for the tuned sum ($0.57$--$0.63$) with
$A_{\rm own}>0.92$ everywhere. The prediction that failed was the extrapolation of
a selection-side argument to pair construction; the selection-side part held
(the validation-selected weight lies in the registered band $[0.001,0.005]$).

\paragraph{Why the tuned weight is the wrong training weight.}
Figure~\ref{fig:pair_contrast} sweeps $w$ over the train pool. Among parsed train
candidates the logit's spread ($9.96$) is ten times the mass term's ($0.97$), so
the pair contrast in feasibility, $\Delta F$, falls monotonically as $w$ shrinks
while the mass contrast $\Delta M$ rises: $(0.64, 0.11)$ at $w=1$, $(0.13, 0.82)$ at
$0.005$, $(0.03, 0.84)$ at $0$. The gate sits at $(0.48, 0.63)$. The scalar curve
passes close to that point near $w\approx0.1$ ($(0.48, 0.52)$), a weight $20\times$
the validation-selected one and one the selection criterion ranks fifth from the
bottom of thirteen; we had registered that no weight would come within $0.85$ of the gate on
both contrasts, and the sweep refutes that ($0.91$ at $w=0.075$). The chosen sets of
$w=0.1$ and of the gate overlap $65\%$, the full pairs $15\%$: the gate's rejected
candidates are the infeasible ones with the worst mass term, while a scalar score
rejects by whichever term dominates at that weight.

\paragraph{The contrast-matched weight matches the observed gate endpoint.} Whether that weight
recovers the gate's outcome was registered before training
(\texttt{PLAN\_CONTRAST\_DPO\_2026-09-19.md}: $w=0.1$ at 1.5B, three seeds, recipe
unchanged), with a decision rule: mass gain $\geq0.95$ and composed $\geq0.74$ would
mean contrast suffices and the gate is a weight in disguise; gain $\leq0.78$ and
composed $\leq0.66$ at feasibility $\geq0.75$ would mean the rejected set matters.
The rule returned \emph{contrast suffices}: per draw $0.312\pm0.021$, feasible
$0.867$, supply $0.821\pm0.017$, composed $0.770\pm0.023$, mass gain $0.98\pm0.02$,
$q$ $0.961$, $A_{\rm own}$ $0.927$ (Table~\ref{tab:strong_controls}); the paired
gate-minus-sum differences are $-0.001$ $[-0.015, 0.012]$ per draw, $-0.014$
$[-0.031, 0.003]$ supply, $-0.012$ $[-0.030, 0.006]$ composed and $-0.020$
$[-0.039, -0.001]$ under the log-probability selector. Every registered point value
missed, all in the same direction: we had expected a partial recovery (feasible
$0.79$, per draw $0.27$, composed $0.70$, gain $0.80$, envelope $\rho<0.45$) and the
policy reached the gate on every count and installed envelope conditioning as well
($0.69\pm0.06$), which no other rule does together with unit budget tracking. The
$0.72$ plateau of the three mis-weighted scalar rules is therefore a property of
their contrast, not of scalar scores, and the gate's distinctive rejected set does
not matter. The base-pool forecast of \S\ref{sec:law} lands within $0.024$ on these
arms. The 7B base train pool is used below to test whether the weight-selection
criterion itself transfers; a later follow-up then trains the corresponding 7B
scalar arm.

\paragraph{Choosing the contrast weight without gold labels.}
The trained $w=0.1$ arm above was selected by oracle-unit contrasts $\Delta F$ and
$\Delta M$. We subsequently pre-specified a label-free repetition using only the
critic on the same train pool: $\Delta P$ is the difference in predicted
feasibility probability between chosen and rejected samples, and
$\Delta\widehat M$ is the reduction in predicted absolute log-mass error. Across
the weight grid, $\Delta P$ and $\Delta F$ have Spearman correlation $1.000$ at
1.5B and $0.999$ at 7B; $\Delta\widehat M$ and $\Delta M$ have $0.995$ and
$0.999$. Matching the hard gate in critic units selects $w=0.075$ at 1.5B and
$w=0.056$--$0.075$ at 7B, the same narrow region as the oracle-unit sweep.

A further post-hoc criterion removes even the gate as a reference. For each $w$, we
normalise $\Delta P$ and $\Delta\widehat M$ by their respective maxima over the
sweep and maximise the smaller of the two normalised contrasts. This intrinsic
maximin rule selects $w=0.075$ at 1.5B and $w=0.056$ at 7B; maximising their product
or minimising distance to the ideal point $(1,1)$ selects exactly the same weights.
This criterion by itself does not prove that the resulting policy must match a
gate; the 1.5B experiment above and the later 7B follow-up below test that training
claim. It answers the practical weight-selection question using critic outputs
alone and is distinct from the failed label-free
\emph{diagnostic} in Appendix~\ref{app:label_free}.

\paragraph{Post-hoc split stability of the intrinsic criterion.}
As a sampling-sensitivity check, we recompute the critic-only maximin rule on the
independently generated \emph{base test} pools, which were not used to build the
training pairs. It selects exactly the train-pool choices: $w=0.074989$ at 1.5B and
$w=0.056234$ at 7B. In 500 requirement-bootstrap resamples of each test pool, the
1.5B maximin choice is $0.074989$ in $499/500$ resamples (and $0.1$ once); at 7B it
is $0.056234$ in $415/500$, $0.05$ in $66/500$ and $0.074989$ in $19/500$.
Thus every resample at both scales stays in the same $[0.05,0.10]$ region. Product
and distance-to-ideal criteria choose the same full-test-pool points. This audit is
post-hoc and is evidence for split stability of the criterion, not a new policy
training result.

\paragraph{7B follow-up of contrast-matched training.}
We subsequently trained three 7B DPO seeds at $w=0.05$, the nearest implemented
weight to the intrinsic 7B choice near $0.056$, with the same $4\,000$-pair recipe.
This run occurred after the 1.5B contrast result and is not preregistered. The three
policies reach target success per draw $0.218\pm0.021$ and supply@16
$0.615\pm0.068$. Re-evaluating the stored pools with the common two-stage selector
gives composed answer $0.565\pm0.065$, compared with $0.562\pm0.026$ for the
matched 7B hard gate. The paired requirement bootstrap for contrast-matched minus
gate is $0.003$ $[-0.015,0.020]$; against the 7B unit-sum policy it is $0.151$
$[0.126,0.175]$. The scalar's own weighted-sum selector reaches only
$0.261\pm0.045$, so the result supports transfer of the pair-construction criterion
without collapsing the distinction between training and inference-time selection.
The follow-up export retained target labels, critic scores, parse masks, and design
vectors, but not the separate oracle-feasibility label; we therefore do not report a
follow-up feasibility rate in Table~\ref{tab:train}.

\begin{figure}[h]
\centering
\includegraphics[width=0.72\linewidth]{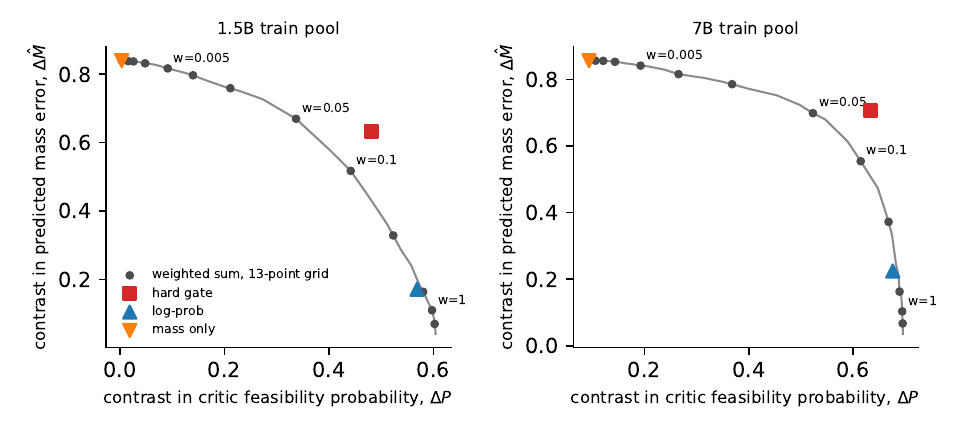}
\caption{Label-free pair contrast in critic units over scalar weights. The balanced
region is near $w\approx0.05$--$0.1$ at both scales, far from the
validation-tuned $w=0.005$. Oracle labels are not used by these axes.}
\label{fig:critic_contrast}
\end{figure}

\section{Proof of the sharp AUC identified set}
\label{app:identifiability_proof}
Let $S^+$ denote a score drawn from the positive class and $S_j^-$ a score drawn
from negative failure bin $B_j$. With the usual half-credit convention for ties,
$a_j=\Pr[S^+>S_j^-]+\tfrac12\Pr[S^+=S_j^-]$. Drawing a negative first chooses
bin $j$ with probability $\pi_j$, so the law of total probability gives
$A=\sum_j\pi_j a_j=\pi^\top a$. Hence every data-generating process compatible
with $(A,\pi)$ lies in $\{a\in[0,1]^J:\pi^\top a=A\}$.

The set is sharp. For any $a\in[0,1]^J$ satisfying the equality, take
$S^+\sim\mathrm{Unif}(0,1)$. For $a_j\in(0,1)$ choose the negative-bin score to be
the constant $S_j^-=1-a_j$; for the endpoints choose any constant above $1$ or below
$0$. Then the positive-versus-bin AUC is exactly $a_j$ for every bin simultaneously,
so every point in the set is attainable. For $J=2$,
$A=(1-q)A_{\rm own}+qA_{\rm cross}$ and eliminating
$A_{\rm cross}\in[0,1]$ yields Eq.~\ref{eq:identified_interval}; exchanging the
roles gives the symmetric bound. If $q=1$, $A_{\rm own}$ drops out completely.
This proves Proposition~\ref{prop:identified_set}.

\begin{figure}[h]
\centering
\includegraphics[width=.48\linewidth]{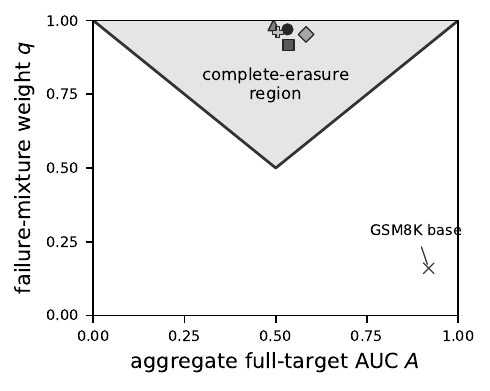}
\caption{Complete erasure for the two-bin audit. In the shaded region
$q\ge\max(A,1-A)$, the sharp identified interval for $A_{\rm own}$ is all of $[0,1]$.}
\label{fig:identified_region_appendix}
\end{figure}

\paragraph{Controlled-mixture audit design.} Suppose the bin-specific ordering vector
$a$ is held fixed while an auditor constructs $K$ negative mixtures with known rows
$\pi^{(k)}$. Writing $\Pi$ for the resulting $K\times J$ matrix, the aggregate
audits satisfy $A=\Pi a$. Hence $a$ is point-identified for every possible value iff
$\operatorname{rank}(\Pi)=J$; otherwise a null-space direction that remains inside
$[0,1]^J$ is observationally equivalent. For the binary split, two audits with
$q_1\ne q_2$ give
\[
 A_{\rm own}=\frac{q_2A_1-q_1A_2}{q_2-q_1},\qquad
 A_{\rm cross}=\frac{(1-q_1)A_2-(1-q_2)A_1}{q_2-q_1}.
\]
This is a controlled-reweighting result; it does not license combining ordinary
policy checkpoints whose conditional score distributions may themselves change.

A useful corollary is the exact condition for complete loss of information about
$A_{\rm own}$ when $q<1$: its sharp interval is $[0,1]$ iff
$1-q\le A\le q$. All four matched engine rows in Table~\ref{tab:law} satisfy this
condition. The result is an identification statement about the aggregate statistic;
it does not say that direct own-bin estimates are unreliable when their support is
adequate.

\section{Why validation-tuned selection weights need not be pair-optimal}
\label{app:weight_geometry}
For two candidates $i,j$ scored by $r_w=w s_F+s_M$, their order can change only
when
\[
  w(s_{F,i}-s_{F,j})+(s_{M,i}-s_{M,j})=0.
\]
Thus, away from the finite set of pairwise crossing weights, every within-pool
ranking is fixed. Any top-$k$ selector and any deterministic top-versus-bottom pair
builder is consequently piecewise constant in $w$. If two rules induce the same
chosen--rejected multiset, the DPO loss is literally the same function of the policy
parameters because the scalar score is no longer an input to DPO. This proves Proposition~\ref{prop:ranking_geometry}.

For the approximate statement, write the per-pair DPO loss as
$\ell_\theta(z)=\log(1+\exp[-\beta d_\theta(z)])$. If
$|d_\theta(z)|\le R$, then $\ell_\theta$ ranges over an interval of width exactly
$\beta R$. The variational characterisation of total variation therefore gives
$|\mathbb E_P\ell-\mathbb E_Q\ell|\le\beta R d_{\rm TV}(P,Q)$. For the
gradient bound, project the gradient difference onto an arbitrary unit vector; the
projected per-pair gradient lies in $[-G,G]$, so its expectation changes by at most
$2Gd_{\rm TV}(P,Q)$. Taking the supremum over unit vectors proves
Proposition~\ref{prop:pair_stability}. For uniform finite pair multisets, total
variation equals one minus their overlap coefficient (counting multiplicity).

\paragraph{A local bridge from pair contrasts to the first DPO update.}
\begin{proposition}[Reference-gradient contrast bound]
\label{prop:contrast_gradient}
Let $\theta_0$ equal the DPO reference policy, let
$g(z)=\nabla_\theta d_\theta(z)|_{\theta_0}$, and let
$\phi(z)\in\mathbb R^d$ be a vector of pair-property contrasts. Suppose
$g(z)=B\phi(z)+r(z)$, with
$\mathbb E_P\|r\|\le\epsilon_P$ and $\mathbb E_Q\|r\|\le\epsilon_Q$. Then
\[
\|\nabla L_P(\theta_0)-\nabla L_Q(\theta_0)\|
\le \frac{\beta}{2}\Bigl(
\|B\|\,\|\mathbb E_P\phi-\mathbb E_Q\phi\|+\epsilon_P+\epsilon_Q\Bigr).
\]
\end{proposition}
For DPO, $\ell_\theta(z)=\log(1+\exp[-\beta d_\theta(z)])$. At the reference
policy $\theta_0$, $d_{\theta_0}(z)=0$ for every pair, hence
\[
\nabla L_P(\theta_0)=-\frac{\beta}{2}\,\mathbb E_P[g(z)].
\]
Subtracting the expressions for $P$ and $Q$, applying the triangle inequality and
then Jensen proves Proposition~\ref{prop:contrast_gradient}. The assumption is an
explicit local feature model for the preference-log-ratio gradient. We did not retain
the per-pair gradient artifacts needed to estimate its residual for these runs, so
the proposition is a sufficient-condition bridge, not an empirically established
mechanism or a claim that oracle contrasts are sufficient statistics for later
optimization.

There is nevertheless no implication from validation top-1 optimality to balanced
training-pair contrast. A two-prompt counterexample suffices. Each row below is a
candidate with gold conjuncts $(F,M)$ and proxy components $(s_F,s_M)$:
\[
\begin{array}{c|ccc}
&1&2&3\\ \hline
\text{prompt A}&(0,1;5,1)&(1,1;0,6)&(0,1;2,0)\\
\text{prompt B}&(0,1;5,2)&(0,0;3,4)&(1,1;4,5).
\end{array}
\]
At $w=0.5$, the successful $(1,1)$ candidate is top-ranked for both prompts, so
top-1 success is $1$, but top-versus-bottom pairs have average oracle contrasts
$(\Delta F,\Delta M)=(1,0)$. At $w=2$, top-1 success falls to $1/2$, while the
average contrasts become $(1/2,1/2)$. Hence a weight that is strictly better for
selection can erase one training contrast that a worse selector preserves. The
example is only an existence result; the engine experiment tests whether this
mismatch matters in the observed system.

\paragraph{Pair overlap is not the empirical explanation.}
All 1.5B round-1 rules use $4\,000$ pairs from the same train pool. Their full-pair
overlap with the hard gate and eventual composed answer are
\[
\begin{array}{lcc}
\toprule
\text{ranking rule} & \text{pair overlap with gate} & \text{composed @16}\\
\midrule
\text{unit sum} & 0.099 & 0.596\\
\text{validation-tuned sum} & 0.210 & 0.571\\
\text{contrast-matched sum }(w=0.1) & 0.150 & 0.770\\
\text{log-probability} & 0.140 & 0.675\\
\text{mass-only} & 0.262 & 0.519\\
\bottomrule
\end{array}
\]
Thus the best-performing scalar is not the one whose examples most resemble the
gate's. Its chosen-candidate set overlaps the gate by $65\%$, but only $15\%$ of
complete chosen--rejected pairs coincide.

\begin{figure}[h]
\centering
\includegraphics[width=0.46\linewidth]{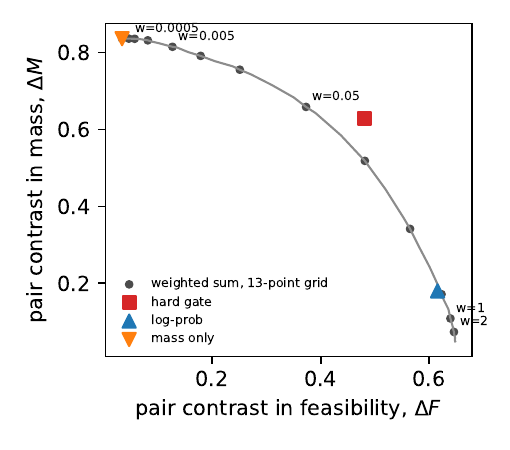}
\caption{Matched 1.5B pair geometry: tuned $w=0.005$ is mass-dominated, the unit sum
feasibility-dominated; the pre-specified $w\approx0.1$ passes near the gate.}
\label{fig:pair_contrast}
\end{figure}

\paragraph{Across-rule contrast alignment.}
For the five scalar rules trained with the same 1.5B recipe, define the bottleneck
contrast relative to the hard gate as
\[
b(r)=\min\!\left\{\frac{\Delta F_r}{\Delta F_{\rm gate}},
                  \frac{\Delta M_r}{\Delta M_{\rm gate}}\right\},
\qquad
(\Delta F_{\rm gate},\Delta M_{\rm gate})=(0.481,0.628).
\]
The quantities below use only the already reported pair-content and endpoint tables:
\[
\begin{array}{lccc}
\toprule
\text{scalar rule} & b(r) & \text{pair overlap with gate} & \text{composed @16}\\
\midrule
\text{unit sum} & 0.174 & 0.099 & 0.596\\
\text{validation-tuned sum} & 0.262 & 0.210 & 0.571\\
\text{contrast-matched sum }(w=0.1) & 0.828 & 0.150 & 0.770\\
\text{log-probability} & 0.287 & 0.140 & 0.675\\
\text{mass-only} & 0.071 & 0.262 & 0.519\\
\bottomrule
\end{array}
\]
Across these five points, Spearman correlation with composed @16 is $0.90$ for
$b(r)$ and $-0.60$ for exact pair overlap. With $n=5$ this is deliberately presented
as a descriptive ordering check, not a significance result or a fitted predictive
model. It nevertheless tests the claim on every scalar rule that was actually
trained under the matched recipe, rather than only on the successful $w=0.1$ arm.

These are descriptive overlaps from the
fixed training pool, not a statistical test.

\section{Perfect-conjunct and finite-pool controls}
\label{app:perfect_head}
These are post-hoc controls on existing candidates, not part of the original
preregistration. The exact feasibility oracle (engine) and boxed-format regex
(GSM8K) replace only the head of the corresponding conjunct for AUC measurement.
A binary perfect conjunct head ties every successful candidate with every failure
that satisfies that conjunct. Thus its full-target AUC is $1-q/2$, with half credit
for ties. We average over requirements containing both target classes; engine
candidates are parsed-only. Support is $285$ base requirements and $658,657,593$
for the three matched 1.5B sum-DPO seeds. Their perfect-head AUCs are
$0.8326$ and $0.5187,0.5145,0.5113$, respectively. GSM8K support is $474$ at base
and $431,443,439$ after sum-DPO, with AUCs $0.9219$ and
$0.5189,0.5210,0.5155$. Own-task oracle correctness is exact by construction;
this is not an oracle for the full conjunction.

We also measure best-of-$k$ averaged over every uniform size-$k$ subset of each
stored GSM8K pool of $N=64$, without replacement. For a deterministic selector
order, candidate rank $r\in\{1,\ldots,N\}$ is selected with probability
$\binom{N-r}{k-1}/\binom Nk$; averaging its target label with these weights gives
the exact conditional answer rate. Supply is
$1-\binom{N-P}{k}/\binom Nk$, where $P$ is the number of successes. Ties use
original candidate index. Exhaustive small-pool subset enumeration verifies the
implementation. This is a finite-pool estimator, not a new scaling law or a
prediction beyond $N$; related ROC-based scaling theory is in
\citet{dorner2026roc}.

For sum-DPO, averaging training seeds, raw-sum answers at $k=4,16,64$ are
$0.693,0.680,0.645$, versus base $0.606,0.736,0.762$.
Supplies after sum-DPO are $0.839,0.929,0.969$ at those budgets.
Replacing the learned format gate with the exact boxed regex and retaining the
learned correctness ranker gives $0.679$ at $k=16$ and $0.639$ at $k=64$ after
sum-DPO (base $0.737,0.758$). This control removes format-classification error
without repairing the remaining ranker. It does not establish impossibility for
other correctness models or selection methods.

Finally, a fixed scale grid $T\in\{0.01,0.1,1,10,100\}$ applied to feasibility
logits preserves every head AUC but changes raw-sum answers. Matched 1.5B sum-DPO
has mean answers $0.544,0.312,0.185,0.162,0.160$, respectively.
This test-set sensitivity analysis illustrates the distinction between ranking
metrics and aggregation units; it is not validation-based hyperparameter tuning.
Code and complete tables are in \texttt{scripts/audit\_perfect\_head.py},
\texttt{scripts/audit\_gsm\_tails.py}, and the corresponding results directories.

The paired requirement-bootstrap difference (sum-DPO minus base), conditional on
these fixed training seeds, is $-0.0555$ [$-0.0738,-0.0368$] for raw-sum answers
at $k=16$ and $-0.1167$ [$-0.1500,-0.0833$] at $k=64$.
Exact-format-gate differences are $-0.0578$ [$-0.0767,-0.0402$] and
$-0.1193$ [$-0.1533,-0.0873$]. These intervals quantify requirement uncertainty,
not population training-seed uncertainty.

\begin{figure}[h]
\centering
\includegraphics[width=\linewidth]{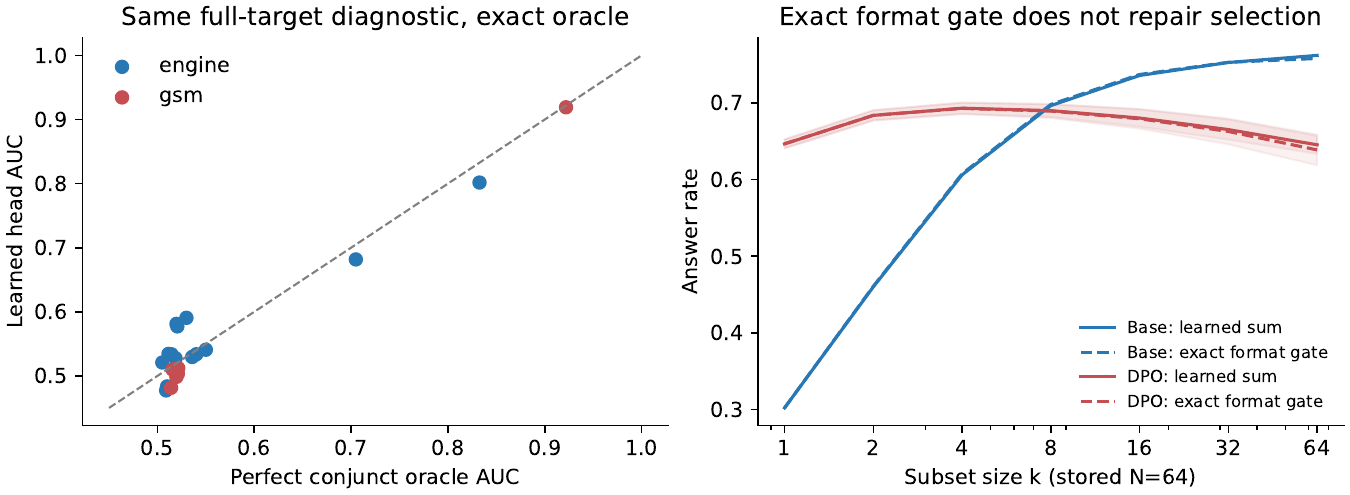}
\caption{Post-hoc controls. Left: full-target AUC for learned versus perfect
own-conjunct heads on the same pools. Right: GSM8K finite-pool subset-averaged
answers; exact format gating does not repair the remaining correctness ranker.
Bands show training-seed SD, not confidence intervals.}
\end{figure}

\section{Transfer forecast of on-policy AUC}
\label{app:transfer}
Equation~\ref{eq:mix} is exact on the pool whose terms it uses, so reconstructing
an AUC there is arithmetic rather than evidence. This control asks the out-of-sample
question instead. The bin terms $A_{\rm own}$ and $A_{\rm cross}$ are measured on a
\emph{source} pool; from a \emph{target} pool we read only $q$ and never the head's
ordering; we then predict the target's on-policy AUC. Every arm answers the same
committed test split, so requirement $i$ denotes the same requirement in both pools:
a prediction is formed per requirement and averaged over the requirements that
contain both target classes in both pools. We compare four predictors: chance
($\tfrac12$); \emph{persistence}, the source pool's own AUC, which is what assuming
an earlier audit of the same verifier still describes the policy amounts to; the
zero-parameter form $(1-q)A_{\rm own}+q/2$; and the transported form
$(1-q)A_{\rm own}+qA_{\rm cross}$.

Five protocols supply $57$ source--target pairs: the 1.5B base pool to $26$ trained
1.5B arms (matched sum and hard-gate DPO, the earlier DPO arms, the five IPO/DPO
mitigation arms, both GRPO rewards and the second round), the 7B base pool to $9$
trained 7B arms, the 1.5B base pool to those same $9$ 7B arms, the 1.5B base pool to
$7$ cVAE-family pools scored by the same critic, and the GSM8K base pool to its $6$
trained arms. Measured arm AUC spans $0.408$ to $0.937$.

\paragraph{Why ordinary checkpoints are not a designed two-mixture audit.}
Suppose an auditor observes $(A_1,q_1)$ and $(A_2,q_2)$ but the conditional terms drift from
$(a_1,c_1)$ to $(a_2,c_2)$. If it nevertheless applies the fixed-component inversion
$\hat a=(q_2A_1-q_1A_2)/(q_2-q_1)$, then relative to the first pool
\[
\hat a-a_1=\frac{q_1(1-q_2)(a_1-a_2)+q_1q_2(c_1-c_2)}{q_2-q_1}.
\]
Thus separation in mixture weights cannot remove conditional drift; it can amplify it when
$|q_2-q_1|$ is small. Proposition~\ref{prop:audit_cost}(ii) is therefore a statement about
\emph{controlled reweighting/resampling of fixed bin-conditional score distributions}. The
transfer experiments below ask how far that assumption approximately travels across real policy
changes; they are diagnostics, not the designed-audit guarantee.

\begin{figure}[t]
\centering
\includegraphics[width=0.62\linewidth]{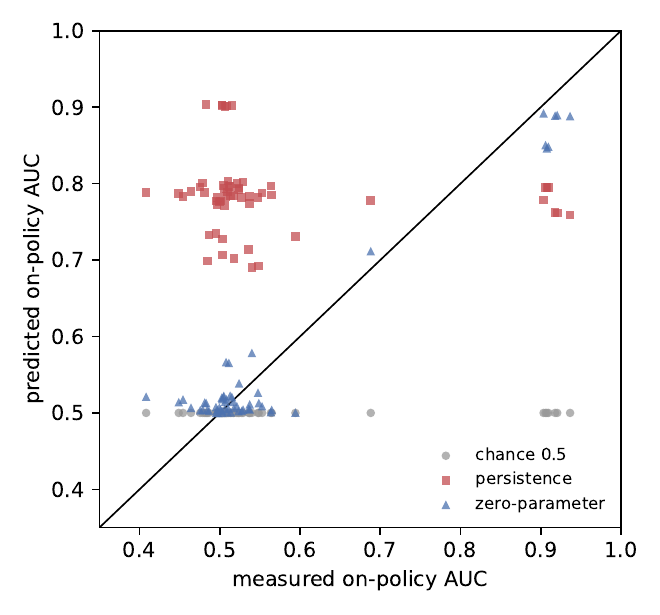}
\caption{Forecast against measurement for all $57$ target arms. Each point is one
target policy or pool. The zero-parameter form tracks the diagonal across the
measured range; persistence sits in a band set by the source audit regardless of the
target; chance is correct only in the saturated cluster.}
\label{fig:transfer}
\end{figure}

\begin{table}[h]
\centering
\caption{Arm-level forecast error over all $57$ target arms: absolute error of the
arm's mean predicted AUC against its mean measured AUC. The bootstrap resamples
requirements jointly across arms and holds the training seeds fixed; it is
requirement uncertainty, not population variance over training runs.}
\label{tab:transfer}
\small
\begin{tabular}{lcccc}
\toprule
predictor & MAE & worst arm & signed bias & $r$(pred, measured) \\
\midrule
chance $\tfrac12$ & $0.0745$ [$0.0742,0.0821$] & $0.4366$ & $-0.0622$ & --- \\
persistence & $0.2566$ [$0.2432,0.2707$] & $0.4204$ & $+0.2231$ & $-0.093$ \\
$(1-q)A_{\rm own}+q/2$ & $0.0280$ [$0.0275,0.0357$] & $0.1130$ & $-0.0012$ & $0.967$ \\
$(1-q)A_{\rm own}+qA_{\rm cross}$ & $0.0325$ [$0.0322,0.0496$] & $0.1162$ & $+0.0010$ & $0.956$ \\
\bottomrule
\end{tabular}
\end{table}

\begin{table}[h]
\centering
\caption{Arm-level forecast MAE by protocol. Chance is accurate exactly where the
target has saturated its own conjunct, and fails on the cVAE pools, which have not.}
\label{tab:transfer_protocol}
\small
\setlength{\tabcolsep}{4pt}
\begin{tabular}{lccccc}
\toprule
source $\rightarrow$ target & arms & chance & persistence & zero-param & transported \\
\midrule
1.5B base $\rightarrow$ 1.5B trained & $26$ & $0.0305$ & $0.2767$ & $0.0279$ & $0.0374$ \\
7B base $\rightarrow$ 7B trained & $9$ & $0.0302$ & $0.1868$ & $0.0278$ & $0.0291$ \\
1.5B base $\rightarrow$ 7B trained & $9$ & $0.0252$ & $0.2668$ & $0.0262$ & $0.0223$ \\
1.5B base $\rightarrow$ cVAE pools & $7$ & $0.4145$ & $0.1365$ & $0.0425$ & $0.0416$ \\
GSM8K base $\rightarrow$ GSM8K trained & $6$ & $0.0088$ & $0.3991$ & $0.0144$ & $0.0213$ \\
\bottomrule
\end{tabular}
\end{table}

The $11$ target arms whose own conjunct has not saturated (mean $q<0.9$: the seven
cVAE pools, two low-$\beta$/low-$\tau$ mitigation arms and two earlier 7B DPO arms)
carry the discriminating evidence, because a chance forecast cannot be right there
by construction. Their measured AUC spans $0.508$ to $0.937$, and the zero-parameter
MAE is $0.0429$ (worst arm $0.0613$) against $0.2863$ for chance and $0.1596$ for
persistence. Transporting the measured cross term gives $0.0404$.

A further post-hoc baseline makes the interpretation of the full $57$
comparisons more conservative. The perfect-head predictor $1-q/2$ has MAE $0.0261$,
and a leave-one-arm-out linear regression of AUC on $q$ has MAE $0.0261$, both
slightly below the source-own forecast's $0.0280$. Deduplicating the repeated 7B
targets leaves $48$ unique target pools and gives the same conclusion ($0.0261$ and
$0.0259$). Thus the stored-arm analysis demonstrates the composition relation over a
wide policy range, but does not show incremental predictive value from source-bin
terms beyond $q$ alone. That incremental comparison is reserved for the later arms
below.

\paragraph{What this does and does not establish.} The forecast avoids re-auditing
the head's ordering on the target policy; it does not avoid target-side labels,
since $q$ is an oracle-label statistic of the target pool. All arms existed before
this control was written, so this is a post-hoc transfer protocol across stored
pools, not a preregistered forward forecast of an unrun experiment. Transporting
$A_{\rm own}$ assumes the head's ordering against own-conjunct failures is stable
across policies; the residual errors in Table~\ref{tab:transfer} bound how far that
assumption travels here and do not license it for arbitrary policies, critics or
tasks. Accuracy on the saturated arms is unimpressive on its own, since chance
achieves it. Code: \texttt{scripts/predict\_auc\_transfer.py}; complete per-arm
tables in \texttt{results/predictive\_audit/}.

\paragraph{Arms trained after the forecast was frozen.} The forecast code and the
1.5B base pool's bin terms were fixed on the day the control was written; the nine
strong-controls arms (Appendix~\ref{app:pair_audit}) and the three three-conjunct
arms (Appendix~\ref{app:physics_gate}) were trained afterwards, with predictions
registered before any term was computed on their pools
(\texttt{PLAN\_NEW\_ARMS\_FORECAST\_2026-09-19.md}). What remains post-hoc is that
their outcome rates had been read, so the rough location of $q_t$ could be
anticipated; that anticipation is written in the plan. We therefore call this a
pre-specified transfer test with frozen source terms, not a fully prospective forecast.
Table~\ref{tab:prospective} gives the result. The registered arm-level error of $0.030$ (band $\leq0.05$) was
met ($0.020$). Two post-hoc baselines address how much of this is merely $q$: the
perfect-head form $1-q/2$ gives MAE $0.030$, and a linear AUC--$q$ model fitted once
on the earlier $57$ comparisons gives $0.025$. On the six non-saturated arms the
corresponding errors are $0.016$, $0.036$ and $0.027$. A post-hoc two-level bootstrap resamples the $1\,000$ committed test requirements
jointly across arms and, within each of the four training rules, resamples the three
observed training seeds. It gives source-own advantages over $1-q/2$ of $0.0098$
$[0.0046,0.0139]$ and over the historical linear baseline of $0.0054$
$[0.0005,0.0096]$. On the six non-saturated arms the corresponding gaps are $0.0193$
$[0.0092,0.0271]$ and $0.0100$ $[0.0006,0.0177]$. A requirement-only bootstrap with
arms/seeds fixed gives the slightly narrower intervals reported in the artifact.
These are empirical uncertainty checks over the observed seeds and requirements,
not population intervals over unseen model families or training recipes. Averaged within training-rule families, source-own improves on
$1-q/2$ in all four families and on the linear baseline in three of four; the
three-conjunct family is essentially tied with the latter. On that subset persistence was slightly better than the
registered ``$>0.10$'' ($0.093$), a miss that leaves the ranking of the registered
predictors unchanged. Spearman correlation between measured AUC and $q$ over the fifteen
single-gate 1.5B arms is $-0.82$; $A_{\rm own}$ lies in $[0.92, 0.98]$ on every one
of the eighteen arms, and the three-conjunct training leaves the feasibility head's
terms where the two-conjunct gate had them ($q$ $0.946$ against $0.916$).
Artifact: \texttt{results/new\_arms\_audit/}; code \texttt{scripts/audit\_new\_arms.py}.

\begin{table}[h]
\centering
\caption{Pre-specified transfer test on the twelve arms trained after the source terms were
frozen (means of three seeds; 1.5B base as source). $q_t$, measured AUC and the
forecast are over the requirements whose pools support the terms in both source and
target; $A_{\rm own}$ is the target arm's own, shown for comparison and not used.
``Forecast'' is the zero-parameter form $(1-q_t)A_{\rm own}^{\rm src}+q_t/2$.}
\label{tab:prospective}
\small
\setlength{\tabcolsep}{4pt}
\begin{adjustbox}{max width=\linewidth}
\begin{tabular}{lcccccc}
\toprule
training rule & feasible & $q_t$ & measured AUC & $A_{\rm own}$ & forecast & $|$error$|$ \\
\midrule
tuned sum ($w=0.005$) & $0.658$ & $0.642$ & $0.644$ & $0.928$ & $0.661$ & $0.017$ \\
log-prob & $0.871$ & $0.993$ & $0.509$ & $0.947$ & $0.503$ & $0.034$ \\
mass-only & $0.451$ & $0.429$ & $0.748$ & $0.948$ & $0.757$ & $0.015$ \\
three-conjunct gate & $0.812$ & $0.961$ & $0.518$ & $0.958$ & $0.517$ & $0.013$ \\
\midrule
\multicolumn{7}{l}{12 arms: source-own $0.020$, transported $0.019$, $1-q/2$ $0.030$, historical linear AUC--$q$ $0.025$} \\
\multicolumn{7}{l}{6 non-saturated arms: source-own $0.016$, transported $0.012$, $1-q/2$ $0.036$, historical linear AUC--$q$ $0.027$} \\
\bottomrule
\end{tabular}
\end{adjustbox}
\end{table}

\paragraph{Along a run.} The within-run version, registered before any checkpoint
pool was generated (\texttt{PLAN\_WITHIN\_RUN\_FORECAST\_2026-09-19.md}), forecasts
the head's on-policy AUC at checkpoint $k{+}1$ from the bin terms at checkpoint $k$,
exactly as across arms. Four seed-0 runs with adapters saved every $250$ steps
(log-probability, tuned sum, three-conjunct gate, mass-only) were sampled on the
first $250$ committed test requirements at $K{=}16$; step $0$ and step $1\,500$ are
the corresponding rows of the stored base and final pools. Table~\ref{tab:within_run}
and Figure~\ref{fig:within_run} give the result. Over the $18$ transitions of the
three arms whose $q$ moves, the zero-parameter forecast has arm-level error $0.049$
(registered $0.03$, band $\leq0.06$) against $0.096$ for persistence and $0.138$ for
chance, and beats persistence on $15$ of $18$; from the base pool's terms alone the
error over every later checkpoint is $0.032$. That per-requirement comparison is
confined to the requirements where the source checkpoint still has both negative
bins, which on a saturated checkpoint are its least saturated ones; applying the
source's pooled own term to every target requirement instead (the arm-level form
used out of distribution, and what Figure~\ref{fig:within_run} plots) gives $0.021$
against $0.064$ for persistence and $0.080$ for chance, $13$ of $18$ transitions in
the forecast's favour, worst transition $0.064$. Saturation is fast: on the
log-probability and three-conjunct runs $q$ exceeds $0.9$ by step $500$ (registered),
the first transition is persistence's worst, and the head's AUC has already fallen
to $0.54$ and $0.49$ at step $250$ while per-draw success is $0.09$ and $0.16$. The
mass-only run is the registered null: its $q$ stays in $[0.28, 0.41]$, persistence
is within $0.030$, and the forecast gains only $0.010$ on it --- the forecast wins
where $q$ moves and only there. Two registered checks missed: $A_{\rm own}\geq0.90$ at
every checkpoint fails once ($0.853$, log-probability step $1\,000$, on a saturated
pool with few own-bin pairs among $250$ requirements), and the tuned sum's $q$
reached $0.498$ rather than $>0.5$ by step $500$. Single seed, $250$ requirements,
and generation batches that differ between the checkpoint pools and the stored
endpoints are the limits of this control. Artifact:
\texttt{results/within\_run\_forecast/}; code \texttt{scripts/forecast\_within\_run.py}.

\begin{table}[h]
\centering
\caption{Within-run forecast from the previous checkpoint: arm-level mean absolute
error of each predictor over the six transitions of each run ($250$ requirements,
$K{=}16$, seed 0). The mass-only run is the registered null case whose $q$ does not
move.}
\label{tab:within_run}
\small
\begin{tabular}{lccccc}
\toprule
run & $q$: step 0 $\to$ 250 $\to$ 1500 & chance & persistence & zero-param & transported \\
\midrule
log-probability & $0.34 \to 0.95 \to 0.97$ & $0.095$ & $0.116$ & $0.074$ & $0.119$ \\
tuned sum $w{=}0.005$ & $0.34 \to 0.44 \to 0.59$ & $0.214$ & $0.052$ & $0.014$ & $0.026$ \\
three-conjunct gate & $0.34 \to 0.88 \to 0.93$ & $0.105$ & $0.120$ & $0.061$ & $0.083$ \\
\midrule
mass-only (null) & $0.34 \to 0.30 \to 0.33$ & $0.282$ & $0.030$ & $0.020$ & $0.015$ \\
\bottomrule
\end{tabular}
\end{table}

\begin{figure}[h]
\centering
\includegraphics[width=0.85\linewidth]{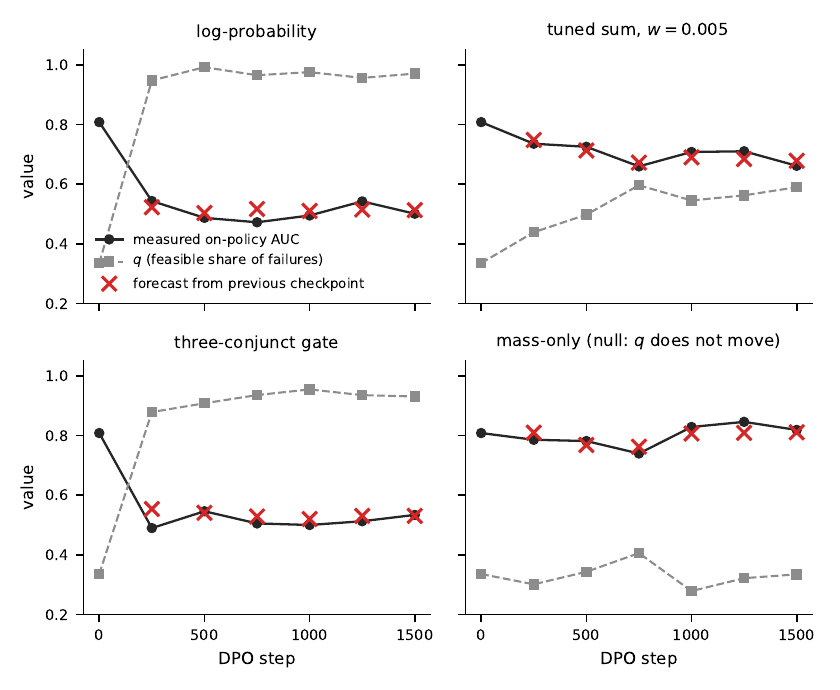}
\caption{Measured on-policy AUC of the feasibility head (dots), $q$ (squares) and the
zero-parameter forecast from the previous checkpoint's pooled own term (crosses)
along four DPO runs; $250$ requirements, $K{=}16$, seed 0.}
\label{fig:within_run}
\end{figure}

\paragraph{Out of distribution.} The benchmark ships a fourth split of $500$
requirements drawn beyond the training prior (thrust to $61.8$\,N against a
$5$--$60$\,N prior, expansion ratio to $23.3$). Fifteen stored cVAE pools answer it.
Requirement indices do not correspond across splits, so this forecast is arm-level:
each generator family's mean bin terms from the test split, that family's mean $q$ on
the OOD split, and its OOD AUC predicted. Over the five families the zero-parameter
mean absolute error is $0.0282$, against $0.0752$ for persistence and $0.3640$ for
chance, and the identity still reconstructs every OOD requirement to $10^{-9}$. Four
families land within $0.024$. The capped decoder is the exception, missing by $0.101$,
and the terms say why: its own-bin ordering does not transport, falling from
$A_{\rm own}=0.965$ on the test split to $0.821$ out of distribution. That is the
assumption the forecast rests on and this arm is where it breaks. The transported
variant fares worse still on latent search, whose $A_{\rm cross}$ drops from $0.817$
to $0.651$. Artifact: \texttt{results/ood\_split/audit.json}; code
\texttt{scripts/audit\_ood\_split.py}.

\section{A label-free version of the diagnostic does not transfer}
\label{app:label_free}
Every measurement above needs oracle labels: $q$, $A_{\rm own}$ and $A_{\rm cross}$
are all label statistics. A deployed pipeline has none, so we asked whether the
critic's own scores already reveal that an additive composition is losing best-of-$K$
lift. Three label-free statistics were computed per arm from the two term arrays
alone: the share of pools whose sum-argmax equals the saturated term's argmax, the
median within-pool SD of the saturated term over the other term's, and the share of
candidates the saturated head calls satisfied with $p>0.99$. They were asked to
predict two oracle-measured quantities: $\eta_{\rm raw}$, the share of attainable
best-of-$K$ lift the additive rule captures, and the gain from switching to a gate.
Leave-one-arm-out linear fits were compared against predicting the training-fold mean
over $32$ engine arms and the same $32$ plus $7$ GSM8K arms.

On the engine arms alone the relations look usable: saturated fraction against gate
gain reaches $r=0.70$ with leave-one-out MAE $0.140$ against $0.227$ for the mean,
and the SD ratio against $\eta_{\rm raw}$ reaches $r=0.73$ with MAE $0.112$ against
$0.151$. The SD ratio's sign is nevertheless opposite to the mechanism, because a
dominant feasibility term also marks pools where feasibility genuinely discriminates.
Adding GSM8K removes the effect: saturated fraction against gate gain falls to
$r=0.02$ and the SD ratio to $r=-0.14$, both worse than predicting the mean
($0.339$ and $0.336$ against $0.332$).

We therefore make no label-free claim. This is a negative result on three specific
statistics and two targets, not a proof that no label-free diagnostic exists; a
learned combination, per-requirement rather than per-arm targets, or a calibrated
head might do better. Code: \texttt{scripts/label\_free\_diagnostic.py}; full table
in \texttt{results/label\_free\_diagnostic/audit.json}.

\section{Feasibility's own sub-constraints: a case where the mix moves and the AUC does not}
\label{app:subconstraint}
The packaged oracle defines \texttt{feasible} as \texttt{printable AND
fits\_envelope}: no DfAM flag (overhang, minimum wall, channel-versus-ligament,
minimum TPMS cell) and length $\le$ envelope (the radius clause $2r_{\max}\le$
envelope cannot bind: $2\times55\,\mathrm{mm}<200\,\mathrm{mm}$). Both are
consequential geometric constraints, and neither is syntactic. The main text treats
feasibility as one conjunct; here we split it.

The two are learned at very different rates. Printability rises from $0.396$ at the
1.5B base and $0.544$ at the 7B base to $0.93$--$1.00$ across every matched DPO and
GRPO arm; envelope fit starts at $0.808$ in both bases and ends between $0.811$ and
$0.880$. The envelope rate is capped: for $4.9\%$ of test requirements the sized
length exceeds the envelope even at the policy's shortest setting, so no design fits.
Among attainable requirements the fit rate moves from $0.85$ to $0.92$ under
sum-ranked pairs and stays at $0.86$--$0.87$ under gated pairs
(Appendix~\ref{app:conditioning}). The composition of the head's negative bin
therefore shifts hard: writing $p=\Pr[\text{printable}\mid\text{not feasible}]$, the
share of remaining failures that print but do not fit rises from $0.081$ (1.5B base)
and $0.052$ (7B base) to $0.21$--$0.93$ after training.

Splitting the feasibility head's negatives the same way gives
$\mathrm{AUC}(s_F;F) = (1-p)A_{\rm print} + p\,A_{\rm env}$, checked to $10^{-9}$ per
requirement. Unlike the conjunctive case of the main text, the inherited term is not
near chance: $A_{\rm env}$ lies between $0.889$ and $0.995$ on every arm, and
$\mathrm{AUC}(s_F;F)$ stays between $0.940$ and $0.991$ throughout. The extreme case
is the 7B sum arm at seed 2, where $p=0.932$ --- almost every remaining failure is an
envelope failure --- and the AUC is $0.983$.

This is the control the conjunctive claim needs. Negative composition moves at least
as far here as it does in the main text ($q$ reaching $0.98$), in the same pools and
with the same critic, and no collapse follows, because the head kept its ordering in
the bin it inherited. Composition and bin quality are separately necessary; neither
on its own predicts the diagnostic. The envelope-only bin is sparse on the strongest
arms ($33$--$201$ supported requirements, listed per arm in the artifact), so the
$A_{\rm env}$ estimates there carry correspondingly wide uncertainty; the direction
of the contrast does not depend on any single arm. Code:
\texttt{scripts/audit\_feasibility\_substructure.py}; per-arm table in
\texttt{results/subconstraint\_audit/audit.json}.

\section{The same critic as a training target and as a search objective}
\label{app:families}
The cVAE family lets the critic enter the generator in two different places, on one
design space with one oracle. Five families, three decoder seeds each, $1\,000$ test
requirements and $K=64$; mean $\pm$ seed SD.

\begin{table}[h]
\centering
\caption{Generator families on the engine task. Per-draw on-target is flat across the
first three families while supply collapses, which is the dissociation of
\S\ref{sec:select} in the generator-training direction rather than the selection
direction. Latent search is not compute-matched with the decoders.}
\label{tab:families}
\small
\begin{tabular}{lccc}
\toprule
generator & feasible & on target & supply@$64$ \\
\midrule
prior, random latent           & $0.081 \pm 0.001$ & $0.0074 \pm 0.0002$ & $0.321 \pm 0.013$ \\
decoder, $\lambda=0$           & $0.093 \pm 0.005$ & $0.0086 \pm 0.0005$ & $0.336 \pm 0.003$ \\
decoder, $\lambda=1$           & $0.101 \pm 0.120$ & $0.0090 \pm 0.0111$ & $0.091 \pm 0.056$ \\
decoder, $\lambda=1$, capped   & $0.036 \pm 0.038$ & $0.0038 \pm 0.0038$ & $0.072 \pm 0.060$ \\
latent search on $\lambda=0$   & $0.481 \pm 0.047$ & $0.3417 \pm 0.0294$ & $0.937 \pm 0.017$ \\
\bottomrule
\end{tabular}
\end{table}

Training the decoder against the critic's two terms leaves per-draw on-target where
the untrained prior already was and removes two thirds of the coverage; capping the
output layer, the obvious fix for the escape of Appendix~\ref{app:escape}, does not
recover it and in these runs is worse. The $\lambda=1$ rows carry very large seed
spread --- one seed reaches feasible $0.238$ and another $0.049$ --- so we read the
supply collapse, which is consistent across seeds, and not the feasible column.
Using the same critic as a search objective over the same decoder's latent space
instead reaches $0.342$ on target and $0.937$ supply.

We do not claim an optimality result for latent search, which is not compute-matched,
and we do not claim that no decoder training against a critic can work. The reading we
do take is that a critic which ranks candidates well (Table~\ref{tab:law}) is not
thereby a good objective for shaping the generator that produces them, and that the
failure shows up as lost coverage rather than lost per-draw quality --- the same
place it shows up in the language-model arms. The ordering survives the
out-of-distribution split: supply at $K{=}64$ falls to $0.256$ for the prior and
$0.255$ for $\lambda=0$, stays at $0.095$ and $0.083$ for the two critic-trained
decoders --- which have no coverage left to lose --- and falls to $0.596$ for latent
search, which degrades most in relative terms and still dominates. Artifacts:
\texttt{results/generator\_families/audit.json} and
\texttt{results/ood\_split/audit.json}.

\section{What the generator learns to condition on the requirement}
\label{app:conditioning}
The engine target has one conjunct that does not depend on the requirement and two
that do. \emph{Printability} is a set of DfAM rules on walls, ligaments, overhang and
cell size; the requirement enters only through the bore radii. \emph{Envelope fit}
is, from the oracle code, $L(\text{thrust},\varepsilon)\cdot(0.85+0.4\,v_{30})\le
\text{envelope}$: one design coordinate, the length scale $v_{30}$, against a sized
length and a per-prompt envelope. \emph{Mass on target} needs the body's mass to
track the prompt's budget $m^\star$. Reading the rule off the code and re-applying it
to stored candidates reproduces the oracle's envelope flag on $100.0\%$ of parsed
candidates on all $39$ pools. Table~\ref{tab:conditioning} measures, per generator,
whether each conjunct is satisfied as a habit or as a function of the prompt: the
partial correlation of the requirement-mean $v_{30}$ with $\log$ envelope given
$\log$ sized length, and the OLS gain of requirement-mean $\log m$ on $\log m^\star$.
Every quantity is behavioural, measured on outputs; nothing here inspects a
representation inside a model.

\begin{table}[t]
\centering
\caption{What each generator conditions on the requirement. \emph{Printable} is the
requirement-independent conjunct; its rate and its spread across thrust terciles (pp).
\emph{Envelope} is the requirement-dependent conjunct with one control, the length
scale $v_{30}$: partial correlation of the requirement-mean $v_{30}$ with
$\log$ envelope given $\log$ sized length (reads the envelope?), and the fit rate
where fitting is attainable at all ($95.1\%$ of requirements). \emph{Mass}: OLS slope
of requirement-mean $\log m$ on $\log m^\star$ (a gain; $1$ is one-to-one tracking)
and its partial correlation given thrust. Test split, mean $\pm$ SD over three seeds;
single-seed rows unmarked. LLM arms $K{=}16$, cVAE $K{=}64$. $^\dagger$Arms trained
after the audit was written, with predictions registered beforehand.}
\label{tab:conditioning}
\small
\setlength{\tabcolsep}{3.5pt}
\begin{adjustbox}{max width=\linewidth}
\begin{tabular}{l cc cc cc c}
\toprule
& \multicolumn{2}{c}{printable} & \multicolumn{2}{c}{envelope} & \multicolumn{2}{c}{mass} & \\
\cmidrule(lr){2-3}\cmidrule(lr){4-5}\cmidrule(lr){6-7}
generator & rate & $\Delta$thrust & $\rho_{\rm part}$ & fit$\mid$att. & gain & $\rho_{\rm part}$ & on target \\
\midrule
1.5B base & $0.40$ & $42$ & $0.22$ & $0.85$ & $0.04$ & $0.16$ & $0.027$ \\
1.5B DPO, sum & $0.99 \pm 0.01$ & $3 \pm 2$ & $0.80 \pm 0.04$ & $0.92 \pm 0.00$ & $0.48 \pm 0.04$ & $0.87 \pm 0.01$ & $0.134 \pm 0.005$ \\
1.5B DPO, gate & $0.98 \pm 0.01$ & $3 \pm 1$ & $0.13 \pm 0.02$ & $0.86 \pm 0.01$ & $1.03 \pm 0.08$ & $0.97 \pm 0.00$ & $0.311 \pm 0.013$ \\
1.5B DPO, gate, round 2 & $1.00 \pm 0.00$ & $1 \pm 0$ & $0.45 \pm 0.14$ & $0.95 \pm 0.01$ & $1.11 \pm 0.02$ & $0.97 \pm 0.00$ & $0.363 \pm 0.042$ \\
1.5B DPO, tuned sum $w{=}0.005$$^\dagger$ & $0.80 \pm 0.03$ & $35 \pm 4$ & $0.08 \pm 0.01$ & $0.89 \pm 0.01$ & $0.74 \pm 0.04$ & $0.97 \pm 0.00$ & $0.258 \pm 0.028$ \\
1.5B DPO, contrast-matched sum $w{=}0.1$$^\dagger$ & $0.99 \pm 0.00$ & $3 \pm 1$ & $0.69 \pm 0.06$ & $0.92 \pm 0.00$ & $0.98 \pm 0.02$ & $0.96 \pm 0.00$ & $0.312 \pm 0.021$ \\
1.5B DPO, log-prob$^\dagger$ & $0.99 \pm 0.00$ & $1 \pm 1$ & $0.79 \pm 0.05$ & $0.92 \pm 0.00$ & $0.72 \pm 0.04$ & $0.90 \pm 0.00$ & $0.180 \pm 0.015$ \\
1.5B DPO, mass-only$^\dagger$ & $0.57 \pm 0.07$ & $36 \pm 6$ & $0.11 \pm 0.02$ & $0.88 \pm 0.01$ & $0.73 \pm 0.02$ & $0.97 \pm 0.00$ & $0.190 \pm 0.035$ \\
1.5B DPO, three-conjunct gate$^\dagger$ & $0.99 \pm 0.01$ & $1 \pm 1$ & $0.11 \pm 0.02$ & $0.86 \pm 0.00$ & $1.06 \pm 0.02$ & $0.96 \pm 0.00$ & $0.257 \pm 0.002$ \\
\addlinespace
7B base & $0.54$ & $22$ & $0.63$ & $0.85$ & $0.30$ & $0.68$ & $0.059$ \\
7B DPO, sum & $1.00 \pm 0.00$ & $0 \pm 0$ & $0.92 \pm 0.02$ & $0.92 \pm 0.01$ & $0.20 \pm 0.01$ & $0.63 \pm 0.04$ & $0.107 \pm 0.004$ \\
7B DPO, gate & $0.94 \pm 0.01$ & $11 \pm 2$ & $0.55 \pm 0.07$ & $0.87 \pm 0.01$ & $0.62 \pm 0.06$ & $0.93 \pm 0.01$ & $0.189 \pm 0.010$ \\
1.5B GRPO, sum & $0.99 \pm 0.00$ & $2 \pm 1$ & $0.68 \pm 0.03$ & $0.90 \pm 0.00$ & $0.20 \pm 0.02$ & $0.61 \pm 0.03$ & $0.086 \pm 0.002$ \\
1.5B GRPO, soft gate & $0.96 \pm 0.01$ & $7 \pm 1$ & $0.60 \pm 0.01$ & $0.90 \pm 0.00$ & $0.21 \pm 0.02$ & $0.63 \pm 0.03$ & $0.089 \pm 0.003$ \\
\addlinespace
cVAE prior & $0.13$ & $15$ & $0.01$ & $0.80$ & $0.00$ & $0.00$ & $0.007$ \\
decoder, $\lambda=0$ & $0.15 \pm 0.01$ & $16 \pm 1$ & $0.06 \pm 0.02$ & $0.80 \pm 0.01$ & $-0.00 \pm 0.00$ & $-0.03 \pm 0.03$ & $0.009 \pm 0.001$ \\
decoder, $\lambda=1$ & $0.11 \pm 0.13$ & $7 \pm 5$ & $0.07 \pm 0.06$ & $1.00 \pm 0.00$ & $-0.00 \pm 0.00$ & $-0.04 \pm 0.03$ & $0.009 \pm 0.011$ \\
decoder, $\lambda=1$, capped & $0.04 \pm 0.04$ & $6 \pm 9$ & $0.09 \pm 0.12$ & $1.00 \pm 0.00$ & $-0.00 \pm 0.00$ & $-0.02 \pm 0.01$ & $0.004 \pm 0.004$ \\
latent search on $\lambda=0$ & $0.61 \pm 0.00$ & $14 \pm 0$ & $0.42 \pm 0.00$ & $0.93 \pm 0.00$ & $0.76 \pm 0.00$ & $0.99 \pm 0.00$ & $0.374 \pm 0.003$ \\
\bottomrule
\end{tabular}
\end{adjustbox}
\end{table}

\paragraph{The requirement-independent conjunct goes to ceiling under every rule
that contrasts it.} Printability is $0.40$ and $0.54$ at the two bases and
$0.93$--$1.00$ on every DPO and GRPO arm whose pairs contrast feasibility; its
dependence on thrust, $42$ and $22$\,pp across terciles at the bases, falls to
$0$--$5$\,pp on those arms. This is the conjunct whose head saturates in the main
text. The two rules whose pairs barely contrast feasibility (Table~\ref{tab:pair_content})
leave it short of ceiling: $0.57\pm0.07$ under mass-only, whose pairs contrast it by
$0.03$, and $0.80\pm0.03$ under the tuned sum ($0.13$), with the thrust dependence
of the base intact ($35$--$36$\,pp). We had registered the mass-only value as
$0.60$ (band $<0.80$) before computing it.

\paragraph{Which requirement-dependent conjunct is conditioned on is set by the
ranking rule.} Pairs ranked by the unit-weight sum, which the saturated feasibility
logit dominates, install envelope conditioning: the partial correlation of $v_{30}$
with the envelope rises from $0.22$ to $0.76$--$0.84$ at 1.5B and from $0.63$ to
$0.91$--$0.94$ at 7B, on every seed, while the mass gain reaches only $0.44$--$0.52$
at 1.5B and \emph{falls} from the 7B base's $0.30$ to $0.19$--$0.22$. Pairs ranked
by the hard gate leave envelope conditioning where the base had it ($0.12$--$0.16$ at
1.5B; $0.48$--$0.62$ at 7B) and install mass conditioning with near-unit gain:
$1.03\pm0.08$ at 1.5B with partial correlation $0.97$ on every seed, $0.62\pm0.06$ at
7B. At 1.5B the across-requirement SD of the gated policy's mean $\log m$ is $0.66$,
the SD of the log-uniform budget prior itself being $0.61$: the policy's mass
distribution has become the budget distribution. Per-draw on-target tracks mass gain, the quantity
on-target requires: Spearman $0.96$ over the seven trained LLM families and $0.88$
with the two bases, whose printability is still far from ceiling, included. The two GRPO rewards, whose effect on supply and per-draw is
small in \S\ref{sec:train}, are indistinguishable here as well ($0.68$ vs $0.60$
envelope, $0.20$ vs $0.21$ mass). Both DPO rules still vary $v_{30}$ within a
requirement (SD $0.04$--$0.07$); ``not conditioned'' means not set by the prompt,
not held constant.

\paragraph{Scalar rules that contrast mass install partial tracking; the gates
install full tracking.} The four arms trained after this audit was written
($\dagger$ in Table~\ref{tab:conditioning}) test the rule-chooses-the-conjunct
claim prospectively. On envelope they behave as registered: the tuned sum and
mass-only leave it where the base had it ($0.08$, $0.11$), the log-probability form
installs it like the sum ($0.79$), and the three-conjunct gate leaves it like the
gate ($0.11$). On mass they do not: we registered gains of $0.95$ for the tuned sum
and mass-only (band $[0.8,1.15]$) and $0.45$ for log-prob (band $[0.3,0.65]$), and
all three land at $0.72$--$0.74$, whatever their pairs' mass contrast ($0.18$ for
log-prob, $0.81$--$0.84$ for the other two); the three-conjunct gate lands at
$1.06$, as registered. Across the five single-gate rules the trained printable rate
follows the pairs' feasibility contrast (Spearman $0.9$) but mass gain does not
follow the mass contrast ($0.6$): both gates reach unit gain with an
across-requirement mass SD of $0.66$--$0.68$ that matches the budget prior's
$0.61$, while these scalar rules stop at $0.44$--$0.50$. The contrast-matched sum of
Appendix~\ref{app:pair_audit} settles what this means: at $w=0.1$ a scalar score
installs unit budget tracking ($0.98\pm0.02$, mass SD $0.64$) \emph{and} envelope
conditioning ($0.69\pm0.06$), the only rule that installs both. The plateau belonged
to the wrong weights, not to scalar scores.

\paragraph{The continuous decoder takes the unconditional shortcut; search does not.}
No cVAE decoder conditions mass on the budget (gain $-0.01$--$0.00$ on all nine
seeds, prior included), although the decoder is conditional by construction. Trained
against the critic's terms, the $\lambda=1$ and capped decoders satisfy the envelope
by collapsing $v_{30}$ to $0.000\pm0.001$ on every seed --- the shortest allowed
length for every prompt --- reaching fit $1.000$ while printability falls to
$0.00$--$0.25$. Latent search over the untrained decoder, under the same critic,
attains mass gain $0.76$ with partial correlation $0.99$: in the continuous
representation the prompt dependence arrives through search at inference, not through
the learned map. This is the same contrast as Table~\ref{tab:families}, one level
down: the critic that shapes a token policy's conditioning through pair ranking
shapes a decoder into a constant.

\paragraph{Registered predictions and misses.} Five predictions were written before
the audit (\texttt{context/docs/PLAN\_CONDITIONING\_2026-09-18.md}); each part is
scored there. Held: sum gain below gate gain at both scales; printability flat
across envelope classes; gated arms spread mass across requirements more than sum
arms. Missed, in the informative direction: we predicted that trained token policies
would take the decoders' shortcut (global shrink, $|\rho|<0.15$) --- they condition
instead, on the conjunct the rule rewards; that gated mass gain would lie in
$0.30$--$0.70$ --- at 1.5B it is $1.03$; that the base 7B model would not read the
budget --- it does, partially ($0.30$); and that sum arms would narrow their mass
below the base's spread --- at 1.5B they widen it, because they track the budget
partially. Also missed: the attainable fraction is $0.951$, not $0.90$, and
printability on the 7B gated and GRPO gated arms keeps a $6$--$14$\,pp thrust
dependence. Base and single-lineage rows
are one seed and all arms pre-existed the audit; this is a post-hoc characterisation
with pre-registered quantities, not a prospective test. Code:
\texttt{scripts/audit\_conditioning.py}, tests in
\texttt{scripts/test\_conditioning.py}; per-arm values in
\texttt{results/conditioning\_audit/audit.json}.

\section{A physics conjunct outside every score}
\label{app:physics_gate}
Nothing in the paper's engine target asks whether the printed body would survive a
firing. The repository's print-and-fire readiness screen does: Bartz throat heat
flux at the operating point the requirement implies ($p_c = F/(C_F A_t) = 1.30$\,bar
for every requirement, since $r_t \propto \sqrt{F}$; $\phi=1$, $T_c = 2287$\,K, $c^* =
1249$\,m/s, within $1.1\%$ of the repository's Cantera values); a closed-form steady
hot face through the liner, cooled by regenerative air at 9\,bar through the body's
own cooling-zone channels (film coefficient from the air velocity in the channels that
fit mid-wall); von Mises margin against the AlSi10Mg yield at the hot-face
temperature; and a Coffin--Manson low-cycle-fatigue life of at least $20$ fires. It is
closed-form, $0.14$\,ms per candidate, and we applied it to every stored candidate in
all $39$ pools (Table~\ref{tab:physics_gate}). Ignition is excluded because $L^*$ is
fixed by the oracle's sizing and the verdict is constant. Neither the critic nor any
policy was trained on any of this; the critic scores printability, envelope and mass.

\begin{table}[t]
\centering
\caption{A physics conjunct no score has seen. Every stored candidate is re-graded by
the readiness screen (hot-face temperature of the AlSi10Mg liner under Bartz throat
flux, cooled by regenerative air through its own channels; yield margin at
temperature; Coffin--Manson life $\ge 20$ fires) at the operating point the
requirement implies ($p_c=1.30$\,bar, $\phi=1$). $A_{\rm cool}$: mean total
cooling-channel flow area (mm$^2$), which sets the coolant velocity and hence the hot
face. \emph{ready}$\mid$\emph{feasible}: survival among feasible bodies. Supply and
two-stage answer are given for the original target (on target) and the extended one
(on target $\wedge$ ready), with the screen's selection efficiency $\eta_{\rm ext}$
and the efficiency of the same rule if it could also see the readiness verdict.
Mean $\pm$ SD over seeds; $K{=}16$ (LLM) and $64$ (cVAE). $^\dagger$Arms trained after
the audit was written; the three-conjunct gate ranks pairs by feasibility $\wedge$
readiness, then mass.}
\label{tab:physics_gate}
\small
\setlength{\tabcolsep}{2.6pt}
\begin{adjustbox}{max width=\linewidth}
\begin{tabular}{l cccc cc cc cc}
\toprule
& & $T_{\rm hot}$ & no cooling & ready & \multicolumn{2}{c}{supply@$K$} & \multicolumn{2}{c}{two-stage answer} & & \\
\cmidrule(lr){6-7}\cmidrule(lr){8-9}
generator & $A_{\rm cool}$ & med.\ [$^\circ$C] & zone & $\mid$feasible & target & extended & target & extended & $\eta_{\rm ext}$ & $\eta_{\rm ext}$, physics-aware \\
\midrule
1.5B base & $105$ & $730$ & $0.07$ & $0.32$ & $0.28$ & $0.09$ & $0.21$ & $0.06$ & $0.62$ & $0.80$ \\
1.5B DPO, sum & $56$ & $499 \pm 29$ & $0.11 \pm 0.03$ & $0.41 \pm 0.03$ & $0.64 \pm 0.04$ & $0.37 \pm 0.02$ & $0.60 \pm 0.03$ & $0.26 \pm 0.02$ & $0.64 \pm 0.03$ & $0.92 \pm 0.01$ \\
1.5B DPO, gate & $59$ & $485 \pm 10$ & $0.06 \pm 0.01$ & $0.40 \pm 0.01$ & $0.81 \pm 0.01$ & $0.44 \pm 0.03$ & $0.76 \pm 0.01$ & $0.30 \pm 0.01$ & $0.56 \pm 0.01$ & $0.91 \pm 0.01$ \\
1.5B DPO, tuned sum $w{=}0.005$$^\dagger$ & $93$ & $692 \pm 86$ & $0.05 \pm 0.01$ & $0.29 \pm 0.06$ & $0.65 \pm 0.02$ & $0.27 \pm 0.04$ & $0.57 \pm 0.03$ & $0.15 \pm 0.02$ & $0.37 \pm 0.05$ & $0.85 \pm 0.02$ \\
1.5B DPO, contrast-matched sum $w{=}0.1$$^\dagger$ & $62$ & $507 \pm 30$ & $0.06 \pm 0.01$ & $0.40 \pm 0.03$ & $0.82 \pm 0.02$ & $0.46 \pm 0.03$ & $0.77 \pm 0.02$ & $0.29 \pm 0.02$ & $0.52 \pm 0.02$ & $0.90 \pm 0.03$ \\
1.5B DPO, log-prob$^\dagger$ & $54$ & $479 \pm 21$ & $0.10 \pm 0.02$ & $0.42 \pm 0.03$ & $0.72 \pm 0.04$ & $0.46 \pm 0.01$ & $0.68 \pm 0.03$ & $0.30 \pm 0.00$ & $0.59 \pm 0.01$ & $0.90 \pm 0.03$ \\
1.5B DPO, mass-only$^\dagger$ & $104$ & $719 \pm 20$ & $0.06 \pm 0.02$ & $0.28 \pm 0.03$ & $0.60 \pm 0.02$ & $0.23 \pm 0.01$ & $0.52 \pm 0.03$ & $0.12 \pm 0.01$ & $0.40 \pm 0.01$ & $0.87 \pm 0.04$ \\
1.5B DPO, three-conjunct gate$^\dagger$ & $40$ & $365 \pm 13$ & $0.04 \pm 0.01$ & $\mathbf{0.55 \pm 0.02}$ & $0.74 \pm 0.01$ & $\mathbf{0.50 \pm 0.02}$ & $0.69 \pm 0.01$ & $\mathbf{0.37 \pm 0.03}$ & $0.65 \pm 0.04$ & $0.89 \pm 0.02$ \\
7B base & $61$ & $516$ & $0.14$ & $0.29$ & $0.36$ & $0.13$ & $0.34$ & $0.10$ & $0.74$ & $0.91$ \\
7B DPO, sum & $51$ & $409 \pm 49$ & $0.01 \pm 0.01$ & $0.54 \pm 0.09$ & $0.43 \pm 0.03$ & $0.27 \pm 0.05$ & $0.41 \pm 0.03$ & $0.23 \pm 0.04$ & $0.83 \pm 0.03$ & $0.97 \pm 0.03$ \\
7B DPO, gate & $86$ & $604 \pm 6$ & $0.01 \pm 0.00$ & $0.31 \pm 0.01$ & $0.59 \pm 0.03$ & $0.25 \pm 0.02$ & $0.56 \pm 0.03$ & $0.20 \pm 0.01$ & $0.76 \pm 0.03$ & $0.93 \pm 0.03$ \\
1.5B GRPO, sum & $75$ & $738 \pm 27$ & $0.17 \pm 0.02$ & $0.23 \pm 0.01$ & $0.48 \pm 0.00$ & $0.16 \pm 0.01$ & $0.46 \pm 0.00$ & $0.12 \pm 0.01$ & $0.70 \pm 0.01$ & $0.94 \pm 0.02$ \\
1.5B GRPO, soft gate & $84$ & $647 \pm 6$ & $0.07 \pm 0.01$ & $0.26 \pm 0.00$ & $0.53 \pm 0.00$ & $0.18 \pm 0.02$ & $0.49 \pm 0.01$ & $0.13 \pm 0.01$ & $0.68 \pm 0.02$ & $0.91 \pm 0.01$ \\
\addlinespace
cVAE prior & $113$ & $620$ & $0.07$ & $0.33$ & $0.33$ & $0.13$ & $0.23$ & $0.07$ & $0.58$ & $0.79$ \\
decoder, $\lambda=0$ & $112$ & $619 \pm 40$ & $0.07 \pm 0.01$ & $0.29 \pm 0.03$ & $0.34 \pm 0.00$ & $0.11 \pm 0.01$ & $0.23 \pm 0.01$ & $0.07 \pm 0.02$ & $0.60 \pm 0.09$ & $0.81 \pm 0.04$ \\
decoder, $\lambda=1$ & $0$ & --- & $1.00 \pm 0.00$ & $0.00 \pm 0.00$ & $0.09 \pm 0.06$ & $0.00 \pm 0.00$ & $0.04 \pm 0.02$ & $0.00 \pm 0.00$ & --- & --- \\
decoder, $\lambda=1$, capped & $0$ & --- & $1.00 \pm 0.00$ & $0.01 \pm 0.02$ & $0.07 \pm 0.06$ & $0.00 \pm 0.00$ & $0.04 \pm 0.03$ & $0.00 \pm 0.00$ & --- & --- \\
latent search on $\lambda=0$ & $116$ & $1114 \pm 2$ & $0.26 \pm 0.00$ & $0.18 \pm 0.00$ & $0.94 \pm 0.00$ & $0.53 \pm 0.00$ & $0.82 \pm 0.02$ & $0.19 \pm 0.01$ & $0.26 \pm 0.03$ & $0.87 \pm 0.02$ \\
\bottomrule
\end{tabular}
\end{adjustbox}
\end{table}

\paragraph{The conjunct binds, and its lever is the cooling channels.} Among feasible
bodies, $29$--$64\%$ survive. The hot face is set by the coolant film coefficient, i.e.\
by the total channel flow area $A_{\rm cool}=n\pi r_{\rm eff}^2$: more or wider
channels slow the air and heat the wall. This is the metal that mass minimisation
removes. The 7B gated policy, whose mass tracks the budget with gain $0.62$
(Appendix~\ref{app:conditioning}), opens $86$\,mm$^2$ against the sum policy's
$51$\,mm$^2$ and runs its median hot face at $604^\circ$C against $409^\circ$C, so
$0.31$ of its feasible bodies survive against $0.54$. Latent search, the best generator
under the original target ($0.375$ per draw, supply $0.94$), opens $116$\,mm$^2$, runs
at $1114^\circ$C and survives at $0.18$. The two critic-trained decoders have no
cooling zone in $100\%$ of their designs --- the printability rule that penalises a
channel against its ligament is satisfied by deleting the channels --- and zero
survival.

\paragraph{What the training rules protect is what the score ranks.} Under the
extended target, on target $\wedge$ ready, the gate's composed advantage over the sum
shrinks at 1.5B ($0.296$ against $0.260$; supply $0.444$ against $0.371$) and inverts
on two of three seeds at 7B ($0.203$ against $0.231$; supply $0.249$ against $0.268$).
Survival among feasible bodies is higher than the base's on every trained DPO arm
($12/12$), but the ordering between rules is not stable across scales. We had
registered that the gate would keep its advantage on every seed at both scales; it
does not, and the reversal at 7B is the informative miss. The selection side is more
robust: the two-stage screen, which sees no physics, still captures $\eta_{\rm ext} =
0.51$--$0.86$ of the attainable extended lift, because feasibility and mass are
necessary conditions; the same rule given the readiness verdict as one more
conjunct reaches $0.87$--$0.97$. The feasibility head's own ordering against the
physics bin is not at chance ($A_{\rm phys}$ $0.44$--$0.69$; three-bin identity exact
to $2\times10^{-16}$), which is the wall--ligament coupling, not knowledge of heat
transfer.

\paragraph{Registered predictions and misses.} Eight predictions were written before
the audit (\texttt{context/docs/PLAN\_PHYSICS\_GATE\_2026-09-18.md}). Held: every
trained DPO arm survives better than its base; the identity is exact; $A_{\rm phys}$
on gated 1.5B is $0.554$ against a registered $0.55$; $\eta_{\rm ext}$ on gated 1.5B is
$0.557$ against $0.55$. Missed: every absolute survival rate (we predicted $0.45$--$0.65$
and measured $0.29$--$0.41$; the screen is harsher on AlSi10Mg than the CuCrZr numbers
we calibrated on); gate above sum on every seed (reversed at 7B); supply and answer
under the extended target for gated 1.5B ($0.444$ and $0.296$ against $0.55$ and
$0.35$); latent search survival ($0.18$ against $0.50$); and that no trained policy
sets its wall from the thrust --- the sum-ranked and GRPO policies do (partial
correlation $0.30$--$0.63$), through the bore radii that printability depends on,
while the gated 1.5B policy does not ($0.01$--$0.08$). All arms in the paragraphs
above pre-existed the audit. Code:
\texttt{scripts/audit\_physics\_gate.py}; per-arm values and the declared operating
point in \texttt{results/physics\_gate\_audit/audit.json}.

\paragraph{The policies that track mass hardest run hottest.} The four arms trained
afterwards ($\dagger$ in Table~\ref{tab:physics_gate}) were screened with predictions
registered first (\texttt{PLAN\_NEW\_ARMS\_CPU\_2026-09-19.md}). The two rules whose
pairs contrast mass and not feasibility open the widest channels ($93$ and
$104$\,mm$^2$ against the gate's $59$), run at $692$ and $719^\circ$C, and survive
at $0.28$--$0.29$ among feasible bodies, below the gate on every seed and at the
base's level; the log-probability form sits with the sum ($0.42$; registered within
$0.05$ of it). The registered claim that every trained arm survives better than the
base in absolute terms fails on two of twelve seeds (tuned sum and mass-only, at
$0.26$--$0.27$ against the base's $0.28$), because those policies never installed
the feasibility that survival rides on.

\paragraph{The gate composes conjuncts, at a price.} Putting the readiness verdict
into the gate --- pairs ranked by \texttt{mass\_term} $-100\,[\texttt{prob}<0.5]
-100\,[\lnot\text{ready}]$, the readiness label computed for the stored train pool
in seconds, the recipe otherwise unchanged --- and training three seeds
(\texttt{PLAN\_PHYSICS\_DPO\_2026-09-18.md}, registered before the job) raises survival
among feasible bodies from $0.40$ to $0.55$ ($0.52$--$0.57$, every seed above the
two-conjunct gate), the extended composed answer from $0.30$ to $0.37$, and the
physics-aware ceiling from $0.41$ to $0.46$, with mass tracking intact (gain
$1.04$--$1.08$) and printability at ceiling. Of eleven registered checks seven held
and four missed, all in the direction of a smaller effect: survival $0.55$ against a
registered $0.60$, extended per-draw $0.135$ against $0.20$, extended supply $0.50$
against $0.60$; the falsifier (survival below $0.45$) did not fire. The miss that
says most: the policy's channel area still does not depend on the thrust
($0.30\to0.31$ partial correlation), so it improved survival unconditionally ---
channels $59\to40$\,mm$^2$, median hot face $485\to365^\circ$C --- rather than by
learning the coolant requirement the thrust implies, the same pattern as the
unconditional shortcut in Appendix~\ref{app:conditioning}. The price is paid on the
original target: per-draw $0.311\to0.257$, supply $0.81\to0.74$, composed
$0.76\to0.69$. The feasibility head's mixture terms are unchanged ($q$ $0.946$,
$A_{\rm own}$ $0.958$). Artifacts: \texttt{reproduced/physics\_gate\_dpo/}; code
\texttt{scripts/summarize\_physics\_gate\_dpo.py}.

\section{Selection ambiguity and the audit card: proofs and validation}
\label{app:theory}

\paragraph{Proof of Proposition~\ref{prop:interface_separation}.}
Let the pool be $\{x_1,\ldots,x_n\}$ with distinct score pairs $(f_i,m_i)=(s_F(x_i),s_M(x_i))$
and let $\operatorname{Par}$ be the set of candidates not weakly dominated by any other.
\emph{Only Pareto candidates can win.} If $x_j$ dominates $x_i$ ($f_j\ge f_i$, $m_j\ge m_i$,
one strict), then for any strictly increasing $g$, $g(f_j)+m_j>g(f_i)+m_i$, so $x_i$ is never
the top-1 choice. \emph{Every Pareto candidate can win.} Fix $p\in\operatorname{Par}$. Every
$x_i$ with $f_i>f_p$ has $m_i<m_p$ (otherwise it would dominate $p$), and no other candidate has
$f_i=f_p$ with $m_i\ge m_p$. Let $\Delta=\min\{m_p-m_i:f_i\ge f_p,\,i\ne p\}>0$ (or $\Delta=1$ if
the set is empty), $\rho=\max_i f_i-\min_i f_i$, $\mu=\max_i m_i-\min_i m_i$, choose
$\varepsilon>0$ with $\varepsilon\rho<\Delta$, and set $g(t)=\varepsilon t$ for $t<f_p$ and
$g(t)=K+\varepsilon t$ for $t\ge f_p$ with $K>\mu+\varepsilon\rho$. Then $g$ is strictly
increasing; every candidate with $f_i<f_p$ scores below $p$ because $K>\mu+\varepsilon\rho$; and
every candidate with $f_i\ge f_p$, $i\neq p$, scores
$K+\varepsilon f_i+m_i<K+\varepsilon f_p+m_p$ because $\varepsilon(f_i-f_p)\le\varepsilon\rho<\Delta\le m_p-m_i$.
Hence $p$ is the unique top-1 choice. \emph{Linear rescalings.} For $g(t)=ct$, $c>0$, the
top-1 choice maximizes the linear functional $(c,1)\cdot(f,m)$, whose maximizers over the finite
set are exactly the points of the convex hull exposed by a direction in the open positive quadrant;
these are Pareto points on the upper hull. The identification statements follow because an
audit statistic invariant to monotone reparameterization is constant on the class
$\{g(s_F)\}$ while the decision ranges over $\operatorname{Par}$. $\square$

With ties (duplicated candidates are common when a policy repeats a program), apply the result to
distinct score pairs; tied candidates share scores and, in our tasks, oracle labels.

\paragraph{Measured ambiguity.}
Table~\ref{tab:pareto} reports, per requirement pool of parsed candidates, the mean Pareto size,
the share of pools whose Pareto set is a singleton, the share mixing on- and off-target
candidates, and the per-pool outer range of Answer@16 reachable by audit-equivalent composites
(all monotone $g$) and by positive rescalings. Pair-weight-trained Phi policies nearly eliminate
ambiguity because they emit one program (Section~\ref{sec:train}).

\begin{table}[h]
\centering
\caption{Selection ambiguity on ECG test pools (post-hoc). Ranges are per-pool outer bounds:
a single transform must serve all pools, so the attainable aggregate range can be narrower.}
\label{tab:pareto}
\scriptsize
\begin{tabular}{lcccccc}
\toprule
pools & $\overline{|\operatorname{Par}|}$ & singleton & mixed labels & Answer@16 range (monotone $g$) & range (linear $c$) \\
\midrule
Llama base & 2.28 & .28 & .21 & [.33, .54] & [.34, .54] \\
Gemma base & 1.75 & .41 & .00 & [.00, .00] & [.00, .00] \\
Phi base & 2.18 & .28 & .14 & [.46, .60] & [.46, .59] \\
Phi gate (3 seeds) & 2.43--2.53 & .19--.24 & .06--.09 & [.65--.68, .73--.74] & [.65--.68, .73--.74] \\
Phi inference (3 seeds) & 2.27--2.71 & .18--.28 & .06--.10 & [.65--.67, .73--.75] & [.65--.67, .73--.74] \\
Phi pair (3 seeds) & 1.01--1.10 & .91--.99 & .00--.01 & [.51--.54, .51--.55] & [.51--.54, .51--.55] \\
\bottomrule
\end{tabular}
\end{table}

\paragraph{Audit-card validation.}
All numbers use the frozen ECG base test pools, parsed candidates, and 12-subject cluster
resampling (seed fixed in the released script). \emph{Identified interval (step 1).} For the
deployability head, the per-pool intervals average to $[0.472,1]$ for Llama (292 pools with a
positive and a negative; outer 95\% interval $[0.396,1]$) and $[0.452,1]$ for Phi ($[0.389,1]$),
while the direct value is $1.000$; $37\%$ and $48\%$ of pools are individually erased. For the
utility head, $q$ is small ($0.013$--$0.018$), the interval is tight ($[0.894,0.900]$ and
$[0.872,0.877]$), and no labelling is needed. \emph{Label budget (step 2).} We label $m$
off-target candidates uniformly at random, compare each revealed own-failure with all positives in
its pool, and use the ratio estimator of the pair-weighted within-pool $A_{\rm own}$, with a
subject-cluster bootstrap over the labelled items (200 repetitions per $m$).

\begin{table}[h]
\centering
\caption{Label budget for a direct own-bin audit of the ECG utility head (post-hoc).
Target: pair-weighted within-pool $A_{\rm own}$ with all labels ($0.895$ Llama, $0.842$ Phi).}
\label{tab:label_budget}
\scriptsize
\begin{tabular}{lcccccc}
\toprule
$m$ labelled negatives & 25 & 50 & 100 & 200 & 400 & 800 \\
\midrule
Llama: mean estimate & .896 & .896 & .895 & .894 & .895 & .896 \\
Llama: CI width / coverage & .19 / .89 & .14 / .91 & .12 / .96 & .10 / .98 & .09 / 1.00 & .08 / 1.00 \\
Phi: mean estimate & .852 & .846 & .841 & .841 & .843 & .842 \\
Phi: CI width / coverage & .25 / .91 & .21 / .93 & .18 / .98 & .16 / .99 & .15 / 1.00 & .14 / 1.00 \\
\bottomrule
\end{tabular}
\end{table}

\emph{Controlled mixtures (step 3).} Pooling the bin-conditional score distributions, we draw
200 positives and 400 negatives at mixture weights $q_1,q_2$, compute the two aggregate AUCs, and
solve $A_{\rm own}=(q_2A_1-q_1A_2)/(q_2-q_1)$. Over 200 repetitions the bias is below $0.015$ in all
settings; RMSE is $0.010$--$0.021$ at $(q_1,q_2)=(0.2,0.8)$, $0.034$--$0.047$ at $(0.4,0.6)$ and
$0.075$--$0.090$ at $(0.45,0.55)$ for both heads and families. The utility head's inherited bin is
small (25 and 57 candidates), so its mixtures resample those candidates heavily.

\section{Preregistered embedded-ECG stress test}
\label{app:ecg}

\paragraph{Purpose.}
This study was designed after the original engine results as an external stress test, not as an exploratory search for another positive domain. The benchmark, critic, model revisions, pair builders, scalar grids, DPO recipe, primary endpoint, equivalence margin, and 30-run matrix were frozen before confirmatory DPO. The optional INCART OOD phase was disabled before the primary freeze and was not run. The freeze is an internal commitment: protocol, manifest and predictions are SHA-256 hashed in the project repository (freeze commit \texttt{f0b5895}) before any confirmatory DPO run, but were not deposited with an external registry, so their timing rests on the repository history rather than on a third-party timestamp.

\paragraph{Data and split.}
The primary signals are MIT-BIH Arrhythmia Database v1.0.0 (DOI 10.13026/C2F305), with baseline-wander, muscle-artifact and electrode-motion recordings from MIT-BIH Noise Stress Test Database v1.0.0. Records 201 and 202 are grouped as one subject. Subjects are split before window construction into 27 train, 8 validation and 12 test subjects, yielding 1,000/300/500 fixed 30-s requirements. Beat annotations are used only for eligibility and the biomedical oracle and do not enter the prompt features.

\paragraph{Pipeline DSL and target.}
Each generator emits a finite causal ECG pipeline with whitelisted high-pass and low-pass settings, filter family, notch setting, downsampling, derivative flag, integration window, adaptive threshold multiplier, refractory period and numeric format. The deployment label is
\[
F=\text{valid/stable}\wedge[\mathrm{cycles}\le C_r]\wedge[\mathrm{RAM}\le R_r]\wedge[\mathrm{latency}\le L_r].
\]
The biomedical label is
\[
M=[F1_{\rm QRS}\ge0.80]\wedge[\mathrm{MAE}_{R}\le65\text{ ms}],
\]
and the full target is $y=F\wedge M$. Benchmark difficulty was resolved from 131,072 Sobol configurations before generator experiments: $P(F)=0.5848$, $P(M\mid F)=0.2408$, with semantic counts 54,421 $F=0$, 58,192 $F=1,M=0$, and 18,459 $F=1,M=1$.

\paragraph{Fixed critic.}
One two-head MLP critic is trained independently of the generator families and then frozen. Its held-out test AUC is 1.000 for feasibility and 0.930 for utility. The same critic scores Llama, Gemma and Phi candidate pools.

\paragraph{Models and generation.}
The frozen generator revisions are Llama-3.2-3B-Instruct (revision \texttt{0cb88a4...}), Gemma-3-4B-IT (\texttt{093f9f3...}), and Phi-4-mini-instruct (\texttt{cfbefac...}). Each requirement has $K=16$ samples at temperature 0.8 and top-$p=0.95$. The first prompt implementation was rejected before the preregistration freeze because it produced no usable Gemma outputs and almost no schema-valid Llama outputs; no valid biomedical outcomes from that attempt were available for scientific tuning. A shared prompt-v2 was accepted by a parse-only train pilot under a criterion fixed before inspecting those outputs, then hashed and frozen. The Phi checkpoint required disabling incompatible remote model code before any model weights or completions were produced; all scientific settings remained unchanged. These engineering deviations are retained in the package.

\paragraph{Selection and pair construction.}
The scalar grid is $\{.001,.003,.01,.03,.1,.3,1,3,10\}$. Validation selection gives $w_{\rm sel}=.003,.001,.03$ for Llama, Gemma and Phi. A critic-only normalized maximin of chosen/rejected feasibility and utility contrasts on the frozen train pools (each contrast divided by its maximum over the grid; the gate-free criterion of Appendix~\ref{app:pair_audit}) gives $w_{\rm pair}=1,.1,.3$. Each rule constructs 4,000 deterministic top-4-versus-bottom-4 pairs. Llama additionally includes a unit-sum arm. Exact pair overlap with the hard gate is 92.3\%/25.6\% for Llama inference/pair weights, 99.5\%/71.1\% for Gemma, and 79.6\%/31.0\% for Phi. Llama pair-weight and unit-sum pairs are exactly identical.

\paragraph{DPO and evaluation.}
All runs use sigmoid DPO with $\beta=.1$, two epochs, learning rate $10^{-5}$, LoRA rank 16 ($\alpha=32$, dropout .05), batch 2 with accumulation 8, and maximum length 768. The matrix contains 12 Llama runs and 9 each for Gemma and Phi, three seeds per rule. All 30 runs completed and all 30 aligned test pools contain 500 requirements by 16 candidates. Evaluation uses one common hard-feasibility-then-utility selector, not the training rule. The primary endpoint is Answer@16; the preregistered gate-equivalence margin is $\pm.04$.

\paragraph{Per-seed endpoints.}
Table~\ref{tab:ecg_seeds} lists all 30 runs. No seed was removed or repeated. Typical
invalid outputs include numeric values in the notch field, quoted numeric notch
values, unquoted \texttt{Q15}, and refractory values outside the finite schema.

\begin{table}[h]
\centering
\caption{All ECG DPO runs: Answer@16 (parse rate) per training seed under the common
evaluator. Llama pair-weight and unit-sum were trained on identical pair sets.}
\label{tab:ecg_seeds}
\scriptsize
\begin{tabular}{llccc}
\toprule
family & rule & seed 0 & seed 1 & seed 2 \\
\midrule
Llama & inference & .512 (1.00) & .518 (1.00) & .470 (1.00) \\
Llama & gate & .512 (.47) & .520 (1.00) & .488 (1.00) \\
Llama & pair & .000 (.00) & .000 (.00) & .512 (1.00) \\
Llama & unit sum & .000 (.00) & .002 (.00) & .512 (1.00) \\
Gemma & inference & .008 (1.00) & .000 (.00) & .006 (.96) \\
Gemma & gate & .006 (.97) & .000 (.00) & .006 (1.00) \\
Gemma & pair & .000 (.00) & .124 (1.00) & .244 (1.00) \\
Phi & inference & .718 (.75) & .724 (.98) & .714 (.79) \\
Phi & gate & .722 (.98) & .698 (1.00) & .724 (.99) \\
Phi & pair & .512 (.92) & .512 (.82) & .544 (.89) \\
\bottomrule
\end{tabular}
\end{table}

\paragraph{Sensitivity to the number of test subjects (post-hoc).}
The frozen bootstrap resamples only 12 test subjects. A leave-one-subject-out
jackknife with $t_{11}$ intervals gives $[-.386,-.273]$, $[+.043,+.194]$ and
$[-.264,-.128]$ for pair minus inference (Llama, Gemma, Phi), $[-.390,-.283]$,
$[+.043,+.194]$ and $[-.253,-.131]$ for pair minus gate, and $[-.003,+.016]$,
$[-.002,+.001]$ and $[-.018,+.010]$ for gate minus inference. Dropping any single
subject moves no estimate by more than $0.024$. Per-seed contrasts are the informative
caveat: pair minus gate is $-.21,-.19,-.18$ for Phi, but $-.51,-.52,+.02$ for Llama
and $-.01,+.12,+.24$ for Gemma, so only the Phi sign is stable across training seeds.

\paragraph{Output diversity and a requirement-blind program (post-hoc).}
One pipeline (IIR4 filter, 0.5--25\,Hz band, no notch, no downsampling, derivative on,
140\,ms integration, threshold multiplier 0.35, 220\,ms refractory, Q15 arithmetic)
satisfies the full target on $51.2\%$ of test requirements regardless of the request.
It is already the most frequent Llama base output ($38.5\%$ of train samples). It
accounts for $78$--$100\%$ of parsed outputs in six of the eight Llama arms that parse,
and for $97.1$--$100\%$ of parsed outputs in all three Phi pair-weight seeds, versus
$0$--$8\%$ in Phi gate and inference arms, which retain 79--113 distinct pipelines per
500$\times$16 pool. Gemma never emits it. Endpoints near $0.512$ therefore indicate a
requirement-blind policy, not partial success.

\paragraph{Gate-referenced versus gate-free contrast criteria (post-hoc).}
Rebuilding the pairs of the frozen builder on the frozen train pools for every grid
weight (which reproduces the frozen gate pair sets exactly), the critic contrasts at
the gate coincide with those of the smallest weights: the gate's feasibility contrast
is matched within $1.2\%$ at $w\le0.01$ in every family, while its utility contrast
equals the sweep maximum within $0.2\%$. The gate-free maximin therefore moves away from the gate toward
feasibility. Selecting instead the grid weight nearest the gate in gate-normalized
critic-contrast space gives $w=0.01$ (Llama), $0.003$ (Gemma) and $0.003$ (Phi), with
exact pair overlap with the gate of $87.3\%$, $99.0\%$ and $94.7\%$. At $w_{\rm pair}$,
the oracle utility contrast falls from $.524$ to $.461$ (Llama) and from $.303$ to
$.259$ (Phi) relative to the gate, while the oracle feasibility contrast is unchanged
or slightly higher. No policy was trained with the gate-referenced weights; these
numbers locate the criteria, they are not training outcomes.

\paragraph{Interpretation.}
The study supports transfer of the audit and selection separations and the
pair-equivalence reading of the training interface. It rejects the gate-free contrast
criterion as a transferable recipe; in its one informative family (Phi) the failure
has the engine's feasibility-dominated signature and ends in a requirement-blind
program. It does not establish that gate-referenced contrast matching would succeed
outside the near-identical-pair regime.

\section{Audit, selection and reranker analyses: proofs and CPU-only results}
\label{app:v23}
All analyses in this appendix are post hoc, CPU-only, and use the frozen ECG pools, the
frozen critic and the frozen common evaluator; scripts are \texttt{v23\_*.py} and \texttt{v24\_*.py} in the
supplement. No policy was trained for this appendix.

\subsection{Proof of Proposition~\ref{prop:audit_cost}}
\emph{(i), upper bound.} Let $N\sim\mathrm{Bin}(m,1-q)$ be the number of revealed own
failures and $\mu=m(1-q)$. Conditional on $N=k$, the revealed $h$ values are i.i.d.\ in
$[0,1]$ with mean $A_{\rm own}$, so by Hoeffding's inequality
\citep{hoeffding1963probability} $\Pr[|\hat A-A_{\rm own}|>\varepsilon\mid N=k]\le
2e^{-2k\varepsilon^2}$. The multiplicative Chernoff bound gives $\Pr[N<\mu/2]\le
e^{-\mu/8}\le\delta/2$ when $\mu\ge8\log(2/\delta)$, and on $N\ge\mu/2$ the conditional
failure probability is at most $2e^{-\mu\varepsilon^2}\le\delta/2$ when
$\mu\ge\log(4/\delta)/\varepsilon^2$. A union bound completes the claim.
\emph{(i), lower bound.} Take two instances that share the positives, $q$, and the
inherited bin, but whose own bins have $A_{\rm own}=0$ and $A_{\rm own}=1$. On the event
$N=0$, which has probability $q^m$ under both, the labelled sample has the same law, so
any estimator has the same output distribution; since the targets differ by
$1>2\varepsilon$, it errs by more than $\varepsilon$ on at least one instance with
probability at least $q^m/2$. Requiring this to be at most $\delta$ gives
$m\ge\log(1/(2\delta))/\log(1/q)\ge q\log(1/(2\delta))/(1-q)$, using
$\log(1/q)\le(1-q)/q$.
\emph{(i), $\varepsilon$-dependent lower bound.} Take two instances that share the
positives, $q$ and the inherited bin, and whose own-bin negatives all score either below
every positive ($h=1$) or above every positive ($h=0$), with
$\Pr[h=1]=\tfrac12+\tfrac32\varepsilon$ in one instance and
$\tfrac12-\tfrac32\varepsilon$ in the other. For $\varepsilon\le\tfrac14$ these are
valid Bernoulli laws and the two values of $A_{\rm own}$ differ by $3\varepsilon>2\varepsilon$.
A labelled negative reveals its bin, which has the same law under both, and, if it is an
own failure, one Bernoulli draw. Hence the KL divergence between the laws of an
$m$-sample is
$m(1-q)\,\mathrm{KL}(\mathrm{Ber}(\tfrac12+\tfrac32\varepsilon)\,\|\,
\mathrm{Ber}(\tfrac12-\tfrac32\varepsilon))$. A direct one-variable bound gives
$\mathrm{KL}\le24\varepsilon^2$ on $0<\varepsilon\le\tfrac14$ (equivalently, with
$x=3\varepsilon\le3/4$, $x\log((1+x)/(1-x))\le(8/3)x^2$). Because the targets are
more than $2\varepsilon$ apart, an estimate cannot be within $\varepsilon$ of both.
Bretagnolle--Huber \citep{tsybakov2009introduction} therefore implies that the larger
of the two probabilities of error greater than $\varepsilon$ is at least
$\tfrac14\exp(-24m(1-q)\varepsilon^2)$. Requiring it to be at most
$\delta<\tfrac14$ gives
$m(1-q)\ge\log(1/(4\delta))/(24\varepsilon^2)$. Together with the upper bound, the
passive budget is $\Theta(\log(1/\delta)/(\varepsilon^2(1-q)))$ up to universal
constants.
\emph{(ii).} From $A_i=(1-q_i)A_{\rm own}+q_iA_{\rm cross}$ for $i=1,2$,
$A_{\rm own}=(q_2A_1-q_1A_2)/(q_2-q_1)$; substituting estimates with errors
$e_i$, $|e_i|\le\varepsilon_i$, changes the right-hand side by
$(q_2e_1-q_1e_2)/(q_2-q_1)$. \qed

\subsection{Proof of Proposition~\ref{prop:ambiguity_K}}
With continuous coordinates there are no ties almost surely. Order candidates by
decreasing $s_F$. Candidate $i$ in this order is dominated iff some earlier candidate has
larger $s_M$, so it is Pareto iff its $s_M$ is the maximum of the first $i$ values. By
independence of $s_F$ and $s_M$, the $s_M$ sequence in this order is i.i.d.\ continuous,
and its record indicators are independent Bernoulli$(1/i)$ (R\'enyi's record theorem;
\citealt{bentley1978maxima}). Summing gives $\mathbb E|\operatorname{Par}|=H_K$;
$|\operatorname{Par}|=1$ iff the first candidate holds the overall maximum, probability
$1/K$. Under comonotone heads the top-$s_F$ candidate also has top $s_M$ and dominates
all others. A Monte-Carlo check with $20\,000$ pools per $K\in\{2,4,8,16,64\}$ matches
$H_K$ and $1/K$ to within $0.006$ and $0.004$. \qed

\subsection{Proof of Proposition~\ref{prop:pair_stability}(iv)}
Let $\delta_t=\|\theta^P_t-\theta^Q_t\|$ and $\Delta=\sup_\theta\|\nabla L_P(\theta)-\nabla
L_Q(\theta)\|\le2G\,d_{\rm TV}(P,Q)$ by part (iii). Then
$\delta_{t+1}\le\|(\theta^P_t-\eta\nabla L_P(\theta^P_t))-(\theta^Q_t-\eta\nabla L_P(\theta^Q_t))\|
+\eta\Delta$. For convex $L$-smooth $L_P$ and $\eta\le2/L$ the gradient step is
non-expansive \citep{hardt2016train}, so $\delta_{t+1}\le\delta_t+\eta\Delta$ and
$\delta_T\le\eta T\Delta$. With smoothness alone the step is $(1+\eta L)$-Lipschitz,
giving $\delta_T\le\Delta\sum_{t<T}\eta(1+\eta L)^t=\Delta((1+\eta L)^T-1)/L$. Any
evaluation functional that is $\Lambda$-Lipschitz in $\theta$ then differs by at most
$\Lambda\delta_T$. The DPO objective is convex for log-linear policies; for LoRA
fine-tuning neither assumption is guaranteed, which is why the main text treats overlap
as a predictor, not a guarantee, for endpoints. \qed

\subsection{Selection ambiguity versus pool size}
For each requirement we draw $K$ of the $16$ stored candidates ($20$ draws for $K<16$),
keep parsed candidates, collapse identical score pairs, and compute the Pareto set. $H_n$
and $1/n$ are the independence values of Proposition~\ref{prop:ambiguity_K} at the
observed number $n$ of distinct points; the permutation null shuffles $s_M$ within the
pool. The range is the per-pool outer range of Answer@$K$ over audit-equivalent
composites, averaged over requirements.

\begin{table}[h]
\centering
\caption{Selection ambiguity on ECG test pools as a function of $K$ (post hoc).}
\label{tab:v23_ambiguity}
\small
\setlength{\tabcolsep}{3.5pt}
\begin{tabular}{lcccccccc}
\toprule
pools & $K$ & distinct $n$ & $|\operatorname{Par}|$ & $H_n$ & permuted & $\Pr[|\operatorname{Par}|{=}1]$ & $1/n$ & Answer@$K$ range \\
\midrule
Llama base & 2 & 1.79 & 1.25 & 1.40 & 1.39 & 0.75 & 0.60 & 0.058 \\
 & 4 & 3.09 & 1.57 & 1.83 & 1.81 & 0.52 & 0.35 & 0.107 \\
 & 8 & 5.10 & 1.88 & 2.27 & 2.25 & 0.39 & 0.21 & 0.153 \\
 & 16 & 8.32 & 2.28 & 2.73 & 2.68 & 0.28 & 0.13 & 0.212 \\
Phi base & 2 & 1.73 & 1.28 & 1.36 & 1.36 & 0.72 & 0.64 & 0.049 \\
 & 4 & 2.79 & 1.60 & 1.74 & 1.73 & 0.49 & 0.40 & 0.084 \\
 & 8 & 4.17 & 1.89 & 2.09 & 2.09 & 0.36 & 0.26 & 0.112 \\
 & 16 & 6.31 & 2.18 & 2.46 & 2.44 & 0.28 & 0.17 & 0.140 \\
Phi gate (3 seeds) & 2 & 1.88 & 1.31 & 1.44 & 1.44 & 0.70 & 0.56 & 0.030 \\
 & 4 & 3.44 & 1.65 & 1.94 & 1.93 & 0.48 & 0.31 & 0.051 \\
 & 8 & 5.93 & 2.05 & 2.42 & 2.41 & 0.32 & 0.18 & 0.064 \\
 & 16 & 9.53 & 2.49 & 2.87 & 2.82 & 0.22 & 0.11 & 0.069 \\
\bottomrule
\end{tabular}
\end{table}

The permutation null reproduces $H_n$ within $0.05$ in every row, confirming the
proposition on real score marginals; the observed Pareto sets are smaller because
$s_F$ and $s_M$ are positively dependent within pools, and they still grow roughly
logarithmically in $K$.

\subsection{Audit decomposition on post-DPO pools}
Table~\ref{tab:v23_postdpo} lists every trained ECG arm with at least one evaluable pool
(a pool is evaluable when its parsed candidates contain both oracle classes). ``Own
support'' is the share of evaluable pools whose own-failure bin ($F=0$) is non-empty;
``erased'' is the share with $q\ge\max(A,1-A)$, which includes every pool with $q=1$.

\begin{table}[h]
\centering
\caption{Deployability-head audit on base and post-DPO ECG test pools (post hoc).}
\label{tab:v23_postdpo}
\small
\begin{tabular}{lccccc}
\toprule
pool & evaluable & own support & erased & macro $A$ & macro $q$ \\
\midrule
llama base & 292 & 0.87 & 0.37 & 0.751 & 0.611 \\
llama hard gate s1 & 85 & 0.00 & 1.00 & 0.671 & 1.000 \\
llama hard gate s2 & 110 & 0.00 & 1.00 & 0.604 & 1.000 \\
llama inference weight s1 & 43 & 0.00 & 1.00 & 0.605 & 1.000 \\
llama inference weight s2 & 123 & 0.00 & 1.00 & 0.628 & 1.000 \\
gemma pair weight s1 & 19 & 0.00 & 1.00 & 0.211 & 1.000 \\
gemma pair weight s2 & 2 & 0.00 & 1.00 & 0.000 & 1.000 \\
phi base & 283 & 0.64 & 0.48 & 0.791 & 0.666 \\
phi hard gate s0 & 168 & 0.36 & 0.64 & 0.836 & 0.661 \\
phi hard gate s1 & 197 & 0.28 & 0.72 & 0.811 & 0.719 \\
phi hard gate s2 & 184 & 0.33 & 0.67 & 0.817 & 0.681 \\
phi inference weight s0 & 158 & 0.36 & 0.64 & 0.800 & 0.649 \\
phi inference weight s1 & 163 & 0.37 & 0.64 & 0.754 & 0.648 \\
phi inference weight s2 & 169 & 0.34 & 0.66 & 0.730 & 0.673 \\
phi pair weight s1 & 2 & 0.50 & 0.50 & 0.500 & 0.500 \\
phi pair weight s2 & 29 & 0.21 & 0.79 & 0.759 & 0.798 \\
\bottomrule
\end{tabular}
\end{table}

\emph{Designed re-injection.} For the $12$ Llama and Phi gate/inference runs we draw
$200$ positives from the trained pool and $400$ negatives mixing the trained pool's
inherited bin with base-policy own failures of the same family at $q_1=0.2$ and
$q_2=0.8$, and invert the $2\times2$ system ($200$ repetitions). Against the direct
pooled value $A_{\rm own}=1.000$ the estimates average 0.998--1.003 with RMSE
0.016--0.020. Because the deployability head is near-perfect, this checks the
procedure rather than detecting degradation.

\subsection{Label budget as a function of $q$}
Utility head, pooled base test candidates; positives $y=1$, own bin $M=0$, inherited bin
$F=0\wedge M=1$. Each of $2\,000$ repetitions labels $m$ negatives from a mixture with
weight $q$ and averages $h$ over the revealed own failures.

\begin{table}[h]
\centering
\caption{Passive audit cost (post hoc). The scaled error $\sqrt{m(1-q)}\cdot$RMSE is
constant, as Proposition~\ref{prop:audit_cost} predicts, until own failures stop
appearing.}
\label{tab:v23_cost}
\small
\begin{tabular}{lcccccc}
\toprule
 & & \multicolumn{3}{c}{$m=100$} & \multicolumn{2}{c}{$m=400$} \\
\cmidrule(lr){3-5}\cmidrule(lr){6-7}
family & $q$ & RMSE & scaled & no own failure & RMSE & scaled \\
\midrule
Llama & 0.00 & 0.021 & 0.20 & 0.000 & 0.010 & 0.21 \\
 & 0.50 & 0.029 & 0.20 & 0.000 & 0.014 & 0.20 \\
 & 0.80 & 0.046 & 0.21 & 0.000 & 0.023 & 0.21 \\
 & 0.90 & 0.067 & 0.21 & 0.000 & 0.033 & 0.21 \\
 & 0.95 & 0.102 & 0.23 & 0.006 & 0.047 & 0.21 \\
 & 0.99 & 0.174 & 0.17 & 0.351 & 0.120 & 0.24 \\
Phi & 0.00 & 0.023 & 0.23 & 0.000 & 0.011 & 0.23 \\
 & 0.50 & 0.033 & 0.23 & 0.000 & 0.016 & 0.23 \\
 & 0.80 & 0.053 & 0.24 & 0.000 & 0.025 & 0.22 \\
 & 0.90 & 0.076 & 0.24 & 0.000 & 0.036 & 0.23 \\
 & 0.95 & 0.106 & 0.24 & 0.005 & 0.053 & 0.24 \\
 & 0.99 & 0.205 & 0.20 & 0.359 & 0.136 & 0.27 \\
\bottomrule
\end{tabular}
\end{table}

\subsection{Seed-and-subject intervals for the ECG contrasts}
The frozen intervals of Table~\ref{tab:ecg_train} condition on the three training
seeds. Table~\ref{tab:v23_seed} resamples, in each of $5\,000$ replicates, the three
seeds of each rule independently and the $12$ test subjects, and recomputes the mean
Answer@16 difference under the frozen common evaluator.

\begin{table}[h]
\centering
\caption{Two-level (seed and subject) bootstrap intervals (post hoc).}
\label{tab:v23_seed}
\small
\begin{tabular}{llcc}
\toprule
family & contrast & estimate & 95\% interval \\
\midrule
Llama & gate-inference & $+0.007$ & $[-0.026,+0.045]$ \\
Llama & pair-gate & $-0.336$ & $[-0.557,-0.005]$ \\
Llama & pair-inference & $-0.329$ & $[-0.551,+0.000]$ \\
Gemma & gate-inference & $-0.001$ & $[-0.008,+0.006]$ \\
Gemma & pair-gate & $+0.119$ & $[-0.002,+0.266]$ \\
Gemma & pair-inference & $+0.118$ & $[-0.001,+0.266]$ \\
Phi & gate-inference & $-0.004$ & $[-0.026,+0.015]$ \\
Phi & pair-gate & $-0.192$ & $[-0.254,-0.141]$ \\
Phi & pair-inference & $-0.196$ & $[-0.263,-0.143]$ \\
\bottomrule
\end{tabular}
\end{table}

\subsection{Prediction-registered convex DPO reranker}
\label{app:v23_reranker}
\emph{Design, frozen before running.} The design and four predictions were written to
\texttt{v23b\_PREDICTIONS.md} (SHA-256 \texttt{d6863c3a8fb1ff40\ldots}) before any training run; this
study is post hoc relative to the ECG preregistration. For Llama and Phi (Gemma excluded a
priori for lack of supply) a reranker $\pi_\theta(i\mid r)\propto\exp(\theta^\top\phi(r,i))$ over the
parsed candidates of each frozen base pool uses DSL one-hot features and their products
with standardized requirement features ($588$ dimensions); the reference is uniform
($\theta_0=0$), so the DPO log-ratio of a same-pool pair is $\theta^\top(\phi_c-\phi_r)$ and the
objective is convex. Sigmoid DPO with $\beta=0.1$, minibatch $16$, two epochs over the rule's
$4\,000$ pairs, ten shuffle seeds; the learning rate was chosen once per family from
$\{1,3,10,30\}$ by the gate-trained reranker's validation success (seed 0) and frozen
($30$ in both, the grid edge). Pairs are rebuilt with the frozen builder and reproduce the
frozen gate, $w_{\rm sel}$ and $w_{\rm pair}$ files exactly. The endpoint is expected per-draw
target success on the test pools; intervals resample seeds and the $12$ test subjects.
Uniform reranking scores $0.360$ (Llama) and $0.443$ (Phi).

\emph{Predictions and outcomes.} P1 (gate-referenced weight within $\pm0.04$ of the gate in
both families): held. P2 (Phi's gate-free $w_{\rm pair}$ not within $\pm0.04$): held, with a small
effect ($-0.016$). P3 ($|$gap$|$ negatively rank-correlated with overlap): held, Spearman
$-0.60$ (Llama) and $-0.93$ (Phi). P4 (full-batch GD distances respect
Proposition~\ref{prop:pair_stability}(iv)): held; distances are monotone in $1-$overlap
(Spearman $1.0$) and at most $1.1\%$ of the bound, which is therefore valid but loose.
\emph{Post hoc, after seeing results:} the test top-1 choice is the requirement-blind
program for $25\%$ (gate), $27\%$ (gate-referenced) and $67\%$ ($w_{\rm pair}$) of Phi
requirements, and $48\%$, $52\%$ and $86\%$ for Llama.

\begin{table}[h]
\centering
\caption{Convex DPO reranker on frozen ECG candidate pools, ten seeds per rule
(prediction-registered, CPU-only). ``dist'' is the full-batch GD parameter distance to the
gate after $200$ steps.}
\label{tab:v23_reranker}
\scriptsize
\setlength{\tabcolsep}{3pt}
\begin{tabular}{lllccccc}
\toprule
family & rule & role & overlap & per-draw & $-$ gate & 95\% (seed+subject) & dist \\
\midrule
Llama & gate &  & 1.000 & 0.512 & --- & --- & 0.00 \\
 & $w=0.001$ &  & 0.938 & 0.509 & $-0.003$ & $[-0.009,+0.001]$ & 0.08 \\
 & $w=0.003$ & $w_{\rm sel}$ & 0.923 & 0.510 & $-0.002$ & $[-0.007,+0.003]$ & 0.08 \\
 & $w=0.01$ & gate-referenced & 0.873 & 0.516 & $+0.004$ & $[-0.005,+0.015]$ & 0.13 \\
 & $w=0.03$ &  & 0.762 & 0.523 & $+0.011$ & $[-0.010,+0.036]$ & 0.36 \\
 & $w=0.1$ &  & 0.559 & 0.511 & $-0.001$ & $[-0.021,+0.025]$ & 0.92 \\
 & $w=0.3$ &  & 0.366 & 0.510 & $-0.002$ & $[-0.028,+0.026]$ & 1.81 \\
 & $w=1$ & $w_{\rm pair}$ (gate-free); unit sum & 0.256 & 0.506 & $-0.006$ & $[-0.034,+0.025]$ & 2.52 \\
 & $w=3$ &  & 0.183 & 0.482 & $-0.030$ & $[-0.056,-0.002]$ & 2.75 \\
 & $w=10$ &  & 0.160 & 0.460 & $-0.052$ & $[-0.081,-0.024]$ & 2.85 \\
Phi & gate &  & 1.000 & 0.537 & --- & --- & 0.00 \\
 & $w=0.001$ &  & 0.962 & 0.535 & $-0.002$ & $[-0.006,+0.001]$ & 0.12 \\
 & $w=0.003$ & gate-referenced & 0.947 & 0.535 & $-0.002$ & $[-0.007,+0.002]$ & 0.14 \\
 & $w=0.01$ &  & 0.903 & 0.537 & $+0.001$ & $[-0.004,+0.006]$ & 0.24 \\
 & $w=0.03$ & $w_{\rm sel}$ & 0.796 & 0.531 & $-0.006$ & $[-0.021,+0.007]$ & 0.69 \\
 & $w=0.1$ &  & 0.558 & 0.533 & $-0.004$ & $[-0.025,+0.017]$ & 1.79 \\
 & $w=0.3$ & $w_{\rm pair}$ (gate-free) & 0.310 & 0.521 & $-0.016$ & $[-0.044,+0.015]$ & 3.04 \\
 & $w=1$ & unit sum & 0.165 & 0.517 & $-0.020$ & $[-0.048,+0.013]$ & 4.06 \\
 & $w=3$ &  & 0.127 & 0.515 & $-0.022$ & $[-0.052,+0.013]$ & 4.21 \\
 & $w=10$ &  & 0.116 & 0.512 & $-0.025$ & $[-0.054,+0.011]$ & 4.24 \\
\bottomrule
\end{tabular}
\end{table}

\section{Richer audit information, realistic selection, and rule-level pair geometry}
\label{app:v24}
All results in this appendix are post hoc and CPU-only on frozen artifacts
(\texttt{v24\_trimming\_bounds.py}, \texttt{v24\_realistic\_selector.py},
\texttt{v24\_rule\_geometry.py}); no policy was trained.

\subsection{Identified set when every negative's score is observed}
Fix a pool with $n$ negatives, of which $n_{\rm own}=(1-q)n$ fail $F$, and let
$h_{(1)}\le\cdots\le h_{(n)}$ be the sorted values $h(x)$ over negatives (the share of the
pool's positives ranked above $x$, ties $\tfrac12$). Suppose the auditor observes all
scores and all outcome labels $y$, and $q$, but not which negatives fail $F$. Then
\[
A_{\rm own}\in\Big[\tfrac1{n_{\rm own}}\textstyle\sum_{i\le n_{\rm own}}h_{(i)},\;
\tfrac1{n_{\rm own}}\sum_{i>n-n_{\rm own}}h_{(i)}\Big],
\]
and the interval is sharp. \emph{Proof.} $A_{\rm own}$ is the mean of $h$ over the own
bin, a subset of size $n_{\rm own}$ of the negatives; every such subset is compatible
with the observations because bin membership is unobserved and $q$ fixes only its size,
and the extreme means are attained by the lowest and highest $n_{\rm own}$ values. In
the population version this is the trimming bound of \citet{horowitz1995identification}.
It equals $[0,1]$ iff at least $n_{\rm own}$ negatives have $h=0$ and at least
$n_{\rm own}$ have $h=1$. If $q$ is known only to an interval, the identified set is the
union over the admissible $n_{\rm own}$. \qed

Table~\ref{tab:v24_trim} compares three information structures on the ECG deployability
head: (I0) $(A,q)$ only, Eq.~\ref{eq:identified_interval}; (I1) plus every negative's
score; (I2) as I1 but with $q$ replaced by a $95\%$ Clopper--Pearson interval from $50$
labelled negatives. Every interval contains the directly measured $A_{\rm own}$ where the
own bin has support (I2: $\ge98\%$ of pools). Observing the scores barely helps, because
with $K=16$ most negatives are ranked below all positives or above them.

\begin{table}[h]
\centering
\caption{Mean interval width / share of evaluable pools with an uninformative interval
(width $\ge0.999$) under three information structures. Post-DPO Llama pools have no own-bin support; pair-weight
arms have almost no evaluable pools (Phi: $0$, $2$, $29$).}
\label{tab:v24_trim}
\scriptsize
\begin{tabular}{lccccc}
\toprule
pool & evaluable & own support & I0: $(A,q)$ & I1: $+$scores & I2: $+$scores, $\hat q$ \\
\midrule
Llama base & 292 & 0.87 & 0.528 / 0.37 & 0.486 / 0.34 & 0.496 / 0.36 \\
Phi base & 283 & 0.64 & 0.548 / 0.48 & 0.538 / 0.48 & 0.538 / 0.48 \\
Phi hard gate s0 & 168 & 0.36 & 0.662 / 0.64 & 0.662 / 0.64 & 0.668 / 0.65 \\
Phi hard gate s1 & 197 & 0.28 & 0.720 / 0.72 & 0.720 / 0.72 & 0.721 / 0.72 \\
Phi hard gate s2 & 184 & 0.33 & 0.682 / 0.67 & 0.682 / 0.67 & 0.684 / 0.67 \\
Phi inference weight s0 & 158 & 0.36 & 0.648 / 0.64 & 0.648 / 0.64 & 0.649 / 0.64 \\
Phi inference weight s1 & 163 & 0.37 & 0.646 / 0.64 & 0.646 / 0.64 & 0.649 / 0.64 \\
Phi inference weight s2 & 169 & 0.34 & 0.677 / 0.66 & 0.677 / 0.66 & 0.679 / 0.67 \\
\bottomrule
\end{tabular}
\end{table}

\subsection{Selection with a validation-tuned common weight}
Proposition~\ref{prop:interface_separation} bounds what an audit leaves open. A deployed
selector uses one common weight $w$ in $w\,s_F+s_M$ chosen on validation Answer@16.
Table~\ref{tab:v24_sel} reports, on the base test pools, the test value at the frozen
$w_{\rm sel}$ and the best and worst values over the frozen grid; all grid points share
every component AUC.

\begin{table}[h]
\centering
\caption{Validation-tuned common weight versus the test range over audit-equivalent
common weights (Answer@16, ECG base test pools).}
\label{tab:v24_sel}
\scriptsize
\begin{tabular}{lcccccc}
\toprule
family & $w_{\rm sel}$ & test at $w_{\rm sel}$ & best common & worst common & regret & hard gate \\
\midrule
Llama & 0.003 & 0.480 & 0.500 & 0.392 & 0.020 & 0.486 \\
Phi & 0.03 & 0.542 & 0.546 & 0.512 & 0.004 & 0.548 \\
\bottomrule
\end{tabular}
\end{table}

\subsection{Rule-level pair geometry for all trained scalar rules}
Table~\ref{tab:v24_geo} lists every trained scalar rule with its exact pair overlap with
the gate, the bottleneck oracle contrast $b(r)=\min\{\Delta F_r/\Delta F_{\rm gate},
\Delta M_r/\Delta M_{\rm gate}\}$ of Appendix~\ref{app:weight_geometry}, and its endpoint.
A family is \emph{informative} for training iff its hard-gate arm exceeds its reference
by more than the $\pm0.04$ margin (ECG: the single requirement-blind pipeline, $0.512$;
engine: the base policy's supply@16, $0.285$, an upper bound on its Answer@16); the rule does not use the gaps. It was fixed after the endpoints were known and is not part of the ECG preregistration. Without it, all three rules with overlap $\ge0.79$ match the gate within $0.007$ (Llama and Gemma inference weights against degenerate gates), and the next-highest overlap, Gemma's pair weight at $0.71$, does not ($+0.119$, with every arm far below the requirement-blind program). Among the seven
informative rules, the one with overlap $\ge0.79$ matches the gate; five of six below
miss by $\ge0.08$. All four with $b<0.5$ miss, but among low-overlap rules with $b\ge0.5$
one matches (engine $w=0.1$) and one misses (Phi pair, whose policies collapse onto the
requirement-blind program). Over the seven informative rules the Spearman correlation of
the signed gap is $0.54$ with $b$ and $-0.21$ with overlap. Neither statistic is
sufficient in the low-overlap regime; overlap near one is sufficient by
Proposition~\ref{prop:pair_stability}.

\begin{table}[h]
\centering
\caption{All eleven trained scalar rules (Answer@16 under the common evaluator; three
seeds). Engine values from Table~\ref{tab:train} and Appendix~\ref{app:weight_geometry};
ECG from Table~\ref{tab:ecg_train} and the frozen pair sets.}
\label{tab:v24_geo}
\scriptsize
\begin{tabular}{lcccccc}
\toprule
rule & overlap & $b(r)$ & endpoint & gate & gap & informative \\
\midrule
engine unit sum & 0.099 & 0.174 & 0.596 & 0.758 & $-0.162$ & yes \\
engine validation-tuned $w{=}$.005 & 0.210 & 0.262 & 0.571 & 0.758 & $-0.187$ & yes \\
engine contrast-matched $w{=}$.1 & 0.150 & 0.828 & 0.770 & 0.758 & $+0.012$ & yes \\
engine log-probability & 0.140 & 0.287 & 0.675 & 0.758 & $-0.083$ & yes \\
engine mass-only & 0.262 & 0.071 & 0.519 & 0.758 & $-0.239$ & yes \\
Llama inference $w{=}$0.003 & 0.923 & 0.969 & 0.500 & 0.507 & $-0.007$ & no \\
Llama pair $w{=}$1 & 0.256 & 0.879 & 0.171 & 0.507 & $-0.336$ & no \\
Gemma inference $w{=}$0.001 & 0.995 & 0.923 & 0.005 & 0.004 & $+0.001$ & no \\
Gemma pair $w{=}$0.1 & 0.711 & 1.000 & 0.123 & 0.004 & $+0.119$ & no \\
Phi inference $w{=}$0.03 & 0.796 & 0.980 & 0.719 & 0.715 & $+0.004$ & yes \\
Phi pair $w{=}$0.3 & 0.310 & 0.855 & 0.523 & 0.715 & $-0.192$ & yes \\
\bottomrule
\end{tabular}
\end{table}

\section{AUC-change decomposition and engine selection controls}
\label{app:moved}

\paragraph{Mixture-shift versus ordering-drift decomposition.}
Across two policies, the AUC change can also be decomposed into mixture shift and
conditional-ordering drift:
\begin{align}
A_t-A_s={}&(q_t-q_s)(A_{\rm cross}^s-A_{\rm own}^s) \label{eq:change}\\
&+(1-q_t)(A_{\rm own}^t-A_{\rm own}^s)
+q_t(A_{\rm cross}^t-A_{\rm cross}^s).\nonumber
\end{align}
The first term changes only composition; the last two change conditional ordering.
Both mechanisms can therefore coexist. The identity extends to any disjoint
failure-bin partition and requires no independence assumptions.

\paragraph{Rules that reduce the saturated component's influence recover selection.}
Holding the candidate pools and critic fixed, the raw unit sum reaches
$\eta=0.10\pm0.02$ on matched 1.5B pools, rank averaging $0.58\pm0.01$, a
validation-tuned feasibility weight of $0.005$ reaches $0.88\pm0.02$, the soft
log-probability rule reaches $0.89\pm0.00$, and threshold-then-mass reaches
$0.92\pm0.01$. At 7B, the raw sum remains poor ($0.05\pm0.02$) while
threshold-then-mass reaches $0.94\pm0.04$ (full sweep in
Appendix~\ref{app:selection_audit}).

The residual feasibility ordering is not useless (permuting it within the passing
set lowers raw-sum answer on gated 7B pools from $0.237$ to $0.209$, post-hoc); the
problem is excessive influence, not absence of cross-signal.

\paragraph{Controls.}
A holistic on-target classifier (held-out AUC $0.974$) avoids the component-AUC
cliff but captures only $\eta=0.49$--$0.60$, versus $0.90$--$0.94$ for the gated
two-head critic. This is not evidence that decomposition is universally superior:
the mass regressor contains ordinal distance-to-budget information. GSM8K shows the
opposite saturation geometry. Its format-head within-pool SD falls from $8.6$ to
$0.02$, so the tested composites effectively reduce to the correctness head and
agree within $0.01$ (Appendix~\ref{app:gsm}).

\section{Audit inversion and the refresh study: proofs, protocol and full results}
\label{app:refresh}

\subsection{Proof of Proposition~\ref{prop:inversion}}
\emph{(i)} Subtract Eq.~\ref{eq:mix} for $f$ and $g$ in the same pool (same $q$); the sign of
$(1-q)(-\Delta_{\rm own})+q\Delta_{\rm cross}$ is positive iff $q>\Delta_{\rm own}/(\Delta_{\rm own}+\Delta_{\rm cross})$.
\emph{(ii)} In a pool with candidate law $P$, $\operatorname{AUC}(s)=\Pr[s(X)>s(X')]+\tfrac12\Pr[s(X)=s(X')]$
for $X\sim P(\cdot\mid y=1)$, $X'\sim P(\cdot\mid y=0)$ independent. The likelihood ratio
$dP(\cdot\mid y=1)/dP(\cdot\mid y=0)$ is a strictly increasing function of $\eta_y$, and any
scorer that is a strictly increasing transform of it maximizes this probability
\citep[Proposition~1]{clemencon2008ranking}. \emph{(iii)} Since $y=F\wedge M$, $y=1$ implies
$F=\phi(x)=1$, so $\eta_y(x)=0$ whenever $\phi(x)=0$, while $\eta_y(x)>0$ almost surely under
$P(\cdot\mid y=1)$; hence every positive outranks every own failure and $a^{\rm own}(\eta_y)=1$.
On $\{\phi=1\}$ we have $y=M$ and $\eta_y=\Pr[M=1\mid x]$, which is the optimal ranker for the
cross subproblem by (ii) applied to that subpopulation, so $a^{\rm cross}(\eta_y)\ge\tfrac12$ (the
value of any constant scorer), with equality iff $\eta_y$ is a.s.\ constant there. The oracle
$\phi$ has $a^{\rm own}=1$ and $a^{\rm cross}=\tfrac12$; substituting in Eq.~\ref{eq:mix} gives
$A(\eta_y)-A(\phi)=q(a^{\rm cross}(\eta_y)-\tfrac12)$. $\square$

\emph{Remark (support).} Refresh data drawn from a policy whose negatives contain no own
failures carry no information about how a head orders own failures: any two heads that agree on
the refresh support fit it equally well, however they rank off-support $F$-failures. This is
Proposition~\ref{prop:identified_set} with $\pi_{\rm own}=0$, now read as a statement about
learning rather than auditing.

\subsection{Protocol}
Data are the frozen ECG pools of Appendix~\ref{app:ecg}: base-policy train pools (Llama, Phi),
base-policy test pools (Llama, Phi, and Gemma as an unseen generator), and the test pools of the
six post-DPO policies per family trained with gate or inference-weight pairs. Features are the
frozen critic's $55$ inputs (requirement descriptors and one-hot DSL fields). Head classes: an MLP
with the critic's architecture ($256$-$128$-$64$, Adam, batch $256$, up to $40$ epochs) and
histogram gradient boosting ($200$ trees, no early stopping). \emph{Label-mixture heads}
(P1) are trained from scratch on the base train pool with each candidate duplicated and labelled
$y$ with probability $\lambda$ and $F$ otherwise, for three seeds. \emph{Refresh heads} continue
training the specialist $S$ ($\lambda=0$): the MLP by $20$ epochs of \texttt{partial\_fit}, the
GBM by $100$ additional trees on the refresh data. Refresh data are all parsed candidates of the
post-DPO pools whose requirement belongs to a random half of the test subjects (seeds
$100$--$109$); $R_{\rm mix}$ adds an equal number of base-train candidates drawn uniformly. All
metrics use the complementary subject half. The monitor is the within-pool full-target AUC of
Section~\ref{sec:setup}, averaged over the six post-DPO policies; canonical $F$-AUC is computed over
all parsed candidates of the base pools; the gate selects the highest frozen utility score among
candidates with head probability $\ge0.5$ and otherwise the head's top candidate.

\subsection{Where the interface ingredients already appear}
\label{app:workflow_relevance}
The exact controller we analyze is a composition of interfaces, not a claim about one named product.
Table~\ref{tab:workflow_relevance} records the closest modern ingredients. Rubric systems expose
criterion-level judgments but ultimately aggregate them for policy optimization; policy-aware rubric
work explicitly measures saturation; and separate reward-model work refreshes models from fresh
policy-interaction data. At LLM scale, POW3R reports $17$--$26\%$ saturated and $20$--$33\%$
dead rubric criteria across its fresh-policy diagnostics, so criterion informativeness is visibly
policy-dependent even before any maintenance decision is made \citep{tyagi2026policyaware}. The
unsafe loop arises when an aggregate outcome statistic is reused as the maintenance signal for one
component head.

\begin{table}[h]
\centering
\caption{Contemporary systems instantiate the ingredients of the analyzed interface. We do not
claim that any cited system implements the exact audit-trigger-refresh controller end to end.}
\label{tab:workflow_relevance}
\scriptsize
\setlength{\tabcolsep}{4pt}
\begin{tabularx}{\linewidth}{lXXX}
\toprule
work & component-level structure & policy-dependent ingredient & interface relevance \\
\midrule
Rubrics as Rewards \citep{gunjal2026rubrics} & prompt-specific rubric criteria judged separately & on-policy rubric RL & criterion signals are aggregated into a scalar reward \\
Reward hacking in rubric RL \citep{mahmoud2026rubric} & criterion-level training-verifier judgments & exploitation grows during RL & aggregate proxy reward can diverge from broader reference evaluation \\
POW3R \citep{tyagi2026policyaware} & independently judged rubric criteria & criteria become saturated, dead or mixed under the current policy & shows criterion informativeness changes with policy state \\
UGDA / iterative RLHF \citep{sun2025ugda,bai2022training} & learned reward/preference model & fresh policy interactions enter later updates & supplies the on-policy refinement step used in our controller abstraction \\
\bottomrule
\end{tabularx}
\end{table}

\subsection{Registered predictions and outcomes}
The predictions file was hashed (SHA-256 \texttt{ac50e1a1\ldots}) before the confirmatory run and is
included in the supplement with the exploratory pilot, which used refitting from scratch rather
than continued training and one split seed. P1 (pure component head ranked last on the post-DPO
audit, all four settings): \textbf{held}. P2 (Phi monitor repaired by $\ge0.05$ in both classes):
\textbf{failed} ($+0.042$, $-0.012$). P3 (outcome refresh lowers off-policy $F$-AUC by $\ge0.05$,
all four): \textbf{held}. P4 (re-injection within $0.01$ of $S$, all four): \textbf{held}. P5 (P3 and
P4 on the unseen generator): \textbf{held}. P6 (more gate violations under $R_y$ than $S$, MLP, both
families): \textbf{held}.

\begin{table}[h]
\centering
\caption{P1: on-policy audit AUC of label-mixture heads $(1-\lambda)F+\lambda y$ (mean of three
seeds). The pure component head ($\lambda=0$) is lowest in every row labelled ``audit''.}
\label{tab:mix}
\scriptsize
\begin{tabular}{llccccc}
\toprule
setting & quantity & $\lambda=0$ & $0.25$ & $0.5$ & $0.75$ & $1$ \\
\midrule
Llama GBM & post-DPO audit & 0.625 & 0.686 & 0.674 & 0.721 & 0.728 \\
 & base audit & 0.855 & 0.897 & 0.902 & 0.905 & 0.917 \\
 & $F$-AUC (base) & 1.000 & 0.999 & 0.999 & 0.997 & 0.964 \\
\midrule
Llama MLP & post-DPO audit & 0.561 & 0.654 & 0.644 & 0.673 & 0.692 \\
 & base audit & 0.747 & 0.829 & 0.858 & 0.874 & 0.891 \\
 & $F$-AUC (base) & 1.000 & 0.999 & 0.998 & 0.997 & 0.936 \\
\midrule
Phi GBM & post-DPO audit & 0.477 & 0.696 & 0.714 & 0.696 & 0.773 \\
 & base audit & 0.809 & 0.859 & 0.874 & 0.875 & 0.881 \\
 & $F$-AUC (base) & 1.000 & 1.000 & 1.000 & 0.999 & 0.965 \\
\midrule
Phi MLP & post-DPO audit & 0.651 & 0.723 & 0.712 & 0.687 & 0.690 \\
 & base audit & 0.710 & 0.819 & 0.815 & 0.840 & 0.851 \\
 & $F$-AUC (base) & 1.000 & 0.998 & 0.998 & 0.996 & 0.941 \\
\bottomrule
\end{tabular}
\end{table}

\begin{table}[h]
\centering
\caption{Refresh minus deployed specialist, paired by subject split (mean; in parentheses the
number of the $10$ splits with a positive difference).}
\label{tab:refresh_diff}
\scriptsize
\setlength{\tabcolsep}{3pt}
\begin{tabular}{llccccc}
\toprule
head & setting & monitor AUC & $F$-AUC base & $F$-AUC unseen & gate $F$-violation & Answer@16 \\
\midrule
$R_y$ & Llama GBM & $-0.050$ (3) & $-0.433$ (0) & $-0.507$ (0) & $+0.055$ (10) & $-0.048$ (0) \\
$R_y$ & Llama MLP & $+0.105$ (9) & $-0.117$ (0) & $-0.291$ (0) & $+0.011$ (10) & $-0.020$ (2) \\
$R_y$ & Phi GBM & $-0.012$ (4) & $-0.481$ (0) & $-0.481$ (0) & $+0.040$ (10) & $-0.056$ (0) \\
$R_y$ & Phi MLP & $+0.042$ (8) & $-0.091$ (0) & $-0.364$ (0) & $+0.018$ (10) & $-0.030$ (0) \\
\midrule
$R_F$ & Llama GBM & $-0.120$ (1) & $-0.235$ (0) & $-0.224$ (0) & $+0.027$ (10) & $-0.017$ (0) \\
$R_F$ & Llama MLP & $+0.008$ (8) & $-0.000$ (0) & $-0.002$ (0) & $+0.001$ (2) & $-0.004$ (0) \\
$R_F$ & Phi GBM & $+0.111$ (10) & $-0.000$ (0) & $+0.000$ (6) & $+0.000$ (0) & $+0.000$ (0) \\
$R_F$ & Phi MLP & $+0.033$ (10) & $-0.006$ (0) & $-0.000$ (8) & $-0.001$ (0) & $-0.000$ (0) \\
\midrule
$R_{\rm mix}$ & Llama GBM & $+0.003$ (6) & $+0.000$ (0) & $-0.000$ (0) & $+0.000$ (0) & $+0.000$ (0) \\
$R_{\rm mix}$ & Llama MLP & $+0.006$ (9) & $+0.000$ (5) & $+0.000$ (9) & $+0.000$ (0) & $+0.000$ (0) \\
$R_{\rm mix}$ & Phi GBM & $+0.255$ (10) & $-0.000$ (0) & $-0.001$ (3) & $+0.000$ (0) & $+0.000$ (0) \\
$R_{\rm mix}$ & Phi MLP & $+0.030$ (10) & $+0.000$ (9) & $+0.000$ (7) & $-0.001$ (0) & $+0.000$ (0) \\
\bottomrule
\end{tabular}
\end{table}

\subsection{Closed-loop trigger analysis}
\label{app:controller}
This post-hoc analysis composes already frozen quantities; it trains no new head. For each of the
four family$\times$head-class settings, let the controller compare the pure specialist's base-policy
audit to its post-DPO audit (Table~\ref{tab:mix}) and choose $R_y$ iff the drop is at least $\tau$.
If it does not trigger, it leaves the deployed specialist $S$ unchanged. Table~\ref{tab:controller}
reports the resulting mean over the four primary settings. A reference-distribution component controller using the same threshold on off-policy $F$-AUC triggers no refresh because $S$ retains $F$-AUC $1.000$ in all four settings. The threshold sweep is illustrative rather than preregistered; its purpose is to test
whether the measurement error can actually induce the harmful action posited by the control-loop
framing.

\begin{table}[h]
\centering
\caption{Counterfactual controller on frozen ECG heads. ``Triggered'' is the number of the four
primary family$\times$class settings in which aggregate audit AUC falls by at least $\tau$ from
base to post-DPO. Triggered settings deploy $R_y$; others keep $S$. No-refresh/component-specific
control has $F$-AUC $1.000$, violation $.148$, Answer@16 $.530$.}
\label{tab:controller}
\small
\begin{tabular}{ccccc}
\toprule
$\tau$ & triggered & mean $F$-AUC & gate $F$-violation & Answer@16 \\
\midrule
.05 & 4/4 & .719 & .179 & .491 \\
.10 & 3/4 & .742 & .175 & .499 \\
.15 & 3/4 & .742 & .175 & .499 \\
.20 & 2/4 & .772 & .172 & .504 \\
.25 & 1/4 & .880 & .158 & .516 \\
\midrule
no refresh / reference-$F$ trigger & 0/4 & 1.000 & .148 & .530 \\
\bottomrule
\end{tabular}
\end{table}

\begin{figure}[h]
\centering
\includegraphics[width=.90\linewidth]{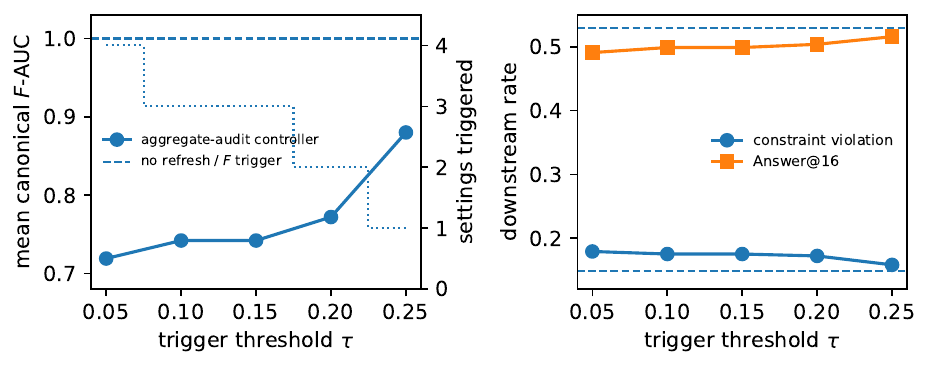}
\caption{\textbf{Frozen closed-loop threshold sweep.} Left: aggregate-audit triggers deploy the
already-trained outcome-refresh head and reduce canonical $F$-AUC; the dotted secondary axis is the
number of the four primary settings that trigger. Right: the same controller raises constraint
violations and lowers Answer@16. Dashed horizontal lines are the no-refresh/reference-$F$ controller.}
\label{fig:controller_sweep}
\end{figure}

\subsection{Natural multi-label semantic stress test: GoEmotions}
\label{app:goemotions}
We use the official GoEmotions development and test splits \citep{demszky2020goemotions}, with
$5{,}426$ and $5{,}427$ Reddit comments. Each comment can carry multiple human-annotated emotion
labels. The component $F$ is fixed throughout as absence of \emph{annoyance}; the development/test
sets contain $303/320$ annoyance-labelled own failures. We train Laplace-smoothed multinomial
Naive Bayes heads on the development split over a fixed bag-of-words vocabulary (terms occurring
at least twice; at most $10{,}000$). This small transparent model is intentional: the experiment tests
the audit geometry without conflating it with a large verifier's capacity.

\paragraph{Near-boundary control fixed in v31.}
With $M$ = absence of \emph{disapproval}, the test set has $q=.434$. The $F$ specialist has
$(A_{\rm own},A_{\rm cross},A_F)=(.683,.569,.679)$; the deliberately entangled
$s_{\rm leak}=z(s_F)+z(s_M)$ has $(.646,.631,.640)$ and predicted $q^*=.374$. At the natural
mixture the audit difference is only $.006$ and its 500-resample bootstrap interval
$[-.009,.019]$ spans zero. This is therefore a useful near-boundary sanity check rather than a
claim of naturally observed inversion.

\paragraph{Natural high-$q$ positive control.}
After the near-boundary result, an explicitly post-hoc exploratory scan kept the same $F$ specialist
and varied the second annotated criterion; the scan identified absence of \emph{gratitude} as a
high-support natural positive control. We then reran this pair under the original frozen
vocabulary/protocol. The test set contains $4{,}757$ target positives, $320$ own failures and $350$
cross failures, hence $q=.522$. The unchanged $F$ specialist has
$A_{\rm own}=.665$, $A_{\rm cross}=.238$, canonical $F$-AUC $.679$, and aggregate outcome AUC
$.442$. With $s_{\lambda}=z(s_F)+\lambda z(s_M)$ at $\lambda=1$, canonical $F$-AUC falls to
$.570$ and $A_{\rm own}$ to $.589$, while $A_{\rm cross}$ rises to $.765$. Proposition~\ref{prop:inversion}
predicts $q^*=.126$ (bootstrap $[.092,.153]$), so the naturally observed mixture lies well past the
crossing: the aggregate audit prefers the worse component verifier, $.681$ versus $.442$. A
500-resample bootstrap interval for the audit difference is $[.210,.270]$, and the canonical
$F$-AUC gap favoring the specialist is $[.089,.128]$. Because the second criterion was selected
post hoc, these intervals quantify stability of this fixed official test split; they are not a
confirmatory multiplicity-adjusted discovery claim.

\begin{table}[h]
\centering
\caption{Leakage-strength sweep for the natural annoyance/gratitude control. The specialist is the
$\lambda=0$ reference: canonical $F$-AUC $.679$ and aggregate audit $.442$ at observed $q=.522$.
Increasing cross-component leakage makes the component classifier worse while making the aggregate
audit better.}
\label{tab:goe_lambda}
\scriptsize
\begin{tabular}{ccccc}
\toprule
$\lambda$ & canonical $F$-AUC & $A_{\rm own}$ & $A_{\rm cross}$ & $q^*$ / observed audit \\
\midrule
.10 & .668 & .656 & .299 & .120 / .470 \\
.50 & .619 & .623 & .564 & .113 / .592 \\
1.00 & .570 & .589 & .765 & .126 / .681 \\
2.00 & .520 & .547 & .879 & .155 / .721 \\
4.00 & .483 & .512 & .919 & .183 / .725 \\
\bottomrule
\end{tabular}
\end{table}

\begin{figure}[h]
\centering
\begin{minipage}{.49\linewidth}\centering
\includegraphics[width=\linewidth]{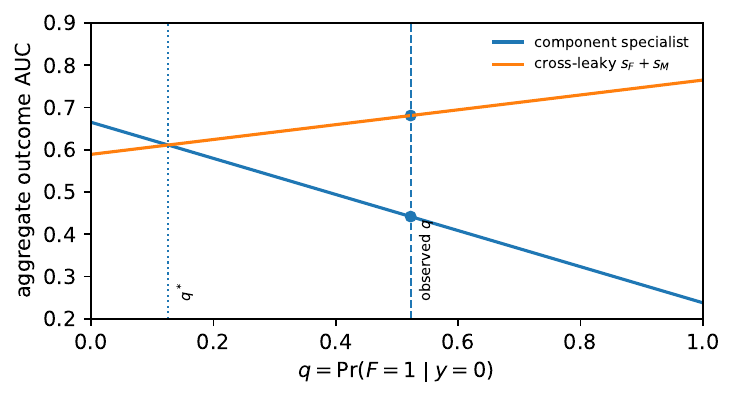}
\end{minipage}\hfill
\begin{minipage}{.49\linewidth}\centering
\includegraphics[width=\linewidth]{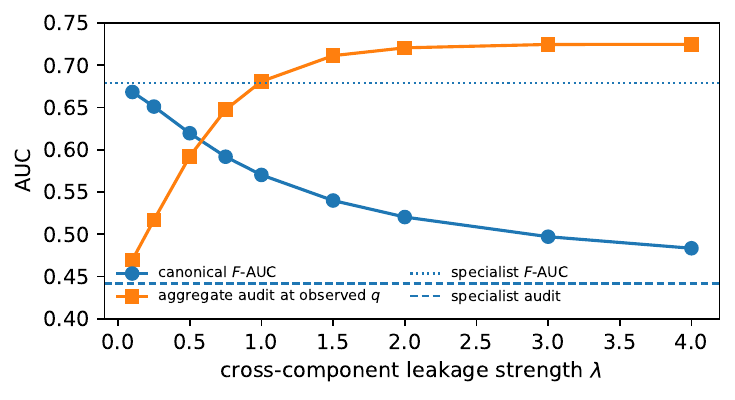}
\end{minipage}
\caption{\textbf{Natural semantic inversion and a severity sweep.} Left: for the post-hoc
annoyance/gratitude positive control, the observed $q=.522$ lies well past the analytic crossing.
Right: increasing cross-component leakage monotonically degrades canonical $F$-AUC while the
outcome audit improves over most of the sweep.}
\label{fig:goe_v32}
\end{figure}

\paragraph{Support depletion and targeted replay.}
We next retrain the correct $F$ classifier while retaining all $F=1$ development examples and a
nested number of annoyance-labelled failures. Canonical $F$-AUC rises from $.537$ with two own
failures to $.545$ (8), $.594$ (64), $.638$ (128), $.671$ (256), and $.679$ with all $303$ failures.
The original fixed 8+80 own-failure replay arm reaches $.615$. To distinguish targeted failure
support from generic extra data, we additionally repeat 100 deterministic draws. Adding 80 random
development examples to the 8-failure arm adds only $4.42\pm1.96$ own failures on average and
raises canonical $F$-AUC only to $.550\pm.004$. Adding 80 targeted own failures instead reaches
$.606\pm.008$. At the development-set prevalence, a uniform sample would require about $1{,}433$
examples in expectation to contain 80 own failures. The rescue is therefore about reconstructing
the missing semantic support, not merely increasing the update set.

\begin{figure}[h]
\centering
\includegraphics[width=.64\linewidth]{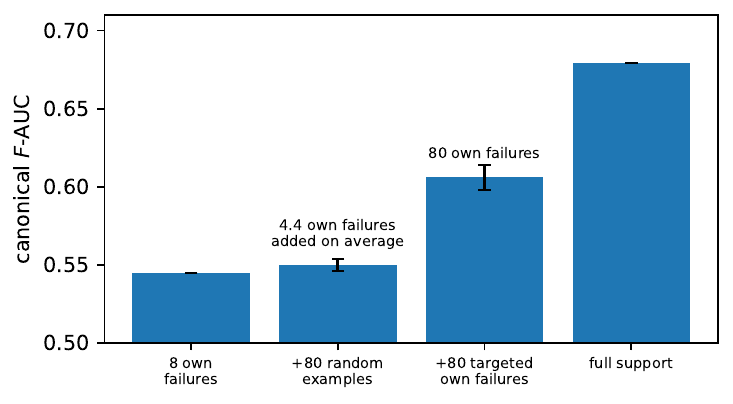}
\caption{\textbf{Failure-targeted replay is sample-efficient.} Starting from only eight own failures,
80 indiscriminate examples add few relevant failures and barely move canonical $F$-AUC; 80
explicit own failures recover substantially more component discrimination. Error bars show SD over
100 deterministic replay draws; the full-support bar uses all 303 development failures.}
\label{fig:goe_replay}
\end{figure}

The exact script, source blob SHAs and frozen v31/v32 JSON summaries are included in the supplement.
All GoEmotions analyses are post hoc; the natural-positive-control scan and replay comparison were
added after v31 and are labelled as such.

\subsection{Additional constructed semantic-language control: AG News}
\label{app:agnews}
This post-hoc stress test uses the 2,000-row AG News sample distributed in the OpenAI Cookbook
(blob SHA \texttt{f63d8be6...}) and the standard four topic labels of \citet{zhang2015character}.
We concatenate title and description and define $F=\mathbf 1\{\text{topic}\neq\text{World}\}$,
$M=\mathbf 1\{\text{topic}\neq\text{Sports}\}$, so $y=F\wedge M$ iff the topic is Business or
Sci/Tech. Thus World examples are the own-failure bin and Sports examples are the cross-failure
bin. The construction is deliberately simple: it supplies semantic natural-language components
with ground-truth labels while retaining the exact two-bin geometry used by the theory.

For each of 24 deterministic 65/35 hash splits, we train Laplace-smoothed multinomial Naive Bayes
heads over a fixed bag-of-words vocabulary (tokens occurring in at least two training documents;
maximum 7,000 terms). The $F$ specialist obtains $A_{\rm own}=0.951\pm0.010$,
$A_{\rm cross}=0.618\pm0.034$, and canonical $F$-AUC $0.949\pm0.011$. Holding the head fixed and changing only the negative-bin mixture
therefore moves the observed audit from $0.951$ at $q=0$ to $0.784$ at $q=.5$, $0.651$ at $q=.9$
and $0.625$ at $q=.98$. To instantiate Proposition~\ref{prop:inversion} as a controlled positive
control, we standardize the separately trained $F$ and $M$ scores on the training split and form
$s_{\rm leak}=0.5z(s_F)+z(s_M)$. This score is worse on both own-failure separability
($A_{\rm own}=0.886\pm0.016$) and canonical $F$-AUC ($0.607\pm0.025$), but nearly perfect on
the cross bin ($A_{\rm cross}=0.993\pm0.002$); the measured crossing is
$q^*=0.148\pm0.030$, and the aggregate audit prefers $s_{\rm leak}$ at $q=.9$ in every split.

For a support-only refresh test we keep the correct $F$ labels and vary only the number of World
failures retained from the training split, leaving all non-World examples available. The canonical
$F$-AUC follows the same dose-response as $A_{\rm own}$: full support gives
$0.949\pm0.011$, the 8-failure arm $0.898\pm0.016$, and 8+80 replay $0.938\pm0.012$. The retained
failure sets are nested within each split. Mean held-out $A_{\rm own}$ rises smoothly with own-failure
support: $0.888\pm0.015$ with 2 failures, $0.895\pm0.014$ with 4,
$0.904\pm0.014$ with 8, $0.917\pm0.014$ with 16, $0.927\pm0.011$ with 32,
$0.937\pm0.011$ with 64, $0.944\pm0.010$ with 128, and $0.949\pm0.010$ with 256,
versus $0.951\pm0.010$ under full support. The 8-failure arm is below full support in every split;
re-injecting 80 baseline World failures raises it to $0.941\pm0.011$ and improves on the collapsed
refresh in every split. These 24 deterministic splits are robustness checks, not independent
replications. Figure~\ref{fig:agnews_semantic} shows the audit inversion and the support-dose curve.
The exact script, source-file SHA and frozen summary JSON are included in the supplement; no result
in this subsection was preregistered.

\begin{figure}[h]
\centering
\includegraphics[width=0.92\linewidth]{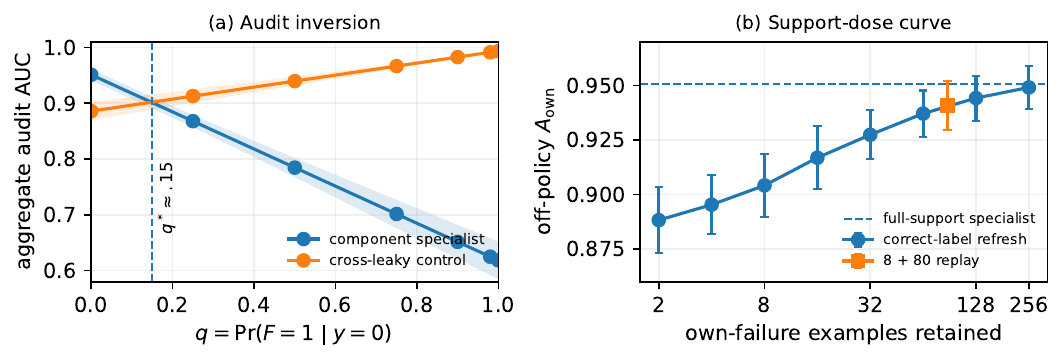}
\caption{\textbf{Controlled natural-language stress test.} Left: the same fixed component verifier
looks progressively worse as its own-failure bin vanishes, while a verifier contaminated with the
other component overtakes it near the predicted $q^*$. Right: with correct component labels fixed,
held-out $A_{\rm own}$ degrades smoothly as own-failure support is removed; the square marks
the 8-failure refresh plus 80 re-injected baseline failures. Bands/error bars show one standard
deviation across 24 deterministic splits.}
\label{fig:agnews_semantic}
\end{figure}

\section{Selection interface: audit equivalence does not imply equal influence}
\label{sec:select}

The audit in Section~\ref{sec:law} uses only the ordering \emph{within} one component
score, so it is invariant to every strictly increasing reparameterization
$s_F\mapsto g(s_F)$. A multi-component selector also consumes the scale of $s_F$
relative to the other heads. The next result characterizes exactly how much of the
selector's decision the audit leaves open.

\begin{proposition}[Selection ambiguity is the Pareto set]
\label{prop:interface_separation}
Fix a pool with component scores $(s_F,s_M)$ and distinct score pairs (duplicates
are merged; Appendix~\ref{app:theory}). As $g$ ranges over strictly
increasing maps---all of which leave every ROC/AUC of $s_F$ unchanged---the set of
top-1 choices of $g(s_F)+s_M$ is exactly the Pareto set $\operatorname{Par}$ of the
pool in $(s_F,s_M)$. Over positive rescalings $g(t)=ct$ it is the subset of
$\operatorname{Par}$ on the upper convex hull. Hence an audit statistic invariant to
monotone reparameterization determines the composite decision iff
$|\operatorname{Par}|=1$, and determines whether that decision is on target iff all
Pareto candidates share one oracle label.
\end{proposition}

The proof is constructive (Appendix~\ref{app:theory}): a dominated candidate loses
under every increasing $g$, and each Pareto candidate wins under a $g$ that jumps at its
own $s_F$. This
is the classical correspondence between Pareto optimality and monotone scalarization
\citep{miettinen1999nonlinear}, read here as a statement about what an audit leaves open.
Its size is predictable.

\begin{proposition}[Ambiguity grows with pool size]
\label{prop:ambiguity_K}
If the $K$ candidates' score pairs are i.i.d.\ with continuous, independent $s_F$ and
$s_M$, then $|\operatorname{Par}|$ is a sum of independent Bernoulli$(1/i)$ variables,
$i=1,\ldots,K$. Hence $\mathbb E|\operatorname{Par}|=H_K=\ln K+0.577+o(1)$ and
$\Pr[|\operatorname{Par}|=1]=1/K$. If the heads are comonotone, $|\operatorname{Par}|=1$.
\end{proposition}

Sorting candidates by decreasing $s_F$, a candidate is Pareto iff its $s_M$ is a running
maximum, which happens with probability $1/i$ independently across $i$
\citep{bentley1978maxima}. Under independent heads the audit therefore determines the
best-of-$K$ decision with probability only $1/K$, which vanishes as test-time compute
grows; in our pools positively dependent heads slow this growth without stopping it (Figure~\ref{fig:v23}a).

\begin{figure}[t]
\centering
\includegraphics[width=0.86\linewidth]{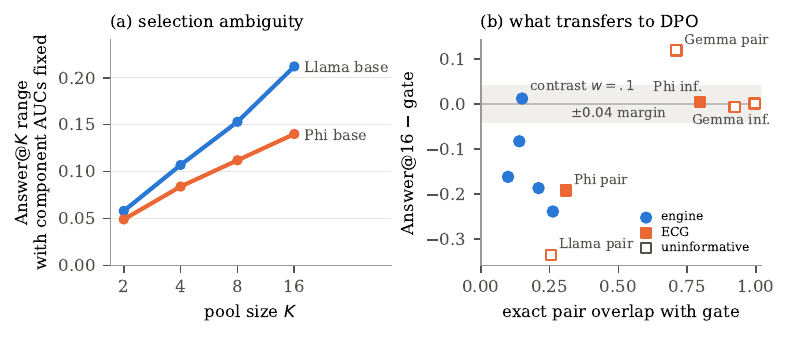}
\caption{\textbf{(a)} With every component AUC fixed, the range of Answer@$K$ reachable
by audit-equivalent composites on the ECG base test pools grows with $K$
(Proposition~\ref{prop:ambiguity_K}; $20$ subsamples per pool). \textbf{(b)} Endpoint
minus hard gate against exact pair overlap for all eleven trained scalar rules (engine
1.5B and ECG; three seeds each; common evaluator). Filled: families whose gate clears
the requirement-blind floor (engine, Phi); open: uninformative (Llama near the blind
pipeline, Gemma without supply); this post-hoc split uses only each family's gate arm,
never a rule's gap (Appendix~\ref{app:v24}). All three rules with overlap $\ge0.79$ match
the gate within $0.007$ (two in uninformative families); five of six informative
low-overlap rules miss by $\ge0.08$, the exception being the contrast-matched $w=0.1$.}
\label{fig:v23}
\end{figure}

\paragraph{How large is the ambiguity?}
On the ECG base test pools at $K=16$ the Pareto set has $2.3$ (Llama) and $2.2$ (Phi)
candidates, below the independence values $2.7$ and $2.5$ ($H_n$ over the $n$ distinct
score pairs) that a within-pool permutation of $s_M$ restores. It mixes on- and off-target candidates in $21\%$ and $14\%$
of pools, so with every component AUC fixed Answer@16 ranges over $[0.33,0.54]$ (Llama)
and $[0.46,0.60]$ (Phi). These are outer bounds because one $g$ must serve all pools, yet
a common rescaling $c\,s_F+s_M$ already moves Llama's Answer@16 from $0.500$ ($c=0.1$)
to $0.392$ ($c=10$) with supply unchanged. A single common weight chosen on
validation still lands within $0.02$ (Llama) and $0.004$ (Phi) of the best common weight
on test (Appendix~\ref{app:v24}): the ambiguity is resolved by validating the consumer,
not by auditing the component.

\paragraph{The engine shows the saturated extreme.}
On matched 1.5B DPO pools, rescaling the feasibility logit by $c\in\{0.01,\ldots,100\}$
leaves its AUC unchanged but moves raw-sum best-of-16 from $0.544$ to $0.160$
(Appendix~\ref{app:selection_audit}); the unit sum follows the saturated logit on
$82\%$ of pools ($\eta=0.10$). Rules that reduce its influence---a validation-tuned weight, a soft
log-probability rule, threshold-then-mass---reach $\eta=0.88$--$0.92$ ($0.94$ at 7B;
Appendix~\ref{app:moved}).

\paragraph{Takeaway.}
Component AUC fixes the order within a score; the selector's decision is fixed only up
to the Pareto set, whose size and label mix at the deployed $K$ should be reported
alongside any component audit. A weight tuned for this interface is a \emph{selection weight};
Section~\ref{sec:train} shows why that does not make it a training weight.

\section{Training interface: a good selector need not build good DPO data}
\label{sec:train}

Once a score constructs preference data, DPO sees only the chosen/rejected pairs,
so selection quality and training-pair quality are different objectives even for one
scalar ranking rule.

\begin{proposition}[What reaches DPO]
\label{prop:ranking_geometry}\label{prop:pair_stability}
Fix pools, policy, reference model, $\beta$, and a deterministic rank-based pair
builder. \textbf{(i)} The empirical DPO objective and gradient depend on the scoring
rule only through the induced pair multiset; for $r_w=ws_F+s_M$ this multiset is
piecewise constant in $w$, changing only at crossing weights
$w_{ij}=-(s_{M,i}-s_{M,j})/(s_{F,i}-s_{F,j})$. \textbf{(ii)} Neither component-order
nor top-$k$ equivalence implies pair equivalence when the builder uses ranks outside
the top $k$. \textbf{(iii)} For uniform pair multisets $P,Q$ of equal size,
$d_{\rm TV}(P,Q)=1-\mathrm{overlap}(P,Q)$; if the preference log-ratio satisfies
$|d_\theta|\le R$ and per-pair gradients are bounded by $G$, then
$|L_P-L_Q|\le\beta R\,(1-\mathrm{overlap})$ and
$\|\nabla L_P-\nabla L_Q\|\le2G\,(1-\mathrm{overlap})$. \textbf{(iv)} If both objectives
are convex and $L$-smooth, full-batch gradient descent with step $\eta\le2/L$ from a
common initialization satisfies $\|\theta^P_T-\theta^Q_T\|\le2\eta GT\,(1-\mathrm{overlap})$;
with $L$-smoothness alone the factor $\eta T$ becomes $((1+\eta L)^T-1)/L$.
\end{proposition}

Together with Proposition~\ref{prop:interface_separation}, this gives a strict hierarchy:
composite order determines the pair multiset, which determines the DPO objective;
component order determines neither, and a good top-$k$ selector need not fix the pairs
(Appendix~\ref{app:weight_geometry} gives a two-candidate example). Exact pair overlap
therefore certifies closeness of the objective, of its gradient and, under (iv), of
short training trajectories; because the trajectory bound grows with the horizon
\citep{hardt2016train}, overlap certifies objectives, and endpoints only when it is near one. Mean property contrasts are only a projection of the pair distribution
with an unmeasured residual (Proposition~\ref{prop:contrast_gradient}).

\paragraph{Matched DPO exposes the mismatch.}
With the same base model, critic, train pool, $4\,000$ pairs and recipe, the hard-gate
policy reaches $0.311$ target/draw and $0.758$ Answer@16; unit-sum pairs reach $0.134$ and
$0.596$ (Table~\ref{tab:train}). A validation pool selects $w=0.005$ for top-1 inference,
yet training on its pairs reaches only $0.571$; registered predictions that it and a soft
log-probability rule would match the gate within $0.04$ failed. In pair space $w=0.005$
contrasts mass but barely feasibility, and the unit sum the opposite.

A pre-specified $w=0.1$ passes near the gate in this two-dimensional contrast space
(Figure~\ref{fig:pair_contrast}, Appendix~\ref{app:weight_geometry})
and reaches $0.312$ target/draw, $0.821$ supply, and $0.770$ Answer@16, versus
$0.311/0.807/0.758$ for the gate. It shares only $15\%$ of pairs with the gate, so the match is not near-duplicate data;
across the five engine rules a bottleneck contrast relative to the gate orders the
endpoints (Spearman with the endpoint $0.90$, versus $-0.60$ for overlap: at low overlap,
overlap does not order endpoints; descriptive), and a pre-specified
critic-only version selects $w=0.075$ without oracle labels (Appendices~\ref{app:weight_geometry},~\ref{app:pair_audit}).

\begin{table}[t]
\centering
\caption{Matched DPO endpoints under a common fixed gate evaluator. $\Delta F$ is
chosen-minus-rejected feasibility and $\Delta M$ is rejected-minus-chosen absolute
oracle log-mass error. ``Overlap'' is exact full-pair overlap with the gate. 7B arms
are in Appendix~\ref{app:new_campaign}.}
\label{tab:train}
\scriptsize
\setlength{\tabcolsep}{2.25pt}
\renewcommand{\arraystretch}{0.94}
\begin{adjustbox}{max width=\linewidth}
\begin{tabular}{lcccccc}
\toprule
pairs ranked by & $\Delta F$ & $\Delta M$ & overlap & target/draw & supply@16 & Answer@16 \\
\midrule
1.5B unit sum & .64 & .11 & .099 & $.134\pm.005$ & $.636\pm.037$ & $.596\pm.033$ \\
1.5B hard gate & .48 & .63 & 1.000 & $\mathbf{.311\pm.013}$ & $\mathbf{.807\pm.007}$ & $\mathbf{.758\pm.005}$ \\
1.5B tuned, $w=.005$ & .13 & .82 & .210 & $.258\pm.028$ & $.649\pm.023$ & $.571\pm.030$ \\
1.5B contrast, $w=.1$ & .48 & .52 & .150 & $\mathbf{.312\pm.021}$ & $\mathbf{.821\pm.017}$ & $\mathbf{.770\pm.023}$ \\
\bottomrule
\end{tabular}
\end{adjustbox}
\end{table}

\paragraph{Boundary case.}
In GSM8K the format-score spread collapses after saturation, sum and gate pairs carry
nearly identical target signal, and the trained policies agree ($0.647$ vs $0.642$
target/draw; Appendix~\ref{app:gsm}): the separation vanishes where the interfaces coincide.

\paragraph{Preregistered external stress test (ECG).}
Because Proposition~\ref{prop:pair_stability} constrains only full pair distributions,
we froze an external test before training: per family, validation selects $w_{\rm sel}$
and the gate-free variant of our engine criterion (normalize each critic contrast by
its grid maximum, maximize the smaller) selects $w_{\rm pair}$; each rule builds
$4\,000$ pairs from one frozen train pool; three seeds; one common evaluator; margin
$\pm0.04$. No seed was rerun or removed.

\begin{table}[t]
\centering
\caption{Preregistered ECG training stress test (Answer@16 under the common
evaluator; mean $\pm$ training-seed SD). Overlap is exact pair overlap with the gate's
$4\,000$ pairs. Last columns: the frozen primary contrast with its paired
patient-cluster bootstrap CI (conditional on the seeds) and a post-hoc CI that also
resamples seeds. The single requirement-blind pipeline scores $0.512$.}
\label{tab:ecg_train}
\scriptsize
\setlength{\tabcolsep}{2.4pt}
\begin{adjustbox}{max width=\linewidth}
\begin{tabular}{lccccccccc}
\toprule
family & $w_{\rm sel}$ & $w_{\rm pair}$ & overlap sel & overlap pair & inference & gate & pair & pair $-$ gate [frozen CI] & [seed+subject] \\
\midrule
Llama & .003 & 1.0 & .923 & .256 & $.500\pm.026$ & $.507\pm.017$ & $.171\pm.296$ & $-.336\;[-.380,-.288]$ & $[-.557,-.005]$ \\
Gemma & .001 & .1 & .995 & .711 & $.005\pm.004$ & $.004\pm.003$ & $.123\pm.122$ & $+.119\;[+.056,+.184]$ & $[-.002,+.266]$ \\
Phi & .03 & .3 & .796 & .310 & $.719\pm.005$ & $.715\pm.014$ & $.523\pm.018$ & $-.192\;[-.249,-.145]$ & $[-.254,-.141]$ \\
\bottomrule
\end{tabular}
\end{adjustbox}
\end{table}

\paragraph{Frozen outcome.}
The preregistered primary outcome is negative for our criterion: pair weighting is
gate-equivalent in no family (Table~\ref{tab:ecg_train}). In Gemma it even exceeds the gate
($+0.119$), but no Gemma arm approaches the requirement-blind pipeline (best seed $0.244$
vs $0.512$), so this is not evidence for the criterion. Gate and
inference weight differ by at most $0.007$ with $79.6$--$99.5\%$ shared pairs. With
training seeds resampled together with subjects, only Phi's pair deficit stays clear of
zero (Appendices~\ref{app:ecg},~\ref{app:v23}). Llama's $w_{\rm pair}=1$ pairs are
identical to those of an added unit-sum arm, and both fail to parse on two seeds, a shared collapse rather than a test of
Proposition~\ref{prop:ranking_geometry}.

\paragraph{What the endpoints mean (post-hoc).}
Llama arms that parse mostly emit the requirement-blind pipeline ($0.512$), and Gemma
has near-zero supply. In Phi, gate and inference policies keep 79--113 distinct
pipelines, whereas all pair-weight seeds collapse onto the blind program ($\ge97\%$ of
parsed outputs). The gate-free criterion fails because in ECG the gate sits at the
$M$-heavy end of the sweep, so normalizing by sweep maxima pushes toward $F$; a
gate-referenced criterion would select $w=0.003$--$0.01$ with $87$--$99\%$ gate overlap.

\paragraph{A trained check on CPU.}
We trained a log-linear DPO reranker on the frozen Llama and Phi pools, where DPO is
convex and Proposition~\ref{prop:pair_stability}(iv) applies exactly: gate and nine grid
weights, ten seeds each, with design and four predictions hashed before any run
(Appendix~\ref{app:v23}). All four held: the gate-referenced weight matches the gate
($+0.004$ $[-0.005,0.015]$; $-0.002$ $[-0.007,0.002]$); Phi's gate-free weight is not shown
equivalent and, like the LLM policies, shifts top-1 choices onto the blind program
($67\%$ vs $25\%$); across the grid $|$gap$|$ shrinks as overlap grows (Spearman of $|$gap$|$ on overlap
$-0.60$, $-0.93$); and parameter
distances respect the bound. The reranker has little headroom above the blind
program, so it corroborates direction and ordering, not LLM effect sizes. Where gate pairs can be built, training on
them is simpler than any gate-referenced scalar.

\paragraph{Takeaway.}
Inference scalarization is a decision rule; pre-DPO scalarization is a
\emph{data-construction operator}. The evidence is one-directional and thin
(Figure~\ref{fig:v23}b, Appendix~\ref{app:v24}): all three rules with $\ge79\%$ gate
overlap reached the gate's endpoint, consistent with Proposition~\ref{prop:pair_stability},
but only Phi's comes from a family whose gate is informative; low overlap left the
endpoint undetermined. A weak bottleneck contrast accompanied all four
informative misses, but Phi's pair weight also missed with balanced contrasts, a residual
that two contrast coordinates cannot see. In practice: when a tuned selection weight trains poorly, check
its pair overlap with the rule it should emulate; if overlap is low, neither overlap
nor mean contrasts certify it, and the reference rule's own pairs are the safe default.

\end{document}